%% file: paper.tex
\documentclass[10pt, a4paper]{article}
\usepackage[utf8]{inputenc}
\usepackage[top=3cm, bottom=4cm, left=3.5cm, right=3.5cm]{geometry}

\usepackage{authblk}

\input{mystyle.tex}

\input{mycommands.tex}

\providebool{cleanbuild}
\setbool{cleanbuild}{false} %

\title{Theoretical guarantees for neural estimators in parametric statistics}

\begin{document}

\author[1]{Almut Rödder}
\author[2]{Manuel Hentschel}
\author[2]{Sebastian Engelke}
\affil[1]{ETH Zürich}
\affil[2]{University of Geneva}
\date{}

\maketitle

\input{tex/00_abstract.tex}

\clearpage

\section{Introduction}
\label{sec:introduction}

\input{tex/01_introduction.tex}

\section{Setup and risk decomposition}

\input{tex/02a_notation.tex}

\input{tex/02b_error_decomp.tex}

\input{tex/02c_examples.tex}

\input{tex/03_bayes.tex}

\input{tex/04_approximation.tex}

\input{tex/05_generalization.tex}

\input{tex/06_combination.tex}

\input{tex/07_simulation.tex}

\input{tex/08_conclusion.tex}

\clearpage
\appendix

\input{tex/A02_proofs.tex}

\input{tex/A03_proofsExamples.tex}

\input{tex/A04_Linearmodel.tex}

\clearpage
\printunsrtglossary[type=symbols,style=long]
\label{sec:symbols}

\bibliographystyle{abbrvnat}
\bibliography{references}

\end{document}

%% file: mystyle.tex
\usepackage{amsmath,amsthm,amsfonts,amssymb,amscd, fancyhdr, color, comment, graphicx, environ}
\usepackage[dvipsnames]{xcolor}
\usepackage{float}
\usepackage[framemethod=TikZ]{mdframed}
\usepackage{enumerate}
\usepackage[shortlabels]{enumitem}
\usepackage{fancyhdr}
\usepackage{indentfirst}
\usepackage{listings}
\usepackage{shadethm}
\usepackage{setspace}
\usepackage{caption}
\usepackage{subcaption}
\usepackage[citecolor=blue]{hyperref}
\usepackage{bbm}
\usepackage{bm}
\usepackage{neuralnetwork}
\usepackage{tikz}
\usepackage{sidecap}
\usepackage{booktabs}
\usepackage{lmodern} %
\usepackage[capitalize,sort]{cleveref}
\usepackage[authoryear]{natbib}
\usepackage{placeins} %

\usepackage{aliascnt}
\usepackage{xparse}
\let\oldtheorem\newtheorem
\RenewDocumentCommand{\newtheorem}{m o m o}{%
    \IfValueTF{#2}{%
        \IfNoValueTF{#4}{%
            \newaliascnt{#1}{#2}%
            \oldtheorem{#1}[#1]{#3}%
            \aliascntresetthe{#1}%
        }{%
            \oldtheorem{#1}[#2]{#3}[#4]%
        }%
    }{%
        \IfNoValueTF{#4}{%
            \oldtheorem{#1}{#3}%
        }{%
            \oldtheorem{#1}{#3}[#4]%
        }%
    }%
}

\crefformat{equation}{#2(#1)#3}
\crefmultiformat{equation}{(#2#1#3)}%
{ and~(#2#1#3)}{, (#2#1#3)}{, and~(#2#1#3)}
\crefname{assumption}{Assumption}{Assumptions}
\crefname{prop}{Proposition}{Propositions}
\crefname{example}{Example}{Examples}
\crefname{figure}{Figure}{Figures}

\usepackage[symbols,nogroupskip,record,numberedsection=autolabel]{glossaries-extra}

\usetikzlibrary{decorations.pathreplacing}
\usetikzlibrary{fadings}
\hypersetup{
    colorlinks=true,
    linkcolor=black,
    filecolor=magenta,
    urlcolor=blue,
}

\newtheorem{theorem}{Theorem}[section]
\newtheorem{prop}[theorem]{Proposition}
\newtheorem{lemma}[theorem]{Lemma}
\newtheorem{definition}[theorem]{Definition}
\newtheorem{assumption}[theorem]{Assumption}
\newtheorem{corollary}[theorem]{Corollary}
\newtheorem{remark}[theorem]{Remark}
\newtheorem{example}[theorem]{Example}

%% file: mycommands.tex
\usepackage{amsmath, amssymb} %
\usepackage{mathtools} %
\usepackage{bbm} %
\usepackage{etoolbox}

\providebool{cleanbuild}

\providecommand{\comment}{}
\renewcommand{\comment}[2]{%
    \ifbool{cleanbuild}{}{%
        \ifmmode%
            \text{\textcolor{#1}{\sffamily\small[{#2}]}}%
        \else%
            \noindent%
            \textcolor{#1}{\sffamily\small[{#2\unskip}]}%
        \fi%
    }%
}

\newcommand{\normalleftrightnull}{\delimitershortfall 5pt \delimiterfactor 901}
\newcommand{\normalleftright}{\normalleftrightnull\aftergroup\normalleftrightnull}

\newcommand{\midleftrightnull}{\delimitershortfall 15pt \delimiterfactor 700}
\newcommand{\midleftright}{\midleftrightnull\aftergroup\midleftrightnull}

\newcommand{\smallleftrightnull}{\delimitershortfall 30pt \delimiterfactor 400}
\newcommand{\smallleftright}{\smallleftrightnull\aftergroup\smallleftrightnull}

\newcommand{\noleftrightnull}{\delimitershortfall 500pt \delimiterfactor 0}
\newcommand{\noleftright}{\noleftrightnull\aftergroup\noleftrightnull}

\normalleftright %

\newcommand{\myleft}{\mathopen{}\mathclose\bgroup\left}
\newcommand{\myright}{\aftergroup\egroup\right}

\newcommand{\lr}[1]{\myleft(#1\myright)}

\newcommand{\biglr}[1]{\bigl(#1\bigl)}

\newcommand{\set}[1]{\myleft\{#1\myright\}}

\newcommand{\setm}[2]{\myleft\{#1 \;\middle\vert\; #2\myright\}}

\newcommand{\norm}[2][]{\myleft\Vert#2\myright\Vert_{#1}}
\newcommand{\tnorm}[2][]{\Vert#2\Vert_{#1}}

\newcommand{\abs}[1]{\left|{#1}\right|}
\newcommand{\tabs}[1]{|#1|}
\newcommand{\bigabs}[1]{\bigl|{#1}\bigr|}

\newcommand{\brackets}[1]{\myleft\lbrack#1\myright\rbrack}

\newcommand{\ceil}[1]{\myleft\lceil#1\myright\rceil}

\newcommand{\pto}{\stackrel{p}{\longrightarrow}}
\newcommand{\Lrto}{\stackrel{L^r}{\longrightarrow}}
\newcommand{\Loneto}{\stackrel{L^1}{\longrightarrow}}
\newcommand{\Ltwoto}{\stackrel{L^2}{\longrightarrow}}

\newcommand{\T}{^\top\!}
\newcommand{\inv}{^{-1}}

\newcommand{\N}{\mathbb{N}}
\newcommand{\R}{\mathbb{R}}

\newcommand{\Rd}[1][d]{\mathbb{R}^{#1}}

\newcommand{\NN}{\mathbb N}

\newcommand{\Scal}{\mathcal{S}}

\newcommand{\Zcal}{\mathcal{Z}}
\newcommand{\Ycal}{\mathcal{Y}}
\newcommand{\Xcal}{\mathcal{X}}
\newcommand{\Fcal}{\mathcal{F}}

\newcommand{\covering}{\mathcal{N}}

\newcommand{\normal}{\mathcal{N}} %

\newcommand{\Pf}{\mathcal{P}}

\newcommand{\Var}{\mathrm{Var}}

\newcommand{\Cov}{\mathrm{Cov}}
\renewcommand{\P}{\mathbb{P}}
\newcommand{\E}{\mathbb{E}}

\renewcommand{\det}[1]{\abs{#1}}

\newcommand{\zerovec}{\mathbf{0}}

\newcommand{\indic}[1]{\mathbbm{1}_{#1}}

\newcommand{\limit}[2][\infty]{\lim_{#2\rightarrow#1}}

\newcommand{\half}{\tfrac{1}{2}}

\DeclareMathOperator*{\argmin}{arg\,min}

\newcommand{\myFunctionDefinition}[5]{
    #1: \; & #2 &&\longrightarrow #3
    \\
    & #4 &&\longmapsto #5
}

\newcommand{\mycases}[1]{
    \begin{cases}
        \begin{aligned}
            #1
        \end{aligned}
    \end{cases}
}
\newcommand{\mycase}[2]{&#1&\qquad&#2}

\newcommand{\provenlabel}[1]{%
    \label{#1}%
    \ifbool{cleanbuild}{}{%
        \makeatletter%
        \ifcsname r@proof:#1\endcsname%
            \hyperlink{proof:#1}{(See Proof.)}%
        \fi%
        \makeatother%
    }%
}

\newcommand{\proofref}[1]{%
    \label{proof:#1}%
    Proof of \hypertarget{proof:#1}{\cref{#1}}%
}

\newcommand{\Matern}{Mat\'{e}rn}
\newcommand{\Frechet}{Fr\'{e}chet}
\newcommand{\BrownResnick}{Brown--Resnick}
\newcommand{\BvM}{Bernstein--von Mises}

%% file: tex/00_abstract.tex
\begin{abstract}
    Neural estimators are simulation-based estimators for the
    parameters of a family of statistical models, which build
    a direct mapping from the sample to the parameter vector.
    They benefit from the versatility of available network architectures
    and efficient training methods developed in the field of
    deep learning. Neural estimators are amortized in the sense
    that, once trained, they can be applied to any new data
    set with almost no computational cost.
    While many papers have shown very good performance
    of these methods in simulation studies and real-world
    applications, so far no statistical guarantees are
    available to support these observations theoretically.
    In this work, we study the risk of neural estimators
    by decomposing it into several terms that can be analyzed
    separately. We formulate easy-to-check assumptions
    ensuring that each term converges to zero, and
    we verify them for popular applications of neural estimators.
    Our results provide a general recipe to derive theoretical
    guarantees also for broader classes of architectures
    and estimation problems.
\end{abstract}

%% file: tex/01_introduction.tex
Estimation of the parameters of a family of distributions 
$\set{\Pf_\theta}_{\theta \in \Theta}$,
taking values in $\Zcal \subseteq \Rd$ and
parametrized by $\theta \in \Theta \subseteq \R^p$,
is one of the most classical problems in statistics. 
The two main approaches are frequentist maximum-likelihood estimators and Bayes estimators relying on posterior distributions.
In both settings we can define the set of all point estimators as
$\Fcal = \set{f: \Zcal^m \rightarrow \Rd[p]}$
that map $m$ i.i.d.~samples $Z = (Z^1, \dots, Z^m)$ 
from some $\Pf_\theta$ to an estimate~$\hat \theta$.
The asymptotic properties of such estimators is well-understood \citep{vaart_1998}.

In many modern applications, the likelihood function is not available, or its evaluation is computationally prohibitive. This includes examples in cosmology \citep{als2019}, spatial Gaussian random fields \citep{rue2005gaussian} or max-stable distributions \citep{deh1984}. In some cases, composite likelihood approaches are a computationally faster alternative, but may result in significant efficiency loss \citep{var2011}. Simulation from such models is typically possible and computationally relatively cheap. For this reason, different simulation-based (or likelihood-free) inference methods have been developed, including indirect inference~\citep{gou1993} and the well-known approach of approximate Bayesian computation (ABC)~\citep{beaumont2002approximate}. Designing informative summary statistics can however be difficult and bad choices may lead to suboptimal performance.

Similar to ABC, neural estimators are simulation-based methods in the  Bayesian setting.
For a given loss function $\ell$, the Bayes estimator $f^*_m$ minimizes the weighted risk 
\begin{align}
    \label{eq:risk_intro}
    R(f)
    &=
    \int_{\Theta}  \E_{\theta}\left[ 
    \ell(f(Z), \theta)\right] \pi(\theta) d\theta
    ,
\end{align}
over all estimators $f\in\mathcal F$, where 
the weight function $\pi(\theta)$ is the 
density of the prior distribution on $\Theta$ \citep[e.g.,][]{vaart_1998}. The corresponding Bayes risk $R(f_m^*)$ is thus the smallest possible population error for sample size $m$.
We focus here on the quadratic loss $\ell(\hat \theta, \theta) =
\norm[2]{\hat \theta - \theta}^2 \midleftright$ where
the Bayes estimator is given by the conditional expectation
$f_m^*(z) = \E(\theta\mid Z=z)$, $z \in \Zcal^m$.
With this perspective of an estimator as a function from data to parameters, neural networks are a natural candidate to learn this function in light of the universal approximation property \citep{cybenko1989approximation}.

A neural estimator frames parameter estimation as supervised learning problem with the following steps.
We first generate the training data consisting of $N$ samples $\theta_1,\dots, \theta_N$ from the prior distribution $\pi(\theta)$ and consider one sample $Z_i = (Z_i^1, \dots, Z_i^m) \in \Zcal^m$ from $\Pf_{\theta_i}^m$ for each of them, $i = 1, \dots, N$. The pair $(Z_i, \theta_i)$ can be seen as an independent draw of a predictor-response pair of the training distribution in a prediction problem.
The neural estimator is then the minimizer of the empirical version $R_N(f)$ of the population risk in~\eqref{eq:risk_intro} based on this simulated data set, i.e., the empirical risk minimizer
\begin{align}
    \label{eq:intro_emp_risk_minimizer}
    \phi_m^N = \argmin_{\phi \in \Phi} 
    \frac{1}{N}
    \sum_{i=1}^N
    \norm[2]{f_\phi(Z_i) - \theta_i}^2
    ,
\end{align}
where we optimize over all neural networks $f_\phi: \Zcal^m \to \mathbb R^p$  parametrized by the set
$\Phi$ of all considered network parameters,
architectures and activation functions. Figure~\ref{fig:NN_architecture} illustrates the architecture of a neural estimator with a fully connected neural network; we will use the parameter $\phi$ and its induced neural network $f_\phi$ exchangeably.
In practice, it is typically not possible and not desired to obtain the global minimizer $\phi_m^N$ exactly. Instead, regularization methods such as penalties, dropout and early stopping are employed during training to avoid overfitting \citep{goodfellowEtAl2016}. The resulting estimator $\tilde \phi_m^N$ usually has a larger training risk~\eqref{eq:intro_emp_risk_minimizer} but a smaller population risk~\eqref{eq:risk_intro}.

The fundamental difference of neural estimators to classical simulation-based approaches such as ABC is that they are amortized, that is, the only significant computational cost is at training time. Once the neural estimator is trained, it can be applied to any new data set with $m$ samples from $\Pf_\theta$ for an arbitrary $\theta\in\Theta$ essentially for free, since the forward pass
in a neural network is computationally very cheap.
Moreover, unlike ABC, neural Bayes estimators do not rely on fixed summary statistics but directly build a map from the data set to the parameter. Intuitively, the summary statistics are learned automatically as latent representations of the network.

The concept of neural estimators exists since a long time and they have been applied to various problems such as parameter estimation in times series \citep{chon1997linear} and birth-and-death processes \citep{BOKMA2006449} using dense networks. More recently, different architectures 
such as convolutional layers for efficient inference in temporal \citep{rud2022} or spatial observations \citep{gerber2021fast, lenzi2021neural, walchessen2023neural} or deep sets for multiple replicates \citep{sainsbury2023likelihood} have been proposed.
While these methods seem to work well in practice and allow for very fast inference, so far no theoretical guarantees for the asymptotic or pre-asymptotic behavior of neural estimators exist; see, for instance, the discussion in~\cite{zam2025}.

Our contribution is threefold. We first decompose the population risk $R(\tilde\phi_m^N)$ of the fitted neural estimator with $m$ replicates and based on $N$ simulated training data as
\begin{align}
    \label{eq:intro_errdec}
    R(\tilde\phi_m^N)
    &=
    \underbrace{R(f_m^*)}_{\text{Bayes Risk}}
    + \underbrace{R(\phi_m^*)-R(f_m^*)}_{\text{Approximation Error}}
    + \underbrace{R(\phi_m^N) - R(\phi_m^*)}_{\text{Generalization Error}}
    + \underbrace{R(\tilde\phi_m^N) - R(\phi_m^N)}_{\Delta{\text{Algorithm}}}
    ;
\end{align}
see Figure~\ref{fig:risk_illustration} for a geometric illustration of this decomposition.
The risk of a neural estimator therefore differs from the optimal Bayes risk through three error terms:
the approximation error quantifies how well the class of neural networks~$\Phi$ can approximate the Bayes estimator;
the generalization error measures the amount of overfitting of the empirical risk minimizer $\phi_m^N$ compared to the optimal neural network estimator;
and a term that depends on the optimization algorithm and regularization scheme used during training.
This last term
is typically negative due to the positive effect of regularization on the generalization of the network.
We do not consider it here and refer to the literature on numerical analysis and statistical learning theory for details \citep[e.g.,][and references therein]{xing2018walksgd}.
We study all other terms of decomposition~\eqref{eq:intro_errdec} separately and formulate easy-to-check assumptions under which they converge to zero as $m,N \to\infty$.
This involves results from different fields such as the approximation theory of neural networks \citep{cybenko1989approximation}, robustness \citep{xu2012robustness},
and statistical efficiency of Bayes estimators \citep{vaart_1998}.
Combining the results, we show consistency of neural estimators~$\phi_m^N$ in a Bayesian sense. We further discuss that this estimator
can be fully efficient compared to the Bayes estimator $f_m^*$ under certain conditions.

As a second contribution, we check the assumptions of our theory  for some of the most common applications of neural estimators.
For spatial Gaussian processes with many locations, for instance, likelihood estimation can become prohibitively expensive \citep{rue2005gaussian}, but simulation remains feasible. For multivariate max-stable distributions in extreme value theory, full-likelihood inference becomes infeasible even in moderate dimensions \citep{wad2015, dombry2017bayesian}, but simulation can be done efficiently \citep{dom2016}.

Thirdly, our work provides a clear guide for proving asymptotic properties of neural estimators also in other settings.
The error decomposition~\eqref{eq:intro_errdec} separates the overall task into sub-problems from different disciplines and allows focussing on the relevant terms.
For instance, if a different network architecture is used (e.g., convolutional or graph neural networks) for the same parametric family $\Pf_\theta$, the Bayes risk will not be affected, and only the risk terms related to neural networks architectures change. On the other hand, for a different parametric model class, only the Bayes risk and a small condition
on the tail heaviness of the statistical model $Z\sim \Pf_\theta$ have to be studied if the network architecture and training procedure remains the same.

\newcounter{mysymbolcounter}
\newcommand{\mysymbol}[2]{%
    \stepcounter{mysymbolcounter}%
    \glsxtrnewsymbol[description={#2}]{\arabic{mysymbolcounter}}{#1}%
}

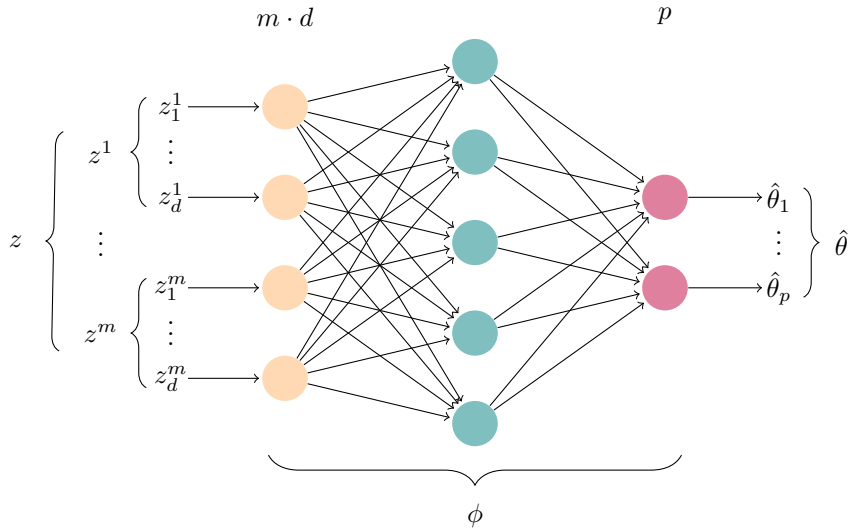
\begin{figure}
    \centering
    \input{tikz/NN.tex}
    \caption{%
        Example of a fully-connected neural network used as a neural estimator $f_\phi: \mathcal Z^m \to\mathbb R^p$.
        The size of the input and output layer is indicated above.
        Note that the subscripts to each $z^i$ and $\theta$
        indicate the corresponding dimension in the input and output space,
        whereas in the rest of this paper, they denote the index $1, \dots, N$ of each sequential training sample.
    }
    \label{fig:NN_architecture}
\end{figure}

%% file: tikz/NN.tex
\begin{tikzpicture}[x=1cm,y=1.2cm]

\newcommand{\inputnum}{4} %
\newcommand{\hiddennum}{5} %
\newcommand{\outputnum}{2} %

\newcommand{\inputcolor}{orange!30}
\newcommand{\hiddencolor}{teal!50}
\newcommand{\outputcolor}{purple!50}

\foreach \i in {1,...,\inputnum}{
    \node[circle, minimum size = 6mm, fill=\inputcolor] (Input-\i) at (0,-\i) {};

    \node (Input-\i-label) at (-1.5,-\i) {};

    \draw[->, shorten <= 1mm, shorten >= 1pt] (Input-\i-label) -- (Input-\i);
}

\foreach \i in {1,...,\hiddennum}{
    \node[circle, minimum size = 6mm, fill=\hiddencolor, yshift=(\hiddennum-\inputnum)*6 mm]
    (Hidden-\i) at (2.5,-\i) {};
}

\foreach \i in {1,...,\outputnum}{
    \node[circle, minimum size = 6mm, fill=\outputcolor, yshift=(\outputnum-\inputnum)*6 mm]
        (Output-\i) at (5,-\i) {};

    \node[yshift=(\outputnum-\inputnum)*6 mm] (Output-\i-label) at (6.5,-\i) {};

    \draw[->, shorten >= 1mm] (Output-\i) -- (Output-\i-label);
}

\foreach \i in {1,...,\inputnum}{
    \foreach \j in {1,...,\hiddennum}{
        \draw[->, shorten >=1pt] (Input-\i) -- (Hidden-\j);
    }
}

\foreach \i in {1,...,\hiddennum}{
    \foreach \j in {1,...,\outputnum}{
        \draw[->, shorten >=1pt] (Hidden-\i) -- (Output-\j);
    }
}

\draw (Input-1-label) node{$z_1^1$};
\draw (Input-2-label) node{$z_d^1$};
\draw (Input-3-label) node{$z_1^m$};
\draw (Input-4-label) node{$z_d^m$};

\path (Input-1-label) -- (Input-2-label) node[midway, yshift=1mm]{$\vdots$};
\draw [decorate, decoration={brace, amplitude=2mm, mirror, raise=2mm}]
(Input-1-label.north west) -- (Input-2-label.south west) node [black, midway, xshift=-8mm] (Input-brace-1) {$z^1$};

\path (Input-3-label) -- (Input-4-label) node[midway, yshift=1mm]{$\vdots$};
\draw [decorate, decoration={brace, amplitude=2mm, mirror, raise=2mm}]
(Input-3-label.north west) -- (Input-4-label.south west) node [black, midway, xshift=-8mm] (Input-brace-2) {$z^m$};

\path (Input-brace-1) -- (Input-brace-2) node[midway, yshift=1mm]{$\vdots$};
\draw [decorate, decoration={brace, amplitude=2mm, mirror, raise=2mm}]
(Input-brace-1.north west) -- (Input-brace-2.south west) node [black, midway, xshift=-8mm] (Input-brace-3) {$z$};

\draw (Output-1-label) node{$\hat\theta_{1}$};
\draw (Output-2-label) node{$\hat\theta_{p}$};

\path (Output-1-label) -- (Output-2-label) node[midway, yshift=1mm]{$\vdots$};
\draw [decorate, decoration={brace, amplitude=2mm, raise=2mm}]
(Output-1-label.north east) -- (Output-2-label.south east) node [black, midway, xshift=7mm] {$\hat\theta$};

\draw [decorate, decoration={brace, amplitude=4mm, mirror, raise=2mm}]
(Hidden-\hiddennum.south east -| Input-4.south west) -- (Hidden-\hiddennum.south east -| Output-2.south east)
node [black, midway, yshift=-1cm] (Input-brace-3) {$\phi$};

\node (Input-Size) at (0,0) {$m \cdot d$};
\node (Output-Size) at (5,0) {$p$};

\end{tikzpicture}

%% file: tex/02a_notation.tex
We introduce the mathematical notation more formally in Section~\ref{sec:notation} and present the error decomposition in Section~\ref{sec:errdec}.
Section~\ref{sec:exampleModels} discusses Gaussian random field models and max-stable distributions, which serve as running examples used to illustrate our results throughout the paper.

\subsection{Setup and notation}\label{sec:notation}

We consider a family of distributions,
$\set{\Pf_\theta}_{\theta \in \Theta}$,
taking values in $\Zcal \subseteq \Rd$ and
parametrized by $\theta \in \Theta \subseteq \R^p$,
for some $d, p \in \N$.
The goal is to estimate the parameter $\theta$
from $m$ i.i.d. samples $Z = (Z^1, \dots, Z^m)$,
taking values in $\Zcal^m \subseteq \R^D$,
for $D = md$,
with
$Z^i \sim \Pf_\theta$,
or equivalently $Z \sim \Pf_\theta^m$.
Throughout this paper, we consider
a continuous distribution with respect to the Lebesgue measure,
having density $p_\theta: \Rd \rightarrow \R$ (set to zero outside $\Zcal$).
The joint density of $Z$ is then given by
\begin{align*}
    p_{\theta, m}(z)
    &=
    \prod_{i=1}^m p_\theta(z^i)
    ,
    \quad 
    z \in \mathbb R^D.
\end{align*}
For a true parameter $\theta$,
the quality of an estimate~$\hat \theta$
is quantified by the loss function~$\ell$.
Note that we allow the estimate to be from $\Rd[p]$,
since a general estimator might map to values outside $\Theta$.
In general, there are few restrictions on the loss function,
but we will mostly consider the quadratic loss
\begin{alignat}{6}
    \label{eq:quadraticLoss}
    \myFunctionDefinition{\ell}{
        \Rd[p] \times \Theta
    }{
        \R
        \nonumber
    }{
        (\hat \theta, \theta)
    }{
        \ell(\hat \theta, \theta)
        =
        \norm[2]{\hat \theta - \theta}^2
    }
    .
\end{alignat}
Let $\Fcal$
denote the set of measurable functions
from $\Zcal^m$ to $\Rd[p]$.
To quantify the quality of an estimator $f \in \Fcal$,
for a given true parameter $\theta$,
we consider the point-wise risk
\begin{align*}
    R_{\theta}(f)
    &=
    \E_{Z \sim \Pf_\theta^m}
    \ell(f(Z), \theta)
    =
    \int_{\Zcal^m}
    p_{\theta, m}(z)
    \norm[2]{f(z) - \theta}^2
    dz
    ,
\end{align*}
assuming a density and the quadratic loss function in the second equality.
For a cleaner notation, the number of simultaneous samples $m$ is implied by the domain of $f$,
and not denoted explicitly.

In the Bayesian setting,
which is implicitly assumed here,
the parameter $\theta$ is considered to be a random variable
distributed according to a prior distribution $\Pi$,
which is typically absolutely continuous with respect to the Lebesgue measure with
density $\pi: \Theta \rightarrow \R$.
The quality of an estimator $f$ is then quantified by the risk
\begin{align*}
    R(f)
    &=
    \E_{\theta \sim \Pi}
    R_{\theta}(f)
    =
    \int_{\Theta}
    \int_{\Zcal^m}
    \pi(\theta)
    p_{\theta, m}(z)
    \norm[2]{f(z) - \theta}^2
    dz
    d\theta
    ,
\end{align*}
again assuming densities and the quadratic loss function in the second equality;
see also \cref{eq:risk_intro}.
An estimator $f_m^*$ that minimizes the risk $R$ for $m$ samples
is called a Bayes estimator and the corresponding risk $R(f_m^*)$ is called the Bayes risk.
For the quadratic loss function, the Bayes estimator is given by the conditional expectation
$f_m^*(z) = \E(\theta|Z=z)$, $z \in \Zcal^m$.
In the following,
unless stated otherwise,
we assume that both $\Pi$ and $\Pf_\theta$ are continuous,
having Lebesgue densities $\pi$ and $p_\theta$,
and consider the quadratic loss function from \cref{eq:quadraticLoss}.

Since the population risk $R(f)$ of some estimator $f\in\mathcal F$ is not accessible, it is typically approximated by the empirical risk.
To this end, let $N \in \N$ be the number of samples $\theta_1, \dots, \theta_N$ from $\theta \sim \Pi$, 
and consider one sample from $\Pf_\theta^m$ for each of them,
denoted by
$Z_i = (Z_i^1, \dots, Z_i^m) \in \Zcal^m$,
for $i = 1, \dots, N$. %
In general, we could also consider multiple samples per $\theta$.
The empirical version of $R(f)$ is then given by
\begin{align*}
    R_N(f)
    &=
    \frac{1}{N}
    \sum_{i=1}^N
    \ell(f(Z_i), \theta_i)
    =
    \frac{1}{N}
    \sum_{i=1}^N
    \norm[2]{f(Z_i) - \theta_i}^2
    ,
\end{align*}
where, as above, we assume the quadratic loss function in the second equation.
Note that empirical risks are random variables,
since they depend on the random samples $Z_i$ and $\theta_i$.

In the context of neural estimators,
the function $f:\Zcal^m \to \Rd[p]$ is given by a neural network.
For a given input size $m$, architecture and activation function,
a neural network with $L$ layers is parametrized by its
weight matrices $\set{w^{(l)}}_{l = 1, \dots, L}$ and bias vectors $\set{\beta^{(l)}}_{l = 1, \dots, L}$,
collectively denoted as $\phi_m = (w, \beta)$,
or just $\phi$ when the input size is implied.
For a given parameter vector $\phi$
we denote by $f_\phi$ the corresponding map from $\Zcal^m$ to $\Rd[p]$.
Below, the corresponding risks are abbreviated as
$R(\phi) := R(f_\phi)$ and $R_N(\phi) := R_N(f_\phi)$.

A full list of symbols used in this paper can be found in \cref{glo:symbols}.

\mysymbol{$\theta \in \Theta \subseteq \R^p$}{Parameter value}
\mysymbol{$\theta_0 \in \Theta$}{True parameter}
\mysymbol{$\Pi$}{Prior distribution of $\theta$}
\mysymbol{$\pi: \R^p \rightarrow \R$}{Density of $\Pi$}
\mysymbol{$\set{\Pf_\theta}_{\theta \in \Theta}$}{Family of distributions}
\mysymbol{$\Zcal \subseteq \R^d$}{Codomain of each $\Pf_\theta$}
\mysymbol{$p_\theta: \Rd \rightarrow \R$}{Density of $\Pf_\theta$}
\mysymbol{$m \in \N$}{Number of simultaneous samples}
\mysymbol{$Z = (Z^1, \dots, Z^m)$}{$m$-Sample (random variable)}
\mysymbol{$\Zcal^m \subseteq \R^{md}$}{Codomain of $Z$}
\mysymbol{$\theta_{\text{Bayes}}$}{Bayes estimator}
\mysymbol{$f^*(z) = \E(\theta | Z = z)$}{Bayes estimator as function of $z$}
\mysymbol{$p_{\theta, m}(z)$}{Density of $Z$}
\mysymbol{$\phi_m = (w, \beta)$}{Neural network parametrization with $m$ input samples}
\mysymbol{$L$}{Number of layers in neural network}
\mysymbol{$\set{w^{(l)}}_{l = 1, \dots, L}$}{Weight matrices of neural network}
\mysymbol{$\set{\beta^{(l)}}_{l = 1, \dots, L}$}{Bias vectors of neural network}
\mysymbol{$\phi_m^*$}{Optimal parametrization}
\mysymbol{$f$}{A mapping, usually from $\Zcal^m$ to $\Theta$}
\mysymbol{$f_\phi$}{Map from $\Zcal^m$ to $\Theta$ defined by $\phi$}
\mysymbol{$\ell(x,y) = \norm[2]{x-y}^2$}{Quadratic loss function}
\mysymbol{$R(f)$}{Risk of $f$}
\mysymbol{$R_N(f)$}{Empirical risk of $f$}
\mysymbol{$R(\phi)$}{Risk of $f_\phi$}
\mysymbol{$R_N(\phi)$}{Empirical risk of $f_\phi$}
\mysymbol{$R(f_m^*)$}{Bayes Risk}
\mysymbol{$\phi_m^*$}{Optimal parametrization with respect to $R$}
\mysymbol{$\phi_m^N$}{Optimal parametrization with respect to $R_N$}
\mysymbol{$\tilde\phi_m^N$}{Approximation of $\phi_m^N$ during training}

%% file: tex/02b_error_decomp.tex
\subsection{Error decomposition}

\label{sec:errdec}

For the sake of a clearer exposure, we assume in the following that all minimizers
exist and are unique.
    This assumption does not fundamentally restrict the presented decompositions.
    Indeed, existence could be guaranteed by considering
    compact parameter spaces for neural networks,
    and in the case of non-unique minimizers,
    one could be chosen arbitrarily.
    On the other hand, the error decomposition in \cref{eq:errdec}
    could also be rephrased by considering the infimum of the risks,
    $\inf_{\phi \in \Phi} R(\phi)$,
    rather than the risk of the minimizer
    $R(\phi_m^*)$.

Recall that all expressions containing $N$,
either as a subscript to the risk or as a superscript to the parametrization,
are random variables,
since they depend on the training data set.
As discussed in the introduction, the Bayes estimator
$f_m^* = \argmin_{f\in\Fcal} R(f)$
can be
approximated by a neural network $f_{\phi_m^*}$
satisfying
\begin{align}
    \label{eq:defPhiN}
    \phi_m^* = \argmin_{\phi \in \Phi} R(f_{\phi})
    ,
\end{align}
where $\Phi$ denotes the set of all considered network parameters,
architectures, and activation functions.
There is a huge literature on possible architectures of neural networks. A concrete example of the parametrization $\Phi$ are single-layer, fully-connected neural networks.
\begin{example}
    \label{example:single_layer_NN}
   Consider the family of single-layer, fully-connected, feedforward neural networks
    with $D$ input neurons, $N_1$ hidden neurons, and a single output neuron,
    with activation function $\sigma$ applied to the hidden neurons.
    This set can be denoted as
    \begin{align*}
        \Phi
        &=
        \setm{
            \phi = (w, \beta)
        }{
            w^{(1)} \in \R^{N_1 \times D},
            w^{(2)} \in \R^{1 \times N_1},
            \beta^{(1)} \in \R^{N_1},
            \beta^{(2)} \in \R
        }
        ,
    \end{align*}
    where
    $w^{(1)}$ and $w^{(2)}$
    are the weight matrices,
    and
    $\beta^{(1)}$ and $\beta^{(2)}$
    are the bias vectors.
    The corresponding set of functions contains all neural networks of the form
    \begin{align*}
        f_{\phi}(z)
        =
        \beta^{(2)}
        + \sum_{j=1}^{N_1}
        w^{(2)}_j \sigma(w^{(1)}_j z + \beta^{(1)}_j)
        ,
        \quad 
        z \in \Rd[D]
        .
    \end{align*}
\end{example}

Since the population risk $R$ in~\cref{eq:defPhiN} is a theoretical quantity,
we instead need to consider the empirical risk minimizer
\begin{align*}
    \phi_m^N = \argmin_{\phi \in \Phi} R_N(f_\phi)
    .
\end{align*}
In practice,
yet another network,
denoted $\tilde\phi_m^N$,
is trained,
due to several reasons.
First, the optimization problem \cref{eq:defPhiN} is non-convex,
and usually infeasible to solve exactly.
Second, and more importantly,
it is well-known that considering the naive empirical risk minimizer
leads to overfitting,
which can be mitigated by regularization techniques.
These differences are illustrated and discussed in a simulation study
in \cref{sec:simulation};
for the theoretical results in the following sections,
we will focus on $\phi_m^N$.

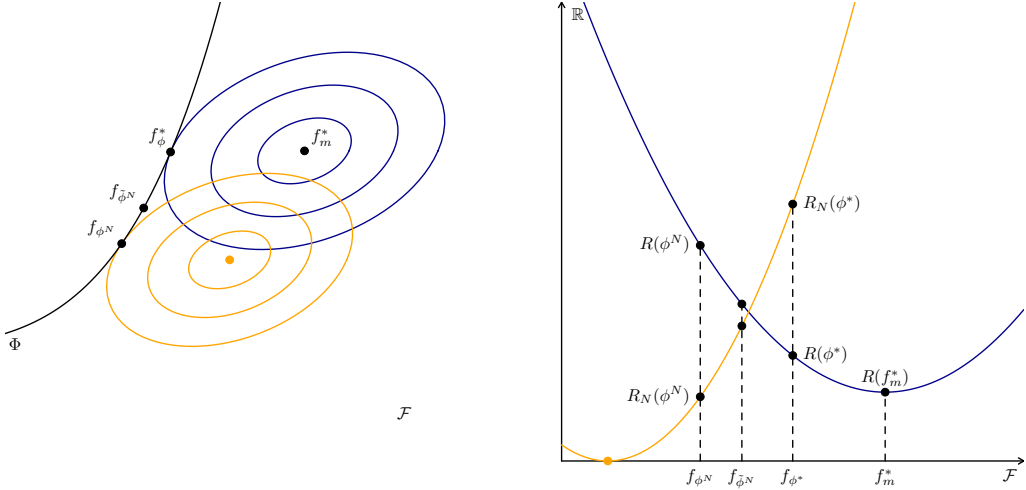
\begin{figure}
    \resizebox{0.48\textwidth}{!}{%
        \input{tikz/risk_contours.tex}
    }
    \resizebox{0.48\textwidth}{!}{%
        \input{tikz/risk_parabolas.tex}
    }
    \caption{%
        Schematic illustration of the different estimators and risks in the decompositions~\cref{eq:errdec} and~\cref{eq:genErrdec}.
        Left:
        The coordinate system represents the space $\Fcal$ of all possible estimators,
        and the black line the subset that are neural networks $\Phi$.
        The blue contour lines correspond to the population risk $R$,
        its minimizer in $\Fcal$ is the Bayes estimator $f_m^*$,
        and in $\Phi$ is $f_{\phi_m^*}$.
        The orange contour lines correspond to the empirical risk $R_N$,
        its minimizer in the $\Phi$ is $f_{\phi_m^N}$.
        Its minimum over $\Fcal$ is usually zero, e.g., achieved by a $1$-nearest neighbor estimator.
        Right:
        The x-axis is a slice of $\Fcal$,
        and the blue and orange curves represent the population risk
        and empirical risk, respectively.
    }
    \label{fig:risk_illustration}
\end{figure}

Using the above chain of approximations
($f_m^* \approx f_{\phi_m^*} \approx f_{\phi_m^N}$),
the population risk $R(\phi_m^N)$ of the empirical risk minimizer
can be decomposed as
\begin{align}
    \label{eq:errdec}
    R(\phi_m^N)
    &=
    \underbrace{R(f_m^*)}_{\text{Bayes Risk}}
    + \underbrace{R(\phi_m^*)-R(f_m^*)}_{\text{Approximation Error}}
    + \underbrace{R(\phi_m^N) - R(\phi_m^*)}_{\text{Generalization Error}}
    .
\end{align}
The first term is the Bayes risk,
inherent to the task of predicting parameters
of a distribution
based on a finite number of observations
$Z = (Z^1, \dots, Z^m)$. 
The second term is the error incurred by approximating the Bayes estimator
$f_m^*(z) = \E (\theta \vert Z = z)$ %
with a neural network~$f_{\phi_m^*}$,
considering an optimal parametrization $\phi_m^*$
(with respect to a suitable choice of architecture and activation function,
discussed in \cref{sec:approximationError}).

The third term, the generalization error,
is due to using the parameters $\phi_m^N$,
learned from~$N$ training samples,
rather than the optimal $\phi_m^*$.
It can be further decomposed into
\begin{align}
    \label{eq:genErrdec}
    \underbrace{R(\phi_m^N) - R(\phi_m^*)}_{\text{Generalization Error}}
    &=
    \underbrace{R(\phi_m^N) - R_N(\phi_m^N)}_{\text{Generalization Error 2}}
    + \underbrace{R_N(\phi_m^N) - R_N(\phi_m^*)}_{\leq 0 \text{ a.s.}}
    + \underbrace{R_N(\phi_m^*) - R(\phi_m^*)}_{\to 0}
    ,
\end{align}
which is driven by the first term,
$R(\phi_m^N) - R_N(\phi_m^N)$,
i.e., the discrepancy between the true risk $R$ and the empirical risk $R_N$
used to find $\phi_m^N$.
The last two terms are easy to control with conventional techniques. Indeed,
the second term is non-positive, since $\phi_m^N$ is optimal with respect to
the empirical risk $R_N$,
and the third term converges to zero,
because
$R_N(\phi_m^*)$ is a finite sample approximation
of $R(\phi_m^*)$.

\cref{fig:risk_illustration} shows a schematic illustration of the estimators
 and risks appearing in the decompositions~\cref{eq:errdec} and~\cref{eq:genErrdec}, also including the regularized version $\tilde\phi_m^N$ of $\phi_m^N$.

%% file: tikz/risk_contours.tex
\begin{tikzpicture}[x=1pt,y=1pt]
\definecolor{fillColor}{RGB}{255,255,255}
\path[use as bounding box,fill=fillColor,fill opacity=0.00] (0,0) rectangle (289.08,289.08);
\begin{scope}
\path[clip] (  0.00,  0.00) rectangle (289.08,289.08);
\definecolor{drawColor}{RGB}{0,0,0}

\node[text=drawColor,anchor=base,inner sep=0pt, outer sep=0pt, scale=  1.00] at (231.26, 46.08) {$\mathcal{F}$};
\definecolor{drawColor}{RGB}{0,0,139}

\path[draw=drawColor,line width= 0.8pt,line join=round,line cap=round] (199.22,202.36) --
	(199.84,204.03) --
	(200.29,205.78) --
	(200.50,207.57) --
	(200.47,209.37) --
	(200.16,211.13) --
	(199.58,212.82) --
	(198.73,214.39) --
	(197.64,215.81) --
	(196.35,217.07) --
	(194.90,218.15) --
	(193.35,219.06) --
	(191.74,219.79) --
	(190.09,220.37) --
	(188.45,220.82) --
	(186.82,221.14) --
	(185.24,221.35) --
	(183.70,221.48) --
	(182.21,221.54) --
	(180.77,221.53) --
	(179.39,221.46) --
	(178.06,221.35) --
	(176.77,221.20) --
	(175.53,221.01) --
	(174.33,220.80) --
	(173.16,220.55) --
	(172.02,220.28) --
	(170.91,219.98) --
	(169.83,219.66) --
	(168.76,219.32) --
	(167.70,218.95) --
	(166.66,218.56) --
	(165.63,218.14) --
	(164.60,217.69) --
	(163.57,217.21) --
	(162.54,216.70) --
	(161.50,216.14) --
	(160.46,215.55) --
	(159.41,214.92) --
	(158.35,214.23) --
	(157.28,213.49) --
	(156.20,212.68) --
	(155.11,211.81) --
	(154.03,210.86) --
	(152.94,209.82) --
	(151.87,208.69) --
	(150.82,207.47) --
	(149.82,206.14) --
	(148.89,204.70) --
	(148.05,203.16) --
	(147.34,201.53) --
	(146.81,199.81) --
	(146.47,198.04) --
	(146.38,196.24) --
	(146.55,194.46) --
	(146.99,192.73) --
	(147.71,191.09) --
	(148.69,189.59) --
	(149.88,188.25) --
	(151.25,187.08) --
	(152.76,186.08) --
	(154.35,185.27) --
	(155.98,184.61) --
	(157.63,184.10) --
	(159.27,183.72) --
	(160.87,183.45) --
	(162.44,183.28) --
	(163.95,183.19) --
	(165.41,183.17) --
	(166.82,183.21) --
	(168.18,183.30) --
	(169.49,183.43) --
	(170.75,183.60) --
	(171.97,183.80) --
	(173.16,184.03) --
	(174.31,184.29) --
	(175.43,184.58) --
	(176.53,184.88) --
	(177.61,185.22) --
	(178.67,185.57) --
	(179.71,185.95) --
	(180.75,186.36) --
	(181.79,186.80) --
	(182.81,187.26) --
	(183.84,187.76) --
	(184.88,188.29) --
	(185.92,188.86) --
	(186.96,189.47) --
	(188.02,190.13) --
	(189.08,190.85) --
	(190.16,191.62) --
	(191.24,192.46) --
	(192.33,193.37) --
	(193.41,194.36) --
	(194.49,195.44) --
	(195.56,196.62) --
	(196.58,197.90) --
	(197.55,199.28) --
	(198.44,200.77) --
	(199.22,202.36);

\path[draw=drawColor,line width= 0.8pt,line join=round,line cap=round] (224.99,202.36) --
	(226.24,205.71) --
	(227.12,209.21) --
	(227.56,212.79) --
	(227.49,216.38) --
	(226.87,219.90) --
	(225.70,223.28) --
	(224.00,226.42) --
	(221.82,229.27) --
	(219.25,231.79) --
	(216.36,233.95) --
	(213.26,235.76) --
	(210.02,237.23) --
	(206.73,238.39) --
	(203.44,239.28) --
	(200.20,239.92) --
	(197.02,240.35) --
	(193.94,240.61) --
	(190.97,240.72) --
	(188.10,240.70) --
	(185.33,240.57) --
	(182.66,240.34) --
	(180.09,240.04) --
	(177.61,239.67) --
	(175.20,239.24) --
	(172.87,238.75) --
	(170.60,238.20) --
	(168.38,237.61) --
	(166.20,236.97) --
	(164.07,236.28) --
	(161.96,235.55) --
	(159.88,234.76) --
	(157.81,233.92) --
	(155.74,233.02) --
	(153.68,232.06) --
	(151.62,231.03) --
	(149.55,229.93) --
	(147.47,228.75) --
	(145.37,227.48) --
	(143.25,226.11) --
	(141.11,224.62) --
	(138.95,223.01) --
	(136.78,221.26) --
	(134.60,219.36) --
	(132.43,217.28) --
	(130.29,215.03) --
	(128.20,212.58) --
	(126.19,209.92) --
	(124.33,207.05) --
	(122.65,203.97) --
	(121.24,200.70) --
	(120.16,197.27) --
	(119.49,193.72) --
	(119.30,190.13) --
	(119.64,186.56) --
	(120.54,183.10) --
	(121.98,179.83) --
	(123.93,176.83) --
	(126.32,174.14) --
	(129.06,171.80) --
	(132.07,169.81) --
	(135.24,168.18) --
	(138.51,166.86) --
	(141.81,165.84) --
	(145.08,165.09) --
	(148.30,164.55) --
	(151.42,164.21) --
	(154.45,164.03) --
	(157.38,163.99) --
	(160.20,164.07) --
	(162.91,164.25) --
	(165.53,164.51) --
	(168.06,164.85) --
	(170.50,165.25) --
	(172.87,165.71) --
	(175.17,166.23) --
	(177.41,166.80) --
	(179.61,167.41) --
	(181.76,168.08) --
	(183.89,168.79) --
	(185.98,169.55) --
	(188.06,170.37) --
	(190.12,171.24) --
	(192.18,172.16) --
	(194.24,173.16) --
	(196.31,174.22) --
	(198.38,175.36) --
	(200.48,176.58) --
	(202.59,177.91) --
	(204.71,179.33) --
	(206.86,180.88) --
	(209.03,182.56) --
	(211.20,184.38) --
	(213.38,186.37) --
	(215.54,188.53) --
	(217.66,190.88) --
	(219.72,193.44) --
	(221.66,196.20) --
	(223.44,199.18) --
	(224.99,202.36);

\path[draw=drawColor,line width= 0.8pt,line join=round,line cap=round] (250.76,202.36) --
	(252.64,207.39) --
	(253.96,212.63) --
	(254.62,218.00) --
	(254.51,223.39) --
	(253.59,228.68) --
	(251.83,233.74) --
	(249.28,238.45) --
	(246.01,242.73) --
	(242.15,246.51) --
	(237.82,249.75) --
	(233.16,252.46) --
	(228.31,254.67) --
	(223.38,256.41) --
	(218.44,257.74) --
	(213.57,258.70) --
	(208.81,259.35) --
	(204.19,259.74) --
	(199.73,259.90) --
	(195.42,259.87) --
	(191.27,259.67) --
	(187.27,259.34) --
	(183.42,258.89) --
	(179.69,258.33) --
	(176.08,257.68) --
	(172.58,256.94) --
	(169.17,256.13) --
	(165.84,255.24) --
	(162.58,254.28) --
	(159.38,253.25) --
	(156.22,252.14) --
	(153.09,250.96) --
	(149.98,249.70) --
	(146.89,248.35) --
	(143.80,246.91) --
	(140.71,245.37) --
	(137.60,243.72) --
	(134.48,241.95) --
	(131.33,240.04) --
	(128.15,237.98) --
	(124.94,235.75) --
	(121.70,233.34) --
	(118.45,230.71) --
	(115.18,227.86) --
	(111.92,224.75) --
	(108.71,221.37) --
	(105.57,217.69) --
	(102.56,213.70) --
	( 99.76,209.39) --
	( 97.25,204.77) --
	( 95.14,199.87) --
	( 93.52,194.72) --
	( 92.52,189.41) --
	( 92.23,184.01) --
	( 92.74,178.66) --
	( 94.08,173.47) --
	( 96.24,168.57) --
	( 99.17,164.06) --
	(102.75,160.03) --
	(106.87,156.52) --
	(111.37,153.54) --
	(116.14,151.09) --
	(121.05,149.12) --
	(125.99,147.59) --
	(130.90,146.45) --
	(135.72,145.65) --
	(140.41,145.14) --
	(144.96,144.87) --
	(149.34,144.81) --
	(153.57,144.93) --
	(157.64,145.19) --
	(161.57,145.59) --
	(165.36,146.09) --
	(169.02,146.70) --
	(172.58,147.39) --
	(176.03,148.17) --
	(179.40,149.02) --
	(182.69,149.94) --
	(185.92,150.94) --
	(189.10,152.01) --
	(192.25,153.15) --
	(195.36,154.37) --
	(198.46,155.67) --
	(201.55,157.07) --
	(204.64,158.56) --
	(207.74,160.15) --
	(210.85,161.86) --
	(213.99,163.70) --
	(217.15,165.68) --
	(220.35,167.82) --
	(223.57,170.14) --
	(226.82,172.66) --
	(230.08,175.40) --
	(233.35,178.38) --
	(236.59,181.62) --
	(239.77,185.15) --
	(242.85,188.98) --
	(245.76,193.13) --
	(248.43,197.59) --
	(250.76,202.36);
\definecolor{drawColor}{RGB}{255,165,0}

\path[draw=drawColor,line width= 0.8pt,line join=round,line cap=round] (152.69,138.76) --
	(153.24,140.23) --
	(153.62,141.76) --
	(153.81,143.33) --
	(153.78,144.91) --
	(153.51,146.45) --
	(153.00,147.93) --
	(152.26,149.31) --
	(151.30,150.56) --
	(150.17,151.67) --
	(148.90,152.61) --
	(147.54,153.41) --
	(146.13,154.05) --
	(144.68,154.56) --
	(143.24,154.95) --
	(141.82,155.23) --
	(140.42,155.42) --
	(139.07,155.53) --
	(137.77,155.58) --
	(136.51,155.57) --
	(135.30,155.51) --
	(134.13,155.42) --
	(133.00,155.28) --
	(131.91,155.12) --
	(130.86,154.93) --
	(129.83,154.72) --
	(128.84,154.48) --
	(127.86,154.22) --
	(126.91,153.94) --
	(125.97,153.64) --
	(125.05,153.31) --
	(124.13,152.97) --
	(123.23,152.60) --
	(122.32,152.21) --
	(121.42,151.78) --
	(120.52,151.33) --
	(119.61,150.85) --
	(118.69,150.33) --
	(117.77,149.78) --
	(116.84,149.17) --
	(115.90,148.52) --
	(114.96,147.82) --
	(114.01,147.05) --
	(113.05,146.21) --
	(112.10,145.30) --
	(111.16,144.32) --
	(110.24,143.24) --
	(109.36,142.07) --
	(108.55,140.82) --
	(107.81,139.47) --
	(107.19,138.03) --
	(106.72,136.53) --
	(106.43,134.97) --
	(106.34,133.40) --
	(106.49,131.83) --
	(106.88,130.31) --
	(107.52,128.88) --
	(108.37,127.56) --
	(109.42,126.38) --
	(110.62,125.36) --
	(111.94,124.49) --
	(113.33,123.77) --
	(114.77,123.19) --
	(116.21,122.75) --
	(117.65,122.41) --
	(119.06,122.18) --
	(120.43,122.03) --
	(121.76,121.95) --
	(123.04,121.93) --
	(124.28,121.97) --
	(125.47,122.05) --
	(126.61,122.16) --
	(127.72,122.31) --
	(128.79,122.49) --
	(129.83,122.69) --
	(130.84,122.92) --
	(131.83,123.17) --
	(132.79,123.44) --
	(133.73,123.73) --
	(134.66,124.04) --
	(135.58,124.37) --
	(136.49,124.73) --
	(137.40,125.11) --
	(138.30,125.52) --
	(139.20,125.95) --
	(140.11,126.42) --
	(141.02,126.92) --
	(141.94,127.46) --
	(142.86,128.04) --
	(143.80,128.66) --
	(144.74,129.34) --
	(145.69,130.08) --
	(146.64,130.88) --
	(147.60,131.75) --
	(148.54,132.70) --
	(149.48,133.73) --
	(150.38,134.85) --
	(151.23,136.06) --
	(152.01,137.37) --
	(152.69,138.76);

\path[draw=drawColor,line width= 0.8pt,line join=round,line cap=round] (175.29,138.76) --
	(176.39,141.70) --
	(177.16,144.77) --
	(177.54,147.91) --
	(177.48,151.06) --
	(176.94,154.15) --
	(175.92,157.11) --
	(174.42,159.86) --
	(172.51,162.36) --
	(170.25,164.57) --
	(167.72,166.47) --
	(165.00,168.05) --
	(162.16,169.34) --
	(159.28,170.36) --
	(156.39,171.14) --
	(153.55,171.70) --
	(150.76,172.08) --
	(148.06,172.31) --
	(145.45,172.40) --
	(142.93,172.38) --
	(140.51,172.27) --
	(138.17,172.08) --
	(135.91,171.81) --
	(133.74,171.48) --
	(131.63,171.10) --
	(129.58,170.67) --
	(127.59,170.20) --
	(125.64,169.68) --
	(123.73,169.12) --
	(121.86,168.51) --
	(120.01,167.87) --
	(118.18,167.18) --
	(116.37,166.44) --
	(114.56,165.65) --
	(112.75,164.81) --
	(110.94,163.91) --
	(109.13,162.94) --
	(107.30,161.91) --
	(105.46,160.79) --
	(103.60,159.59) --
	(101.72,158.28) --
	( 99.83,156.87) --
	( 97.93,155.34) --
	( 96.02,153.67) --
	( 94.11,151.85) --
	( 92.23,149.87) --
	( 90.40,147.72) --
	( 88.64,145.39) --
	( 87.00,142.87) --
	( 85.54,140.17) --
	( 84.30,137.30) --
	( 83.35,134.30) --
	( 82.77,131.19) --
	( 82.60,128.03) --
	( 82.90,124.90) --
	( 83.68,121.87) --
	( 84.95,119.00) --
	( 86.65,116.37) --
	( 88.75,114.01) --
	( 91.16,111.96) --
	( 93.79,110.22) --
	( 96.58,108.78) --
	( 99.45,107.63) --
	(102.34,106.74) --
	(105.21,106.07) --
	(108.03,105.60) --
	(110.77,105.30) --
	(113.43,105.15) --
	(115.99,105.11) --
	(118.46,105.18) --
	(120.85,105.34) --
	(123.14,105.57) --
	(125.36,105.86) --
	(127.50,106.22) --
	(129.58,106.62) --
	(131.60,107.07) --
	(133.56,107.57) --
	(135.49,108.11) --
	(137.38,108.70) --
	(139.24,109.32) --
	(141.08,109.99) --
	(142.90,110.70) --
	(144.71,111.46) --
	(146.52,112.28) --
	(148.32,113.15) --
	(150.13,114.08) --
	(151.96,115.08) --
	(153.79,116.16) --
	(155.64,117.32) --
	(157.51,118.57) --
	(159.39,119.92) --
	(161.29,121.40) --
	(163.20,123.00) --
	(165.11,124.74) --
	(167.00,126.63) --
	(168.86,128.70) --
	(170.66,130.94) --
	(172.37,133.36) --
	(173.93,135.97) --
	(175.29,138.76);

\path[draw=drawColor,line width= 0.8pt,line join=round,line cap=round] (197.89,138.76) --
	(199.54,143.17) --
	(200.70,147.77) --
	(201.27,152.48) --
	(201.18,157.20) --
	(200.37,161.84) --
	(198.83,166.28) --
	(196.59,170.42) --
	(193.73,174.17) --
	(190.34,177.48) --
	(186.54,180.32) --
	(182.46,182.70) --
	(178.20,184.64) --
	(173.87,186.17) --
	(169.55,187.33) --
	(165.28,188.17) --
	(161.10,188.75) --
	(157.05,189.08) --
	(153.13,189.22) --
	(149.36,189.20) --
	(145.72,189.03) --
	(142.21,188.73) --
	(138.83,188.34) --
	(135.56,187.85) --
	(132.40,187.28) --
	(129.33,186.63) --
	(126.34,185.92) --
	(123.42,185.14) --
	(120.56,184.30) --
	(117.75,183.39) --
	(114.97,182.42) --
	(112.23,181.39) --
	(109.51,180.28) --
	(106.80,179.10) --
	(104.09,177.84) --
	(101.37,176.49) --
	( 98.65,175.04) --
	( 95.91,173.48) --
	( 93.15,171.81) --
	( 90.36,170.00) --
	( 87.54,168.05) --
	( 84.70,165.93) --
	( 81.85,163.63) --
	( 78.98,161.12) --
	( 76.13,158.40) --
	( 73.31,155.43) --
	( 70.55,152.20) --
	( 67.92,148.71) --
	( 65.46,144.93) --
	( 63.26,140.88) --
	( 61.40,136.58) --
	( 59.99,132.06) --
	( 59.11,127.40) --
	( 58.86,122.67) --
	( 59.30,117.98) --
	( 60.48,113.42) --
	( 62.38,109.13) --
	( 64.94,105.17) --
	( 68.08,101.64) --
	( 71.69, 98.56) --
	( 75.65, 95.95) --
	( 79.83, 93.79) --
	( 84.13, 92.07) --
	( 88.47, 90.73) --
	( 92.77, 89.73) --
	( 97.00, 89.02) --
	(101.11, 88.57) --
	(105.10, 88.34) --
	(108.95, 88.29) --
	(112.65, 88.39) --
	(116.23, 88.62) --
	(119.67, 88.97) --
	(122.99, 89.41) --
	(126.21, 89.94) --
	(129.32, 90.55) --
	(132.35, 91.23) --
	(135.30, 91.98) --
	(138.19, 92.79) --
	(141.03, 93.66) --
	(143.82, 94.60) --
	(146.57, 95.60) --
	(149.31, 96.67) --
	(152.02, 97.82) --
	(154.73, 99.04) --
	(157.44,100.34) --
	(160.16,101.74) --
	(162.89,103.24) --
	(165.64,104.86) --
	(168.42,106.59) --
	(171.22,108.47) --
	(174.05,110.51) --
	(176.89,112.72) --
	(179.76,115.12) --
	(182.62,117.73) --
	(185.46,120.57) --
	(188.25,123.67) --
	(190.95,127.03) --
	(193.51,130.66) --
	(195.85,134.58) --
	(197.89,138.76);
\definecolor{drawColor}{RGB}{0,0,0}

\path[draw=drawColor,line width= 0.8pt,line join=round,line cap=round] (  0.00, 95.76) --
	(  2.92, 96.71) --
	(  5.84, 97.74) --
	(  8.76, 98.85) --
	( 11.68,100.05) --
	( 14.60,101.34) --
	( 17.52,102.73) --
	( 20.44,104.23) --
	( 23.36,105.83) --
	( 26.28,107.55) --
	( 29.20,109.39) --
	( 32.12,111.36) --
	( 35.04,113.47) --
	( 37.96,115.71) --
	( 40.88,118.11) --
	( 43.80,120.66) --
	( 46.72,123.37) --
	( 49.64,126.26) --
	( 52.56,129.32) --
	( 55.48,132.57) --
	( 58.40,136.01) --
	( 61.32,139.66) --
	( 64.24,143.53) --
	( 67.16,147.61) --
	( 70.08,151.93) --
	( 73.00,156.49) --
	( 75.92,161.30) --
	( 78.84,166.37) --
	( 81.76,171.71) --
	( 84.68,177.34) --
	( 87.60,183.26) --
	( 90.52,189.48) --
	( 93.44,196.03) --
	( 96.36,202.90) --
	( 99.28,210.11) --
	(102.20,217.68) --
	(105.12,225.61) --
	(108.04,233.93) --
	(110.96,242.64) --
	(113.88,251.75) --
	(116.80,261.29) --
	(119.72,271.26) --
	(122.64,281.68) --
	(125.56,292.57) --
	(128.48,303.94) --
	(131.40,315.81) --
	(134.32,328.19) --
	(137.24,341.10) --
	(140.16,354.56) --
	(143.08,368.58) --
	(146.00,383.18) --
	(148.92,398.38) --
	(151.84,414.19) --
	(154.76,430.65) --
	(157.68,447.75) --
	(160.60,465.54) --
	(163.52,484.01) --
	(166.44,503.20) --
	(169.36,523.12) --
	(172.28,543.80) --
	(175.20,565.26) --
	(178.12,587.51) --
	(181.04,610.59) --
	(183.96,634.50) --
	(186.88,659.28) --
	(189.80,684.96) --
	(192.72,711.54) --
	(195.64,739.06) --
	(198.56,767.54) --
	(201.48,797.01) --
	(204.40,827.49) --
	(207.32,859.01) --
	(210.24,891.59) --
	(213.16,925.26) --
	(216.08,960.05) --
	(219.00,995.99) --
	(221.92,1033.09) --
	(224.84,1071.40) --
	(227.76,1110.94) --
	(230.68,1151.74) --
	(233.60,1193.83) --
	(236.52,1237.25) --
	(239.44,1282.01) --
	(242.36,1328.15) --
	(245.28,1375.71) --
	(248.20,1424.72) --
	(249.40,1445.40);

\node[text=drawColor,anchor=base,inner sep=0pt, outer sep=0pt, scale=  1.00] at (  5.78, 85.98) {$\Phi$};
\definecolor{fillColor}{RGB}{0,0,0}

\path[draw=drawColor,line width= 0.4pt,line join=round,line cap=round,fill=fillColor] (173.45,202.36) circle (  2.25);

\path[draw=drawColor,line width= 0.4pt,line join=round,line cap=round,fill=fillColor] ( 95.87,201.72) circle (  2.25);

\node[text=drawColor,anchor=base west,inner sep=0pt, outer sep=0pt, scale=  1.00] at (177.20,207.87) {$f_m^*$};

\node[text=drawColor,anchor=base east,inner sep=0pt, outer sep=0pt, scale=  1.00] at ( 93.75,208.61) {$f_\phi^*$};
\definecolor{drawColor}{RGB}{255,165,0}
\definecolor{fillColor}{RGB}{255,165,0}

\path[draw=drawColor,line width= 0.4pt,line join=round,line cap=round,fill=fillColor] (130.09,138.76) circle (  2.25);
\definecolor{drawColor}{RGB}{0,0,0}
\definecolor{fillColor}{RGB}{0,0,0}

\path[draw=drawColor,line width= 0.4pt,line join=round,line cap=round,fill=fillColor] ( 67.59,148.24) circle (  2.25);

\node[text=drawColor,anchor=base east,inner sep=0pt, outer sep=0pt, scale=  1.00] at ( 64.27,155.12) {$f_{\phi^N}$};

\path[draw=drawColor,line width= 0.4pt,line join=round,line cap=round,fill=fillColor] ( 80.32,169.04) circle (  2.25);

\node[text=drawColor,anchor=base east,inner sep=0pt, outer sep=0pt, scale=  1.00] at ( 76.99,175.92) {$f_{\tilde{\phi}^N}$};
\end{scope}
\end{tikzpicture}

%% file: tikz/risk_parabolas.tex
\begin{tikzpicture}[x=1pt,y=1pt]
\definecolor{fillColor}{RGB}{255,255,255}
\path[use as bounding box,fill=fillColor,fill opacity=0.00] (0,0) rectangle (289.08,289.08);
\begin{scope}
\path[clip] (  0.00,  0.00) rectangle (289.08,289.08);
\definecolor{drawColor}{RGB}{0,0,0}

\path[draw=drawColor,line width= 0.8pt,line join=round,line cap=round] ( 21.41, 21.41) -- (289.08, 21.41);

\path[draw=drawColor,line width= 0.8pt,line join=round,line cap=round] (285.95, 19.61) --
	(289.08, 21.41) --
	(285.95, 23.22);

\path[draw=drawColor,line width= 0.8pt,line join=round,line cap=round] ( 21.41, 21.41) -- ( 21.41,289.08);

\path[draw=drawColor,line width= 0.8pt,line join=round,line cap=round] ( 23.22,285.95) --
	( 21.41,289.08) --
	( 19.61,285.95);

\node[text=drawColor,anchor=base,inner sep=0pt, outer sep=0pt, scale=  1.00] at (281.05,  9.67) {$\mathcal{F}$};

\node[text=drawColor,anchor=base west,inner sep=0pt, outer sep=0pt, scale=  1.00] at ( 27.41,278.75) {$\mathbb{R}$};
\definecolor{drawColor}{RGB}{0,0,139}

\path[draw=drawColor,line width= 0.8pt,line join=round,line cap=round] ( 21.41,323.88) --
	( 24.12,316.36) --
	( 26.82,308.95) --
	( 29.52,301.66) --
	( 32.23,294.47) --
	( 34.93,287.39) --
	( 37.64,280.42) --
	( 40.34,273.56) --
	( 43.04,266.81) --
	( 45.75,260.17) --
	( 48.45,253.63) --
	( 51.15,247.21) --
	( 53.86,240.90) --
	( 56.56,234.69) --
	( 59.27,228.60) --
	( 61.97,222.61) --
	( 64.67,216.73) --
	( 67.38,210.97) --
	( 70.08,205.31) --
	( 72.78,199.76) --
	( 75.49,194.32) --
	( 78.19,188.99) --
	( 80.89,183.76) --
	( 83.60,178.65) --
	( 86.30,173.65) --
	( 89.01,168.76) --
	( 91.71,163.97) --
	( 94.41,159.29) --
	( 97.12,154.73) --
	( 99.82,150.27) --
	(102.52,145.92) --
	(105.23,141.69) --
	(107.93,137.56) --
	(110.64,133.54) --
	(113.34,129.63) --
	(116.04,125.82) --
	(118.75,122.13) --
	(121.45,118.55) --
	(124.15,115.07) --
	(126.86,111.71) --
	(129.56,108.45) --
	(132.27,105.31) --
	(134.97,102.27) --
	(137.67, 99.34) --
	(140.38, 96.53) --
	(143.08, 93.82) --
	(145.78, 91.22) --
	(148.49, 88.73) --
	(151.19, 86.34) --
	(153.89, 84.07) --
	(156.60, 81.91) --
	(159.30, 79.86) --
	(162.01, 77.91) --
	(164.71, 76.08) --
	(167.41, 74.35) --
	(170.12, 72.73) --
	(172.82, 71.23) --
	(175.52, 69.83) --
	(178.23, 68.54) --
	(180.93, 67.36) --
	(183.64, 66.29) --
	(186.34, 65.33) --
	(189.04, 64.47) --
	(191.75, 63.73) --
	(194.45, 63.10) --
	(197.15, 62.57) --
	(199.86, 62.16) --
	(202.56, 61.85) --
	(205.27, 61.66) --
	(207.97, 61.57) --
	(210.67, 61.59) --
	(213.38, 61.72) --
	(216.08, 61.96) --
	(218.78, 62.31) --
	(221.49, 62.77) --
	(224.19, 63.34) --
	(226.89, 64.02) --
	(229.60, 64.80) --
	(232.30, 65.70) --
	(235.01, 66.70) --
	(237.71, 67.82) --
	(240.41, 69.04) --
	(243.12, 70.37) --
	(245.82, 71.82) --
	(248.52, 73.37) --
	(251.23, 75.03) --
	(253.93, 76.80) --
	(256.64, 78.68) --
	(259.34, 80.66) --
	(262.04, 82.76) --
	(264.75, 84.97) --
	(267.45, 87.28) --
	(270.15, 89.71) --
	(272.86, 92.24) --
	(275.56, 94.89) --
	(278.27, 97.64) --
	(280.97,100.50) --
	(283.67,103.47) --
	(286.38,106.55) --
	(289.08,109.74);
\definecolor{drawColor}{RGB}{255,165,0}

\path[draw=drawColor,line width= 0.8pt,line join=round,line cap=round] ( 21.41, 30.78) --
	( 24.12, 28.98) --
	( 26.82, 27.38) --
	( 29.52, 25.96) --
	( 32.23, 24.74) --
	( 34.93, 23.71) --
	( 37.64, 22.87) --
	( 40.34, 22.22) --
	( 43.04, 21.76) --
	( 45.75, 21.49) --
	( 48.45, 21.41) --
	( 51.15, 21.53) --
	( 53.86, 21.83) --
	( 56.56, 22.33) --
	( 59.27, 23.02) --
	( 61.97, 23.90) --
	( 64.67, 24.97) --
	( 67.38, 26.23) --
	( 70.08, 27.68) --
	( 72.78, 29.33) --
	( 75.49, 31.16) --
	( 78.19, 33.19) --
	( 80.89, 35.41) --
	( 83.60, 37.82) --
	( 86.30, 40.42) --
	( 89.01, 43.21) --
	( 91.71, 46.19) --
	( 94.41, 49.36) --
	( 97.12, 52.73) --
	( 99.82, 56.28) --
	(102.52, 60.03) --
	(105.23, 63.97) --
	(107.93, 68.10) --
	(110.64, 72.42) --
	(113.34, 76.93) --
	(116.04, 81.63) --
	(118.75, 86.53) --
	(121.45, 91.61) --
	(124.15, 96.89) --
	(126.86,102.36) --
	(129.56,108.01) --
	(132.27,113.86) --
	(134.97,119.91) --
	(137.67,126.14) --
	(140.38,132.56) --
	(143.08,139.18) --
	(145.78,145.98) --
	(148.49,152.98) --
	(151.19,160.17) --
	(153.89,167.55) --
	(156.60,175.12) --
	(159.30,182.88) --
	(162.01,190.83) --
	(164.71,198.97) --
	(167.41,207.31) --
	(170.12,215.84) --
	(172.82,224.55) --
	(175.52,233.46) --
	(178.23,242.56) --
	(180.93,251.85) --
	(183.64,261.33) --
	(186.34,271.01) --
	(189.04,280.87) --
	(191.75,290.93) --
	(194.45,301.17) --
	(197.15,311.61) --
	(199.86,322.24) --
	(202.56,333.06) --
	(205.27,344.07) --
	(207.97,355.28) --
	(210.67,366.67) --
	(213.38,378.25) --
	(216.08,390.03) --
	(218.78,402.00) --
	(221.49,414.16) --
	(224.19,426.51) --
	(226.89,439.05) --
	(229.60,451.78) --
	(232.30,464.70) --
	(235.01,477.82) --
	(237.71,491.12) --
	(240.41,504.62) --
	(243.12,518.31) --
	(245.82,532.18) --
	(248.52,546.26) --
	(251.23,560.52) --
	(253.93,574.97) --
	(256.64,589.61) --
	(259.34,604.45) --
	(262.04,619.47) --
	(264.75,634.69) --
	(267.45,650.10) --
	(270.15,665.70) --
	(272.86,681.49) --
	(275.56,697.47) --
	(278.27,713.64) --
	(280.97,730.01) --
	(283.67,746.56) --
	(286.38,763.31) --
	(289.08,780.25);
\definecolor{fillColor}{RGB}{255,165,0}

\path[draw=drawColor,line width= 0.4pt,line join=round,line cap=round,fill=fillColor] ( 48.18, 21.41) circle (  2.25);
\definecolor{drawColor}{RGB}{0,0,0}
\definecolor{fillColor}{RGB}{0,0,0}

\path[draw=drawColor,line width= 0.4pt,line join=round,line cap=round,fill=fillColor] (208.78, 61.56) circle (  2.25);

\node[text=drawColor,anchor=base,inner sep=0pt, outer sep=0pt, scale=  1.00] at (208.78, 67.56) {$R(f_m^*)$};

\path[draw=drawColor,line width= 0.8pt,dash pattern=on 4pt off 4pt ,line join=round,line cap=round] (208.78, 21.41) --
	(208.78, 61.56);

\node[text=drawColor,anchor=base,inner sep=0pt, outer sep=0pt, scale=  1.00] at (208.78,  9.67) {$f_m^*$};

\path[draw=drawColor,line width= 0.4pt,line join=round,line cap=round,fill=fillColor] (155.25, 82.98) circle (  2.25);

\path[draw=drawColor,line width= 0.4pt,line join=round,line cap=round,fill=fillColor] (155.25,171.31) circle (  2.25);

\node[text=drawColor,anchor=base west,inner sep=0pt, outer sep=0pt, scale=  1.00] at (161.25, 80.68) {$R(\phi^*)$};

\node[text=drawColor,anchor=base west,inner sep=0pt, outer sep=0pt, scale=  1.00] at (161.25,169.01) {$R_N(\phi^*)$};

\node[text=drawColor,anchor=base,inner sep=0pt, outer sep=0pt, scale=  1.00] at (155.25,  9.67) {$f_{\phi^*}$};

\path[draw=drawColor,line width= 0.8pt,dash pattern=on 4pt off 4pt ,line join=round,line cap=round] (155.25, 21.41) --
	(155.25,171.31);

\path[draw=drawColor,line width= 0.4pt,line join=round,line cap=round,fill=fillColor] (101.71, 58.89) circle (  2.25);

\path[draw=drawColor,line width= 0.4pt,line join=round,line cap=round,fill=fillColor] (101.71,147.22) circle (  2.25);

\node[text=drawColor,anchor=base east,inner sep=0pt, outer sep=0pt, scale=  1.00] at ( 95.71, 56.59) {$R_N(\phi^N)$};

\node[text=drawColor,anchor=base east,inner sep=0pt, outer sep=0pt, scale=  1.00] at ( 95.71,144.92) {$R(\phi^N)$};

\node[text=drawColor,anchor=base,inner sep=0pt, outer sep=0pt, scale=  1.00] at (101.71,  9.67) {$f_{\phi^N}$};

\path[draw=drawColor,line width= 0.8pt,dash pattern=on 4pt off 4pt ,line join=round,line cap=round] (101.71, 21.41) --
	(101.71,147.22);

\path[draw=drawColor,line width= 0.4pt,line join=round,line cap=round,fill=fillColor] (125.80,100.20) circle (  2.25);

\path[draw=drawColor,line width= 0.4pt,line join=round,line cap=round,fill=fillColor] (125.80,113.01) circle (  2.25);

\node[text=drawColor,anchor=base,inner sep=0pt, outer sep=0pt, scale=  1.00] at (125.80,  9.67) {$f_{\tilde{\phi}^N}$};

\path[draw=drawColor,line width= 0.8pt,dash pattern=on 4pt off 4pt ,line join=round,line cap=round] (125.80, 21.41) --
	(125.80,113.01);
\end{scope}
\end{tikzpicture}

%% file: tex/02c_examples.tex
\subsection{Examples}
\label{sec:exampleModels}

Before presenting the main theoretical results,
we introduce several models that are commonly used in the literature of neural estimators. Throughout the paper, we will verify the main assumptions on these examples to show the broad applicability of our results.

Neural estimators are particularly useful in settings
where the likelihood function is intractable
but simulation from the model is possible.
A typical example of this is parameter estimation in
for spatial or spatio-temporal Gaussian random field models,
where likelihood inference becomes computationally infeasible in higher dimensions~\citep{gerber2021fast,sainsbury2023likelihood}.
Another popular example are max-stable distributions~\cite{deh1984} or max-stable random fields 
\citep{lenzi2021neural}, where evaluation of the full likelihood is impossible already in moderate dimensions \citep{dombry2017bayesian}, but exact simulation is feasible \citep{dom2016}.

In \cref{app:linear} we discuss the illustrative example of estimation of the parameters of a linear model. Since in this case we have more explicit representations of risk terms in the decomposition~\cref{eq:errdec}, it provides intuition on the different error contributions in neural estimation.

\subsubsection{Gaussian random fields}
\label{sec:gauss}

A Gaussian random field is a stochastic process $W = \{W(s): s\in S\}$ on some domain $S\subseteq \mathbb R^k$
such that for any finite set of points $s_1, \dots, s_d \in S$,
the random vector $(W(s_1), \dots, W(s_d))$ follows a multivariate normal distribution.
Without losing generality, we restrict the model to have a known constant mean zero, and we consider stationary, isotropic covariance function.
The distribution is then fully characterized
by the covariance function $\Cov(W(s), W(s')) = C(\norm{s-s'})$.
Neural estimators are parametric methods, and we therefore consider a parametric family consisting of a continuous covariance function $\tilde C_{\tilde \theta}(h)$ with nugget effect $\tau^2 \indic{h=0}$, that is,
\begin{align*}
    C_\theta(h)
    =
    \tilde C_{\tilde \theta}(h) + \tau^2 \indic{h=0}
    .
\end{align*}
This model is
parametrized by $\theta = (\tau, \tilde \theta) \in \Theta \subset \R^p$, and for a fixed set of points $s_1, \dots, s_d\in S$,
we have $Z = (W(s_1), \dots, W(s_d)) \sim \normal\lr{\zerovec, \Sigma_\theta}$,
with $(\Sigma_\theta)_{ij} = C_\theta(\norm{s_i - s_j})$.
The following two specific examples of covariance functions
will be considered in later sections.
The powered exponential model is a
Gaussian random field with
parameters $\tilde\theta = (\lambda, \alpha) \in (0, \infty) \times (0,2)$
and continuous part of the covariance function
\begin{align}
    \label{eq:poweredExponential}
    \tilde C_{\tilde \theta}(h)
    &=
    \exp\left\{-\lr{{h}/{\lambda}}^{\alpha}\right\}
    , \quad h\geq 0.
\end{align}
The \Matern{} model is a Gaussian random field with continuous part of the covariance
\begin{align}
    \label{eq:matern}
    \tilde C_{\tilde \theta}(h)
    &=
    \frac{2^{1-\nu}}{\Gamma(\nu)}
    \lr{\sqrt{2\nu} \frac{h}{\lambda}}^\nu
    K_\nu\lr{\sqrt{2\nu} \frac{h}{\lambda}}
    , \quad h\geq 0,
\end{align}
where $\Gamma$ is the Gamma function and $K_\nu$ is the Bessel function of the second kind of order $\nu$.
In line with \cite{gerber2021fast},
we consider a fixed smoothness parameter $\nu > 0$,
leaving the range parameter $\lambda \in (0, \infty)$ as the only free parameter; see the left-hand side of \cref{fig:examples} for a realization of the Mat\'ern field.

For Gaussian random fields evaluated at many spatial locations likelihood estimation can become prohibitively expensive \citep{rue2005gaussian}. The simulation of a reasonable amount samples from such processes remains however feasible.

\begin{figure}[tb!]
    \centering 
    \includegraphics[width=0.47\textwidth]{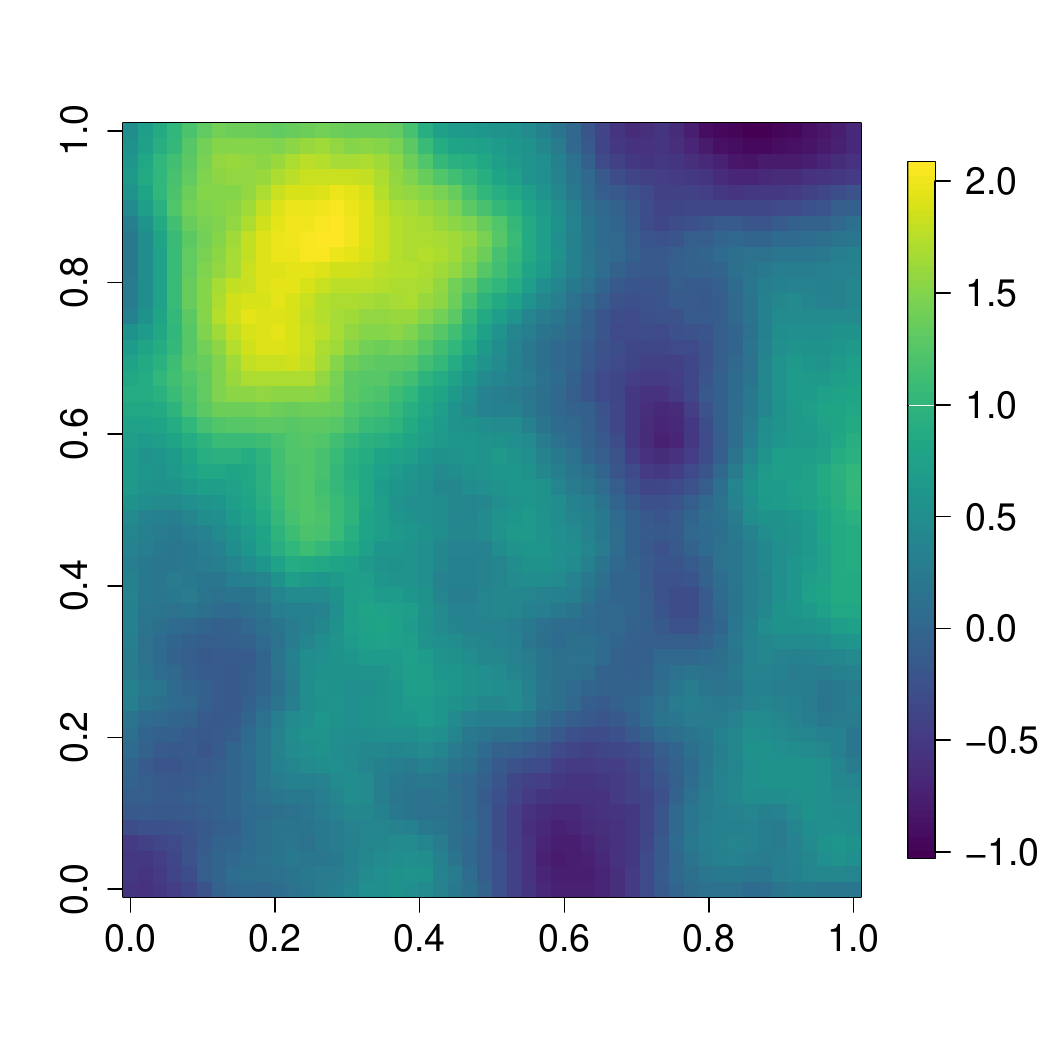}\hfill
        \includegraphics[width=0.47\textwidth]{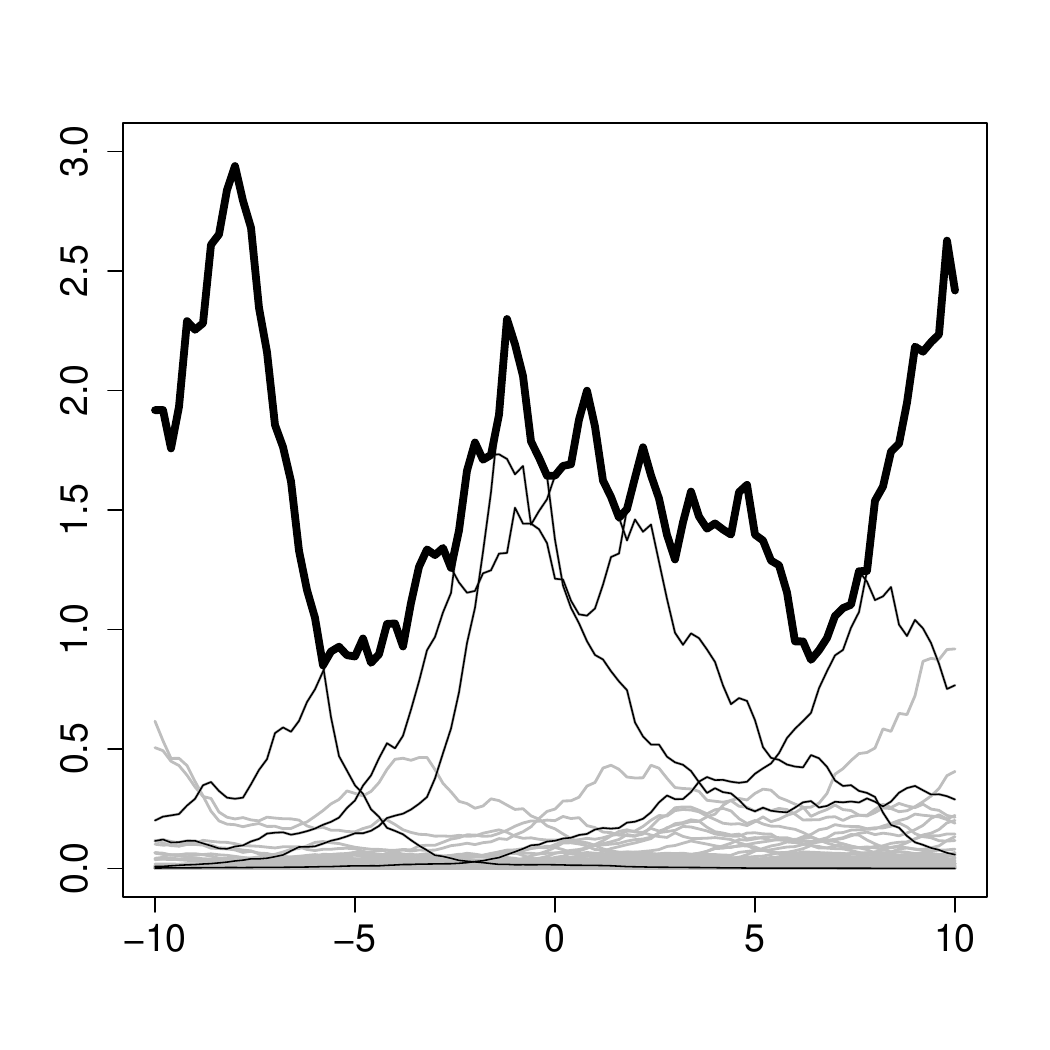}
        \caption{Left: realization of a Mat\'ern random field in $k=2$ dimensions with $\lambda=0.2$ and $\nu = 1.5$. Right: representation~\eqref{max_construction} of a max-stable process in $k=1$ dimension; the max-stable processes (thick black line) is  constructed from extremal functions (gray lines) that contributed to the pointwise maximum (thin back lines).}
        \label{fig:examples}
\end{figure}

\subsubsection{Max-stable distributions}
\label{sec:exampleModels:maxstable}
Max-stable distributions arise naturally in the study of extreme values of random processes. 
For independent copies $U_i$, $i\in\mathbb N$, of some stochastic process $U = \{U(s):s\in S\}$ on $S$, a max-stable process $W$ is the only possible limit in distribution of normalized maxima 
\begin{align}\label{max_conv}
    W(s)
    &=
    \lim_{m \to \infty}
    a_m(s)\inv
        \max_{i=1,\dots, m} \{
            U_i(s)      
        - b_m(s)
    \}
    ,\quad s\in S,
\end{align}
where $a_m(s)>0$ and $b_m(s)\in\mathbb R$ are suitable normalizing functions, and where the convergence is meant in the sense of finite dimensional distributions \citep{deh1984}.
Any max-stable process has the stochastic representation 
\begin{align}\label{max_construction}
    W(s)
    &=
    \max_{i \in \N}
    \xi_i
    Y_i(s)
    , \quad s\in S,
\end{align}
where $\xi_i$ are the points of a Poisson point process with intensity $\xi^{-2}d\xi$,
and $Y_i$, $i\in\mathbb N$, are independent copies of a random field $Y$, the so-called extremal functions \citep{dom2012}.
The right-hand side of \cref{fig:examples} illustrates this representation.

Choosing suitable families of extremal function $Y$ yields 
parametric models of max-stable processes. 
The most popular spatial model in the literature is the class of Brown--Resnick processes \citep{bro1977, kab2009},
where 
\begin{align}
    \label{eq:brownResnick}
    Y(s)
    &=
    \exp\{V(s) - \gamma(s)\}
    ,
\end{align}
is an intrinsically stationary
Gaussian random field
satisfying $V(0) = 0$ almost surely,
and with variogram
$\gamma_{c, \alpha}(h) = ch^\alpha$.
The parameter space for this model is
$\theta = (c, \alpha) \in \Theta = (0, \infty) \times (0,2)$,
and the marginal distributions are unit \Frechet{}.
The finite dimensional margin of the
Brown--Resnick process $W$ at locations $s_1, \dots, s_d\in S$ is a max-stable random vector $Z = (W(s_1), \dots, W(s_d))$ that follows a so-called H\"usler--Reiss distribution \citep{hueslerReiss1989}.
Similarly to normal distributions, the density of this model can be expressed in terms of a precision matrix \citep{HES2022} and is used for graphical modeling of multivariate extremes \citep{engelkeHitz2020}.

For a finite set as the domain $S =\{1,\dots, d\}$ in~\eqref{max_conv} we obtain a $d$-dimensional max-stable distribution $Z$. A classical parametric example is the logistic model \citep{gumbel1960} with unit Fréchet margins and
distribution function
\begin{align}
    \label{eq:logistic}
    \P_\theta\lr{
        Z_1 \leq z_1,\dots,Z_p
        \leq
        z_d
    }
    =
    \exp\lr{-\lr{
        z_1^{-\frac{1}{\theta}}
        + \cdots
        + z_d^{-\frac{1}{\theta}}
    }^\theta}
    ,
\end{align}
which is parametrized by a single parameter $\theta \in \Theta = (0,1)$.

The density of max-stable distributions is computationally intractable, since in $d$ dimensions it is a sum of $B_d$ terms, where $B_d$ is the $d$th Bell number \citep{wad2015}. This number grows super-exponentially in $d$ and makes full-likelihood inference impossible even in moderate dimensions. The reason is that many different extremal functions in the representation \cref{max_construction} can contribute to the maximum, explaining the explosion of terms; see the right-hand side of \cref{fig:examples} for an illustration and \cite{dom2016} for details.

While likelihood inference for max-stable processes is difficult, exact simulation is possible and computationally relatively efficient. In fact, the algorithm introduced in \citet{dom2016} requires in average the simulation of $d$ extremal functions $V$ to generate one realization of a $d$-dimensional max-stable random vector. For Brown--Resnick processes, for instance, this has to be multiplied by the simulation cost of a $d$-dimensional normal distribution.

%% file: tex/03_bayes.tex
\section{Bayes risk}
\label{sec:bayesRisk}
In this section we establish conditions for the Bayes Risk $R(f_m^*)$
to vanish.
This term only depends on the model family $\Pf_\theta$,
the prior distribution $\Pi$ of $\theta$, %
and the number $m$ of i.i.d. samples $(z^1, \dots, z^m) = z$
to which the Bayes estimator is applied.
The approximation quality of neural networks $f_\phi$
or the number $N$ of samples used to train each network
are irrelevant at this stage,
as they do not affect the Bayes estimator $f^*(z)$.

Consequently,
we can build on well-known results from Bayesian estimation theory,
and in particular the \BvM{} theorem.

\begin{theorem}[\BvM{}; Theorem~10.8 in \cite{vaart_1998}]
    \label{thm:bvm}
    Let $\set{p_\theta}_{\theta \in \Theta}$ be
    differentiable in quadratic mean at $\theta_0$ with non-singular
    Fisher information matrix $I_{\theta_0}$. Suppose that for every $\varepsilon > 0$ there exists a sequence of tests
    $(t_m^\varepsilon)_{m \in \mathbb{N}}$
    with
    \begin{align}
        \E_{\theta_0}[t_m^\varepsilon] \to 0
        , \quad \text{ and } \quad
        \sup_{\norm{\theta - \theta_0}> \varepsilon}
        \E_{\theta}[1-t_m^\varepsilon] \to 0
        , \quad \text{ as }  \quad
        m \to \infty.
    \end{align}

    Let $(Z_1,Z_2,\dots)$ be i.i.d. with respect to $p_{\theta_0}$ and parameter $\theta_0 \in \Theta$.
    Moreover, suppose the loss function $\ell$ satisfies $\ell(\theta, \theta_0) = \norm{\theta - \theta_0}^p$
    for some $p\geq 1$ and $\int_\Theta \norm{\theta}^p \pi(\theta) d\theta < \infty$. Let the density $\pi$ of the prior distribution be continuous and positive in a neighborhood of $\theta_0$.
    Then, the Bayes estimator is asymptotically normal and efficient.

\end{theorem}

\mysymbol{$t_m^\varepsilon$}{Test statistic for \BvM{} Theorem}

\cref{thm:bvm} allows conclusions about the point-wise risk
for a given true parameter $\theta_0$.
In order to extend these to the integrated Bayes risk $R(f_m^*)$,
we add the following assumption on the parameter space $\Theta$.

\begin{assumption} %
    \label{asm:parametersCompact}
    The parameter space $\Theta$
    of the considered model is compact.
    In particular, there exists $B < \infty$ with
    $\norm[\infty]{\theta} < B$ for all $\theta \in \Theta$.
\end{assumption}
\mysymbol{$B$}{Bound on parameter space $\Theta$}

This assumption will be used in many of the later results.
Some of these only require boundedness,
such as the following corollary about the Bayes risk $R(f_m^*)$,
but for simplicity we will always assume compactness throughout the paper.
Even though many models originally have unbounded or non-compact parameter spaces,
we can restrict them to compact subsets
thereby satisfying \cref{asm:parametersCompact}.
In practice, this is a very natural restriction,
as the range of ``reasonable'' parameter values is often known in advance.

\begin{corollary} %
    \provenlabel{cor:bayesRiskToZero}
    Let $\Theta$ be bounded (implied by \cref{asm:parametersCompact}),
    consider the quadratic loss from \cref{eq:quadraticLoss},
    and let the conditions from \cref{thm:bvm} be satisfied for some $\theta_0 \in \Theta$.
    Then
    \begin{align}
        \label{eq:bayesRiskPointwise}
        \lim_{m \to \infty}
        R_{\theta_0}(f_m^*)
        &=
        0.
    \end{align}
    If they are satisfied for all $\theta \in \Theta$,
    then the integrated risk converges as well:
    \begin{align}
        \label{eq:bayesRiskIntegrated}
        \lim_{m \to \infty}
        R(f_m^*)
        &=
        0 .
    \end{align}
\end{corollary}

The condition on the prior distribution in \cref{thm:bvm} can be satisfied
by a suitable choice for $\pi$ such as
the continuous uniform distribution, for instance; in the sequel we will assume
that the prior distribution is chosen in such a way.
The remaining conditions from \cref{thm:bvm} and \cref{cor:bayesRiskToZero}
have to be checked on a per-model basis.
The results below confirm the conditions for
the examples given in \cref{sec:exampleModels}.

\begin{lemma}
    \provenlabel{lemma:bvmGauss}
    The Gaussian random field models with
    powered exponential covariance function \cref{eq:poweredExponential} and the \Matern{} covariance function \cref{eq:matern},
    evaluated at a fixed set of points $s_1, \dots, s_d$,
    satisfy the conditions of \cref{cor:bayesRiskToZero}
    if the points are located such that there are at least two distinct, non-zero distances between them.
\end{lemma}

For max-stable distributions, the following result is obtained by using the tests constructed in \cite{dombry2017bayesian}.

\begin{lemma}
    \provenlabel{lemma:bvmEV}
    The logistic model from \cref{eq:logistic}
    and the \BrownResnick{} model from \cref{eq:brownResnick}
    satisfy the conditions of \cref{cor:bayesRiskToZero} %
    if the points are located such that there are at least two distinct, non-zero distances between them.
\end{lemma}

%% file: tex/04_approximation.tex
\section{Approximation error}
\label{sec:approximationError}

In the previous section we showed that,
under suitable assumptions,
the risk $R(f_m^*)$ of the Bayes estimator vanishes.
In the context of neural estimators,
this estimator is approximated by a neural network,
giving rise to the approximation error.
Assuming the quadratic loss function from \cref{eq:quadraticLoss},
this error can be expressed as
\begin{align}
    \label{eq:approxError}
    R(\phi_m^*) - R(f_m^*)
    &=
    \E\left[\norm[2]{
        f_{\phi_m^*}(Z)
        -
        f_m^*(Z)
    }^2\right]
    ;
\end{align}
see \cref{sec:proofApproxError} for details.
Note that without any restrictions,
the concept of an ``optimal'' network $\phi_m^*$ is not well-defined.
Instead, we show that the error term above can be made arbitrarily small
by using a suitable activation function,
large enough architecture for a given input size $z = (z^1, \dots, z^m)$,
and optimal parameter values.
The effect of training a model with this architecture using the empirical
risk $R_N$ instead of the population risk $R$
is investigated in \cref{sec:generalizationError}.

In order to bound the risk of the considered networks,
we make use of the universal approximation theorem.
For the sake of specificity,
we only consider the ReLU activation function, denoted by~$\sigma$.
However, the results given here and in \cref{sec:generalizationError}
also apply to other commonly used activation functions,
possibly with slight modifications.

\begin{theorem}[Cybenko, 1989] %
    \label{thm:cyb}
    For every continuous function $f: \R^D \to \R^p$,
    $M > 0$,
    and
    $\varepsilon > 0$
    there exists
    a shallow neural network parametrization $\phi = (w, \beta)$
    with one hidden layer of width $N_1 \in \NN$,
    such that
    \begin{align*}
        \sup_{x \in \left[-M,M\right]^D}
        \norm[2]{f(x) - f_\phi(x)}
        < \varepsilon
        .
    \end{align*}
\end{theorem}
\mysymbol{$D = md$}{Input dimension of neural network}
\mysymbol{$M$}{Bound on input space of neural network, bound in uniform tightness condition}

\cref{example:single_layer_NN} provides details for the parametrization
of such a shallow network. In order to apply this result---and others like it---to the setting of neural estimators,
we need to introduce some assumptions on the considered models.

\begin{assumption} %
    \label{asm:bayesContinuous}
    For any $m\in\mathbb N$, the Bayes estimator function
    $f_m^*(z) = \E(\theta | Z = z)$
    is continuous on $\mathcal Z^m$. %
\end{assumption}

This assumption ensures that the Bayes estimator function
is actually in the class of functions that can be approximated well
by neural networks.
While the Bayes estimator is usually not available in closed form,
the following lemma gives a sufficient condition for its continuity.

\begin{lemma}
    \provenlabel{lemma:bayesContinuousThree}
    Consider a compact parameter space $\Theta$
    and the quadratic loss function from \cref{eq:quadraticLoss}.
    If $p_\theta(z)$ is continuous in $\theta$ and $z$,
    then $f(z) = \E(\theta | Z = z)$ is continuous.
\end{lemma}

The statement of universal approximation in~\cref{thm:cyb} holds only on 
bounded sets $[-M,M]^D$. We therefore require an assumption that ensures 
that the observations of $Z$ are, with high probability, eventually
contained in a compactum, independently of the parameter~$\theta_0$.

\begin{assumption} %
    \label{asm:modelTight}
    The model is uniformly tight, that is, for every $\varepsilon > 0$
    there exists $M > 0$ such that
    \begin{align*}
        \sup_{\theta_0 \in \Theta}\P_{\theta_0}\left(\norm[\infty]{Z} > M\right)
        <
        \varepsilon        
        .
    \end{align*}
\end{assumption}
While the stronger assumption of bounded realizations of the model would imply the above condition,
it would be too restrictive for the examples of Gaussian processes or extreme value models described in \cref{sec:exampleModels}.
Instead, the error incurred by observations from outside the compactum $[-M,M]^D$
can be controlled by \cref{asm:parametersCompact},
which ensures that the model parameters, and hence the corresponding loss function,
are bounded.

The following example illustrates that
\cref{asm:modelTight} is in some cases implied by \cref{asm:parametersCompact},
similar statements hold for the examples in \cref{sec:exampleModels}.
\begin{example}
    \label{example:tightIfBounded}
    The model $\Pf_\theta = \mathcal{N}(\mu,\sigma^2)$,
    with parameter
    $\theta = (\mu,\sigma) \in \Theta \subset \R \times (0,\infty)$
    is tight if and only if $\Theta$ is bounded.
\end{example}

Combining these assumptions and the universal approximation theorem,
we get the following result.

\begin{prop} %
    \provenlabel{prop:approximation}
    Suppose \cref{asm:parametersCompact,asm:bayesContinuous,asm:modelTight} hold.
    Then, for every $\varepsilon > 0$,
    there exists a ReLU neural network 
    with two hidden layers 
    and parameters $\phi_m^*$ that has bounded outputs and satisfies
    \begin{align}\label{eq:point_risk}
        R_{\theta_0}(\phi_m^*) - R_{\theta_0}(f_m^*)
        =
        \E_{\theta_0}\brackets{
            \norm[2]{f_{\phi_m^*}(Z)
            -
            f_m^*(Z)}^2
        }
        <
        \varepsilon
        ,
    \end{align}
    for all $\theta_0 \in \Theta$; here the expectation is over $Z \sim \Pf_{\theta_0}$
    for a fixed value of $\theta_0$.
    Furthermore, the same network satisfies 
    \begin{align}\label{eq:int_risk}
        R(\phi_m^*) - R(f_m^*)
        &=
        \E\brackets{
            \norm[2]{f_{\phi_m^*}(Z)-f_m^*(Z)}^2
        }
        <
        \varepsilon
        ,
    \end{align}
    where the expectation is over $\theta \sim \pi$ and $Z \sim \Pf_\theta$.
\end{prop}

Since \cref{thm:cyb} guarantees the existence of a uniformly good approximation,
the network parameters $\phi_m^*$ do not depend on $\theta_0$.
Hence, we obtain the result in~\cref{eq:int_risk} for the integrated risk $R$
by integrating out the prior distribution $\theta \sim \pi$ in~\cref{eq:point_risk}.
The network architecture used in the proof of \cref{prop:approximation} has two hidden layers:
one for the actual approximation of the Bayes estimator $f^*$,
and one to make sure the output is contained in the bounded set $\brackets{-B,B}^p$.
By \cref{asm:parametersCompact}, this interval contains the entire parameters space $\Theta$,
and the clipping does therefore not impact the approximation quality of the network.

As an intermediate theoretical result, we can find a sequence of  neural network parametrization $\phi^*_m$ for $m$ replicates
such that 
\[ R_{\theta_0}(\phi_m^*) \to 0 \qquad R(\phi_m^*) \to 0, \qquad  m\to\infty, \]
that is, the pointwise and the joint risks of the optimal neural
networks converge to zero. This follows by choosing 
different $\varepsilon_m$ in \cref{eq:point_risk} and \cref{eq:int_risk} for each value of $m$
such that $\varepsilon_m \to 0$ for $m \rightarrow \infty$,
and using \cref{cor:bayesRiskToZero} for convergence of the Bayes risk.
In practice, this result is not applicable directly since 
the optimal network is unknown. 
In the next section we therefore study the generalization error
that quantifies how far the estimated parameters $\phi_m^N$ are 
from the optimal $\phi^*_m$.

\begin{remark}
    \cref{thm:cyb} on universal approximation states the existence of a shallow neural network for arbitrary good approximation only. \citet{poggio2017and} show that in the absence of  additional assumptions on the smoothness of the function, the complexity of a shallow neural network is at least exponential in the input size when the activation function is not a polynomial. In the literature, this is referred to as the \textit{curse of dimensionality}. To overcome this, it is believed that deep neural networks are necessary.
    \citet{poggio2017and} also argue that deep neural networks can more efficiently approximate functions with hierarchical compositional structure, leading to significantly reduced complexity compared to shallow networks; see also \cite{schmidt2020nonparametric}.
    \citet{yarotsky2017error} constructs approximation schemes for deep ReLU networks, where the required network complexity depends explicitly on the smoothness of the target function. As the function's smoothness increases, the exponential dependence on the input dimension can be mitigated. However, when the smoothness is insufficient, the number of neurons per layer can still grow exponentially with dimension.
\end{remark}

We conclude this section by verifying the conditions for \cref{prop:approximation} to hold for the models considered in \cref{sec:exampleModels}.
We have to check \cref{asm:parametersCompact,asm:bayesContinuous,asm:modelTight}.
The applicability of \cref{asm:parametersCompact} is discussed following its statement,
and the other two assumptions have to be checked on a per-model basis,
which is done below.

\begin{lemma}
    \provenlabel{lemma:approxGauss}
    When restricted to a compact parameter space $\Theta$
    (\cref{asm:parametersCompact}),
    the powered exponential model from \cref{eq:poweredExponential}
    and the \Matern{} model from \cref{eq:matern}
    satisfy \cref{asm:modelTight,asm:bayesContinuous}.
\end{lemma}

\begin{lemma}
    \provenlabel{lemma:approxEV}
    When restricted to a compact parameter space $\Theta$,
    the logistic model from \cref{eq:logistic}
    and the \BrownResnick{} model from \cref{eq:brownResnick}
    satisfy \cref{asm:bayesContinuous,asm:modelTight}.
\end{lemma}
More generally,
any family of multivariate distributions with standardized univariate margins
is uniformly tight,
as the parameter only affects the multivariate interactions,
but not the overall magnitude of the observations.
For Gaussian random fields,
it is sufficient to assume that the covariance function is bounded
over $\Theta$, to guarantee tightness of the model.

%% file: tex/05_generalization.tex
\section{Generalization error}

\label{sec:generalizationError}

In this section we consider the generalization error and its decomposition in
\cref{eq:errdec,eq:genErrdec}.
As discussed in \cref{sec:errdec},
this error term is mainly driven by the term
$R(\phi_m^N) - R_N(\phi_m^N)$,
which captures the difference between the true risk $R$
and the empirical risk $R_N$,
with respect to which the neural network $\phi_m^N$ is trained.
The remaining terms in \cref{eq:genErrdec}
are bounded with conventional techniques in the proof of \cref{cor:approximationRisk}.

In order to bound this error term,
we use the concept of robust algorithms,
introduced by \cite{xu2012robustness}.
Throughout, they assume that the loss function is upper bounded by some $E<\infty$.
We also consider this assumption to be satisfied,
since \cref{asm:parametersCompact} implies $\tnorm[\infty]{\theta} < B$,
and the output $\hat \theta$ of a neural network can be restricted to be bounded by $B$,
as described in the proof of \cref{prop:approximation},
yielding a maximal quadratic loss of
\[\norm[2]{\theta - \hat\theta}^2 \leq 4pB^2:=E.\]
Applying results based on robustness directly to neural networks
requires the input and output space of the network to be compact.
In our setting, this is satisfied for the output space,
since \cref{asm:parametersCompact} implies that $\Theta$ is bounded,
but is not a sound assumption for the input space,
which consists of samples from the support of the analyzed distribution.
We therefore employ the concept of pseudo-robustness,
also introduced by \cite{xu2012robustness},
which can be used together with the \cref{asm:modelTight} of a uniformly tight model
to bound the error incurred by input values coming from outside a certain compact set.

Considering two sets $\Xcal$ and $\Ycal$,
an algorithm $A$ is a mapping from a training sample $s = (x, y) \in \lr{\Xcal \times \Ycal}^N$
to a function $A_s \in \mathcal{F}$,
where $\mathcal{F}$ is a subset of functions $f: \Xcal \to \Ycal$.

\begin{example}
    In the general setting of neural estimators we have that $\Xcal = \Zcal^m$ is the space of $m$-dimensional samples and $\Ycal = \Theta \subset \mathbb R^p$ is the Euclidean space containing the parameter space. The set of functions $\Fcal$ is the set of neural networks considered in the training; one example of such as function class is the set of single-layer, fully-connected, feedforward neural networks in Example~\ref{example:single_layer_NN}. For a set of training points $s = ((z_1,\theta_1), \dots, (z_N,\theta_N) ) \in (\Xcal \times \Ycal)^N$, the algorithm $A$ then is the empirical risk minimizer $s\mapsto \phi^N_m$ as defined in~\cref{eq:intro_emp_risk_minimizer}.
\end{example}

The following definition is a simplified version of Definition~5 in \cite{xu2012robustness}.
\begin{definition}[Pseudo-Robustness]
    Let $K \in \NN$, $\varepsilon \geq 0$,
    and
    $\hat N: \lr{\Xcal \times \Ycal}^N \to \set{1, \dots, N}$.
    An algorithm $A: s \mapsto A_s$
    is called ($K$, $\varepsilon$, $\hat N$)-pseudo-robust
    if $\Xcal \times \Ycal$ can be partitioned into $K$ disjoint sets $\set{C_i}_{i=1, \dots, K}$,
    such that
    for all $s \in \lr{\Xcal \times \Ycal}^N$,
    there exists a subset of training samples,
    indexed by $\hat I \subseteq \set{1, \dots, N}$
    with $\tabs{\hat I} = \hat N(s)$,
    that satisfies the following:
    For all
    $i \in \set{1, \dots, K}$,
    $j \in \hat I$,
    satisfying
    $s_j = (x_j,y_j) \in C_i$,
    and any
    $(x',y') \in C_i$,
    it holds that
    \begin{align*}
        \abs{
            \ell(A_s(x_j), y_j)
            -
            \ell(A_s(x'), y')
        }
        \leq
        \varepsilon
        .
    \end{align*}
\end{definition}

Generally speaking, the value of $\hat N$ needs to be close to $N$
in order to obtain relevant error bounds from the results below.
If $\hat N \equiv N$,
the algorithm is called $(K, \varepsilon)$-robust;
see Definition~2 in the same paper.
Note that (pseudo-)robustness itself is a deterministic property.
However,
if the input to the algorithm is a random variable $S$
consisting of $N$ i.i.d.~samples from the joint distribution
of random variables $X$ and $Y$,
each taking values in $\Xcal \times \Ycal$,
Theorem~4 in \cite{xu2012robustness}
shows that
the generalization error
of a ($K$, $\varepsilon$, $\hat N$)-pseudo-robust algorithm
can be bounded with probability $1 - \delta$ for any $\delta\in(0,1)$ by
\begin{align}
    \label{eq:robustnessBound}
    \abs{
        R(A_S)
        -
        R_N(A_S)
    }
    \leq
    \frac{\hat N(S)}{N}
    \varepsilon %
    +
    E \lr{
        \frac{N - \hat N(S)}{N}
        +
        \sqrt{\frac{2K\log(2)+2\log(\frac{1}{\delta})}{N}}
    }
    .
\end{align}

The error bound in this result still depends on the realization of $S$
through the number $\hat N(S)$.
In the general case, $\hat N$ can be an arbitrary function,
however, in our setting
it will be determined by the number of samples $(X,Y)$ that fall into a certain set $\Xcal' \times \Ycal'$.
Approximating the ratio of ``outliers'', $(N - \hat N(S)) / N$, by its expectation
and bounding the resulting error term then yields the following result.
\begin{lemma}
    \provenlabel{lemma:pseudoRobustnessBound}
    Let $A: s \mapsto A_s$ be a $(K, \varepsilon, \hat N)$-pseudo-robust algorithm,
    where $\hat I(s) = \setm{i}{(x_i, y_i) \in \Xcal' \times \Ycal'}$
    for some sets $\Xcal' \subseteq \Xcal$ and $\Ycal' \subseteq \Ycal$,
    and $\hat N(S) = \tabs{\hat I(s)}$.
    Then, for any $\delta\in(0,1)$ the generalization error
    can be bounded with probability at least $1 - \delta$ by
    \begin{align*}
        \abs{
            R(A_S)
            -
            R_N(A_S)
        }
        \leq
        \varepsilon
        +
        E  \lr{
            \P\left((X,Y) \notin \Xcal' \times \Ycal'\right)
            +
            \sqrt{\frac{\log(\tfrac{2}{\delta})}{2N}}
            +
            \sqrt{\frac{2K\log(2)+2\log(\frac{2}{\delta})}{N}}
        }
        .
    \end{align*}
\end{lemma}

Lastly,
to apply the concept of robustness to neural estimators,
we need to establish that algorithms yielding neural networks are pseudo-robust.
To this end, the following result is a modification of Example~7 in \cite{xu2012robustness},
considering the quadratic loss function and the ReLU activation function.
We consider the set of restricted (fully connected feedforward) neural networks $f_\phi$
with $L$ layers,
that map from $\Xcal \subseteq \Rd[md]$ to a set $\Ycal \subseteq \Rd[p]$,
bounded by $B < \infty$.
These networks can be parametrized by 
\begin{align}
    \label{eq:restrictedParameters}
    \Phi^{\alpha, L}
    &=
    \setm{
        \phi = (w, \beta)
    }{
        N_l \leq \alpha,
        \sum_{i=1}^{N_l}
        \abs{w_{i,j}^l}
        \leq \alpha
        , \forall
        l = 1, \dots, L
        \;\land\;
        \abs{f_\phi(x)} \leq B
        , \forall x \in \Xcal
    }
    ,
\end{align}
for some $\alpha \geq 1$,
where
$N_l$ the size of the $l$th layer for $l=1,\dots, L$,
and $w^{(l)} \in \R^{N_l \times N_{l+1}}$
the corresponding weight matrix.
The first inequality is imposed to ensure existence of the
minimizer in \cref{eq:restrictedMinimizer} below.

\begin{lemma} %
    \provenlabel{lemma:NNRobust}
    Let $\Ycal$ be bounded by $B < \infty$,
    let $\Xcal'$ be a bounded subset of $\Xcal$,
    and fix $\alpha > 0$.
    Then, an algorithm taking values in $\mathcal F = \setm{f_\phi: \Xcal \rightarrow \Ycal}{\phi \in \Phi^{\alpha, L}}$
    is $(K(\gamma), c\gamma, \hat N)$-pseudo-robust for all $\gamma > 0$,
    where
    with %
    \begin{align*}
        K(\gamma)
        &=
        \mathcal{N}(\gamma/2, \Xcal' \times \Ycal, \norm[\infty]{\cdot})
        \\
        c
        &=
        4Bp(\alpha^L + 1)
        \\
        \hat N(S)
        &=
        \#\setm{i}{(x_i, y_i) \in \Xcal'}
        ,
    \end{align*}
    where $\mathcal{N}(r, \Scal, \norm{\cdot})$ denotes the covering number
    of a set $\Scal$ with respect to the norm $\norm{\cdot}$,
    using balls of radius $r$.
\end{lemma}

The above results allows us to bound the generalization error of neural estimators
under some restrictions.
In our setting, $\Ycal$ is already bounded
by the assumption of a bounded parameter space $\Theta$
(\cref{asm:parametersCompact}) and by choosing $\Xcal' = [-M,M]^{md}$
the probability 
\begin{align}\label{M_choice}
    \P(X \notin \Xcal') = \P(\|Z\|_\infty > M)    
\end{align}
can be controlled using model tightness from \cref{asm:modelTight}.

In order to satisfy the restriction from \cref{eq:restrictedParameters},
the absolute sum of the weights in each layer needs to be bounded by $\alpha$,
which is not guaranteed by the empirical risk minimizer $\phi_m^N$.
However, we can consider the restricted optimizer
$\phi_m^{N,\alpha}$ defined as
\begin{align}
    \label{eq:restrictedMinimizer}
    \phi_m^{N,\alpha}
    &=
    \argmin_{\phi \in \Phi^{\alpha, L}}
    R_N(\phi)
    ,
\end{align}
and choose $\alpha = \alpha(N)$ to increase in $N$ at a certain rate.
Here, the architecture,
and in particular the number of input samples $m$ and layers $L$,
are fixed according to the architecture of the ``optimal'' neural network $\phi^*_m$
from \cref{prop:approximation}.
For large enough $N$,
the weights $\phi_m^*$
are admissible for the restricted empirical risk minimizer,
and robustness results can be used to show that $R(\phi_m^{N,\alpha(N)})$
is close to $R(\phi_m^*)$.
\mysymbol{$\phi_m^{N,\alpha}$}{Restricted empirical risk minimizer}

Putting the above considerations together yields the following result.

\begin{corollary} %
    \provenlabel{cor:approximationRisk}
    Suppose \cref{asm:parametersCompact,asm:modelTight} hold.
    Consider a fixed network architecture with $L$ layers,
    let $\phi_m^*$ be its optimal weights with respect to risk $R$,
    and let $\phi_m^{N,\alpha(N)}$ be the restricted empirical risk minimizer
    defined in~\eqref{eq:restrictedMinimizer}.
    Assume that the output of the network $\phi_m^*$ is bounded by the same $B$ as in
    \cref{eq:restrictedParameters} and \cref{asm:parametersCompact}.
    Then there exists $N_1 \in \NN$ and a functional $\zeta$ such that 
    for every $\delta \in (0,1)$ we have $\zeta(\delta,N) \to 0$ %
    as $N \to \infty$, and for any $N \geq N_1$
    it holds with probability at least $1-\delta$
    \begin{align*}
        R(\phi_m^{N,\alpha(N)}) - R(\phi_m^*) < \zeta(\delta,N) %
        . 
    \end{align*}
    The rate of $\alpha(N)$ and the functional $\zeta$ are given explicitly in the proof.
\end{corollary}

Note that the approach from \cref{cor:approximationRisk}
cannot be used to bound $R(\phi_m^N)$,
the risk of the unrestricted empirical risk minimizer,
since its weights change with $N$,
and, in particular, might diverge at a faster rate
than required to obtain the bound in the corollary.
Restriction of the empirical minimizer by penalizing the size of the weights as in~\cref{eq:restrictedParameters} 
is a well-known regularization technique in machine learning
which has been shown to work well in practice \citep{goodfellowEtAl2016}.
It is therefore natural to consider the risk of the estimator $\phi_m^{N,\alpha(N)}$ rather than the empirical
risk minimizer $\phi_m^{N}$.

\mysymbol{$\Xcal = \Zcal^m \subseteq \R^{md}$}{Input space of neural network}
\mysymbol{$\Ycal \approx \Theta$}{Output space of neural network}
\mysymbol{$\mathcal{F}$}{Subset of functions $f: \Xcal \to \Ycal$}
\mysymbol{$s = (x, y) \in \lr{\Xcal \times \Ycal}$}{Training sample (realization)}
\mysymbol{$S = (X, Y)$}{Training sample (random variable)}
\mysymbol{$A_s \in \mathcal{F}$}{Function output by algorithm $A$ for sample $s$}
\mysymbol{$E$}{Upper bound on loss function}
\mysymbol{$K$}{Number of disjoint sets in robustness definition}
\mysymbol{$\varepsilon(s)$}{Robustness parameter}
\mysymbol{$\delta$}{Probability of error bound}

\begin{remark}
    \cref{lemma:approxGauss,lemma:approxEV} show that \cref{asm:modelTight}
    is satisfied for the examples from \cref{sec:exampleModels},
    no further assumptions on the distribution of the data need to be checked.
\end{remark}

%% file: tex/06_combination.tex
\section{Statistical guarantees for neural estimators}
\label{sec:combi}

\subsection{Consistency results}

The previous sections we have analyzed the different 
components of the risk of a neural estimator.
This involved both statistical aspects of the considered 
model class and properties of the neural networks 
used to construct the estimator. 
We can now combine these results to obtain 
asymptotic guarantees for neural estimators. 
In particular, putting together \cref{cor:bayesRiskToZero},
\cref{prop:approximation},
and \cref{cor:approximationRisk},
and choosing the training sample size $N(m)$, the 
size of the neural network $\alpha(m)$
and the probability $\delta(m)$ appropriately,
we get that the risk of the neural estimator converges to zero as $m\to\infty$.
As above,
we consider the quadratic loss function and use the ReLU activation function in the neural network.

\begin{theorem} %
    \provenlabel{thm:approxconsistency}
    Suppose that
    the conditions from \cref{thm:bvm} as well as
    \cref{asm:parametersCompact,asm:bayesContinuous,asm:modelTight} hold.
    For input sample size $m$,
    consider the network architecture implied by $\phi_m^*$ from \cref{prop:approximation}
    for $\varepsilon_m = R(f_m^*)$.
    For replicate count~$m$,
    training set size~$N$,
    and~$\alpha>0$,
    let $\phi_m^{N,\alpha}$
    denote the restricted empirical risk minimizer from \cref{eq:restrictedMinimizer}.
    Then, there exist $\alpha(m)$ and $N(m) \to \infty$
    such that
    \begin{align}
        \label{eq:approxconsistency}
        R(\phi_m^{N(m),\alpha(m)})
        \pto
        0
        ,\quad m \to \infty
        .
    \end{align}
\end{theorem}
The probability in the above results is over the distribution of the training data.
Using the boundedness of the parameter space $\Theta$,
and hence of the loss function,
we also obtain the following convergences.
\begin{corollary}
    \provenlabel{cor:approxconsistency}
    Consider the assumptions and notation of \cref{thm:approxconsistency}.
    Write $\phi_m$ short for $\phi^{N(m),\alpha(m)}_m$,
    and let $S_m$ denote the training sample
    used to train $\phi_m$.
    Then we have the following convergences
    \begin{align}\label{L2_p_conv}
        f_{\phi_m}(Z)
        \Ltwoto
        \theta
        ,\qquad
        f_{\phi_m}(Z)
        \pto
        \theta
        ,
    \end{align}
    where the left-hand sides are random variables with respect to the joint
    distribution of $S_m, \theta, Z$,
    and the $\theta$ on the right-hand sides is a random variable itself.

    With respect to the pointwise risk $R_\theta$,
    we have
    \begin{align*}
        R_{\theta}\lr{
            \phi_m
        }
        \Loneto
        0
        ,\qquad
        R_{\theta}\lr{
            \phi_m
        }
        \pto
        0
        .
    \end{align*}
    The two expressions are random variables with respect to the distribution
    of $S_m, \theta$.
\end{corollary}

The convergences in~\cref{L2_p_conv} can be seen as consistency 
of the neural estimator in a Bayesian sense.
In general, when conditioning on a single fixed $\theta_0\in\Theta$,
the true parameter in the frequentist setting,
the convergence results need not hold.
It is an open question whether on can find suitable additional conditions under which such
classical pointwise consistency guarantees
hold.

Considering the relevant assumptions,
checked in \cref{lemma:bvmGauss,lemma:bvmEV,lemma:approxEV,lemma:approxGauss},
we conclude that the risk of neural Bayes estimators
applied to the examples of Gaussian random fields and max-stable distributions from \cref{sec:exampleModels}
converges to zero in the sense of \cref{thm:approxconsistency,cor:approxconsistency}.

\subsection{Discussion of convergence rates and Bayes efficiency}
\label{sec:rates}

\cref{thm:approxconsistency} and \cref{cor:approxconsistency} state that  neural estimators 
are consistent in a Bayesian sense. 
In order to show that neural estimators can be fully efficient
compared to the Bayes estimator, we need that the risk ratio satisfies
\begin{align}\label{bayes_efficiency}
   \frac{R(\phi_m^{N(m),\alpha(m)})}{R(f_m^*)}
    &=
    1
    + \frac{R(\phi_m^*)-R(f_m^*)}{R(f_m^*)}
    + \frac{R(\phi_m^N) - R(\phi_m^*)}{R(f_m^*)} \pto 1, \quad m\to\infty.
\end{align}
We discuss the rates of the different terms and conditions
that balance the network complexity, the training sample size
and properties of the statistical model to ensure 
this convergence.

Let $\varepsilon_m = R(f_m^*)$ denote the Bayes risk for
estimation based on $m$ replicates.
For a given $m$, the approximation error $R(\phi_m^*)-R(f_m^*)$ can be 
made arbitrarily small because of bound \cref{eq:int_risk} in 
\cref{prop:approximation}. We choose as target for this error any 
sequence $\eta_m$ such that $\eta_m/\varepsilon_m \to 0$.
The size of the neural network $\phi_m^*$ that achieves this approximation grows both in the input dimension $m$ and the required approximation accuracy $\eta_m$. For the shallow neural networks considered here, and under mild assumptions on the target function $f_m^*$, this growth can 
be linear in $\eta_m^{-m}$ \citep{yarotsky2017error}.
Other architectures could improve the required network size substantially,
both in practice \citep[e.g.,][]{sainsbury2023likelihood} and in theory.

Once the network architecture is fixed for every $m\in\mathbb N$, we control
the generalization error $R(\phi_m^N) - R(\phi_m^*)$ by suitable
choices of the training sample size $N(m)$ and the size of the neural
network $\alpha(m)$
in the parameterization~\cref{eq:restrictedParameters}.
We first observe that the tail heaviness of the statistical model is a driving factor in
the generalization error, where heavier tails with
larger probabilities $\P( \|Z\|_\infty > M)$ in~\cref{M_choice} for large $M$ lead to larger errors. To compensate, we need to choose suitably large training sizes 
$N(m)$. For instance, for sub-Gaussian tails, including the Gaussian model
in \cref{sec:gauss}, 
this
implies that
$N(m)$ grows faster than $\log(\eta_m^{-1})^{2(md+p)}$.
For heavier tails, such as the polynomial tail decay of the 
max-stable models in \cref{sec:exampleModels:maxstable},
the training size should grow faster with rate $(\eta_m)^{-C(md+p)}$
for some constant $C$.

The network size $\alpha(m)$ needs to increase fast enough such that
the optimal neural network $\phi_m^*$ is contained in the 
admissible set~\cref{eq:restrictedParameters}.
Larger and more expressive networks are more prone to overfitting
and therefore require larger training sizes $N(m)$ to control 
the generalization error.
A suitable choice of $\alpha(m)$ and a sequence $N(m)$ that asymptotically grows faster than $\eta_m^{-K m^2}$ as $m\to\infty$, for a large enough constant $K$,  ensures such a trade-off. The convergence in~\eqref{bayes_efficiency} then holds if the tail of the model $Z$ is not too heavy,
such as a sub-Gaussian or polynomial tail.

A detailed rate discussion can be found in~\cref{sec:rate}.

%% file: tex/07_simulation.tex
\section{Simulation study}
\label{sec:simulation}

We conduct several numerical experiments in order to illustrate the 
behavior of the neural estimator. This will help to better  
understand the implication of the theoretical guarantees on the different risk terms
in a practical application.
We concentrate here on the max-stable logistic model introduced in~\eqref{eq:logistic}, fixing the dimension to \( d = 5 \). The prior distribution $\Pi$ for the model parameter \(\theta\) is always uniform over \((0, 1)\), and we consider different numbers of simultaneous replicates \( m \in \{1, \dots, 10\} \).
In this relatively low dimension, it is still possible to obtain 
the Bayes estimator exactly through the Markov chain Monte Carlo (MCMC) algorithm proposed in~\cite{dombry2017bayesian}, which we will use as the baseline throughout.

Additional simulations for the illustrative example of linear model fitting 
can be found in~\cref{app:linear}.

\subsection{Comparison of estimators}

In this section we illustrate the potential of the neural estimator $\hat \theta_{\text{N}} = f_{\phi^N_m}(Z)$ to effectively approximate the Bayes estimator $\hat \theta_{\text{B}} = f^*_m(Z)$ obtained by MCMC. 
We choose a large number $N=10^4$ of training samples to mostly ignore the generalization error in~\eqref{eq:errdec}.
Moreover, using the automated machine learning approach~\citep{H2OAutoML20}, we perform an extensive grid search over a set $\Phi$ of many different architectures to minimize the approximation error.

We train the models through regularization techniques such as dropout. 
Although our theoretical analysis and sample guarantees only hold for the empirical risk minimizer, there are several reasons to use regularization techniques in practice. 
First, due to a complex landscape of local minima, standard algorithms take very long to converge to the true global minimum and might be stuck in local minima and flat regions. Therefore, it can be computationally very hard to find the empirical risk minimizer.
Second, the estimator resulting from regularization-based training usually performs better on the test data set than the true minimizer \(\phi_m^N\).

To evaluate the performance of the neural estimator, we generate 500 independent test samples of \(\theta\) from the prior distribution $\Pi$. Theoretical results suggest a decreasing risk for the MCMC Bayes estimator as \(m\) increases, and similar behavior is expected for the neural estimator, provided it has sufficient approximation capacity.  The results in Figure~\ref{fig:risk_log_model_n} show that the optimal neural estimator approximates the Bayes estimator well. Both estimators have a strongly decreasing risk as \(m\) increases.

\begin{figure}[h]
    \centering
    \begin{minipage}[b]{0.495\textwidth}
        \centering
        \includegraphics[width=\textwidth]{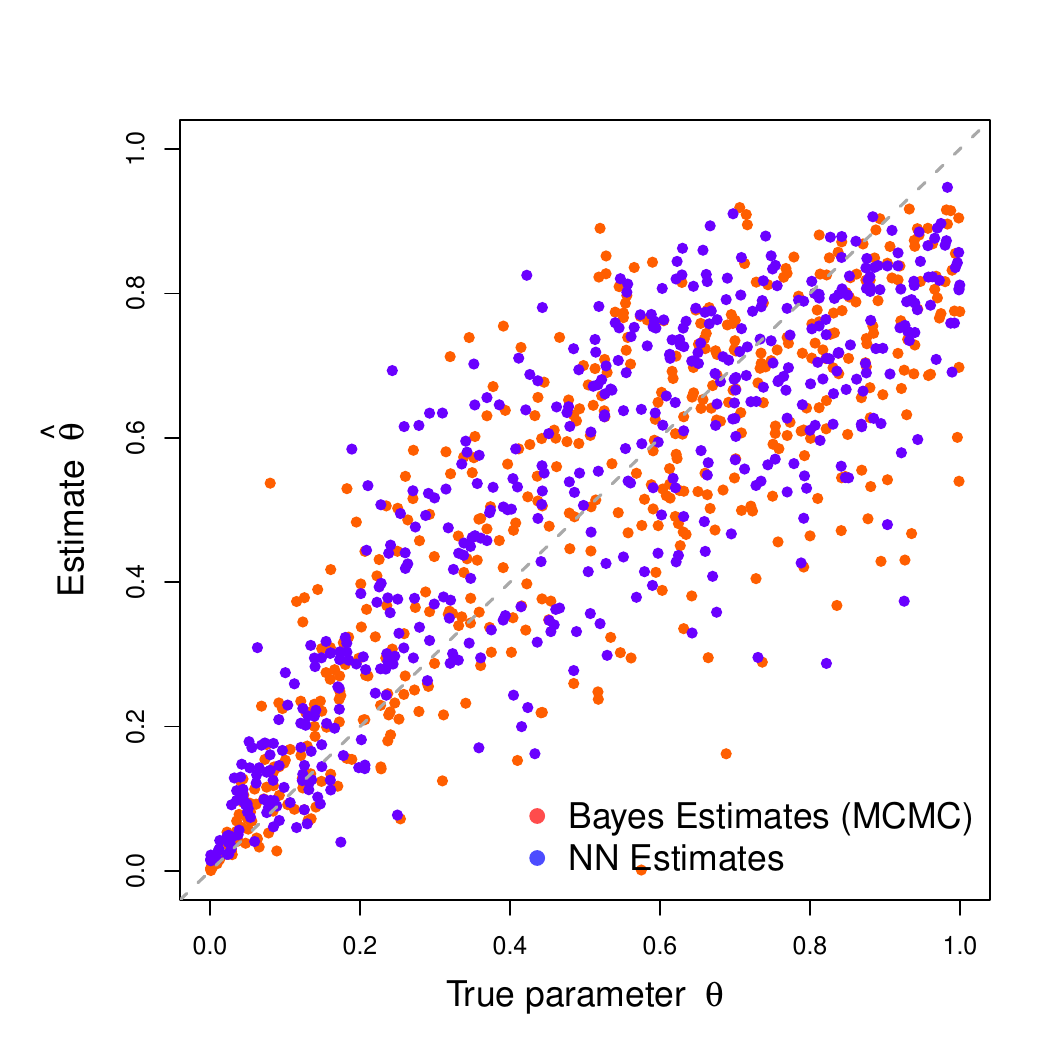}
    \end{minipage}
    \begin{minipage}[b]{0.495\textwidth}
        \centering
        \includegraphics[width=\textwidth]{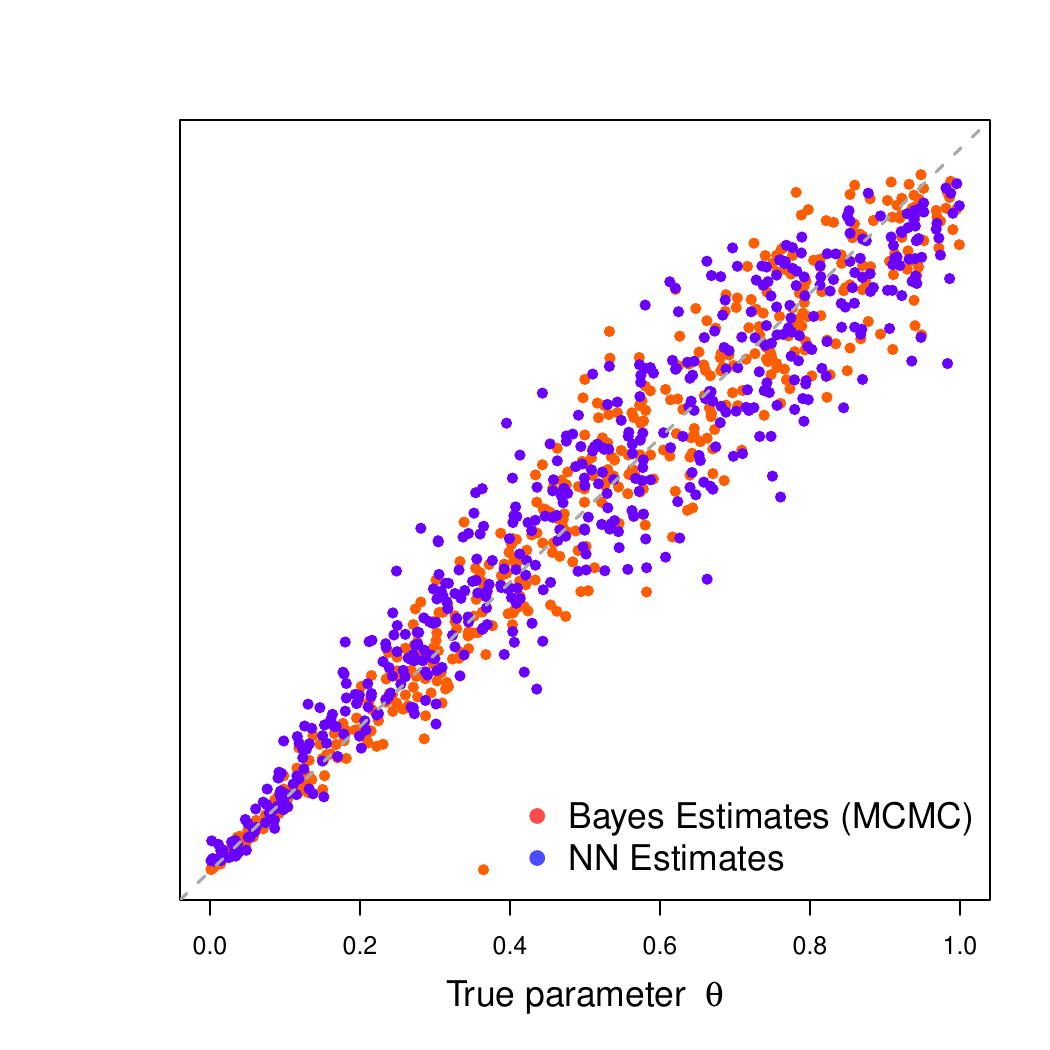}
    \end{minipage}
    
    \caption{Performance of MCMC Bayes estimator and optimal neural estimator for $m=1$ (left) and $m=10$ (right) simultaneous samples, evaluated on 500 test data points with parameters $\theta$ sampled from the uniform prior $\Pi$ on $(0,1)$. The neural estimator was trained with a large number of simulated training samples $N$ and optimized over different network architectures.}
    \label{fig:risk_log_model_n}
\end{figure}

\subsection{Evaluation of training data effects}

We study more carefully the effect of different amounts of training samples $N$ and simultaneous replicates $m$ on the terms of the risk decomposition~\cref{eq:errdec}. In Figure~\ref{fig:errdec_combined} we track all risks that appear in this decomposition, implicitly also containing the approximation error and the generalization error as the differences of relevant curves. All population risks are approximated by the risk on a large test data set.

In addition to the risk of the non-regularized, empirical risk minimizer $\phi_m^N$, we also include the regularized estimator $\tilde{\phi}_m^N$, since it is 
typically applied in practice and improves the risk especially for small $N$.
The optimal neural estimator $\phi^*_m$ is again computed with a large $N$ and by a grid search over many architectures.

We observe that the optimal neural estimator has a comparable risk to the Bayes risk across different numbers of replicates $m=1,\dots, 10$, confirming the results of the simulations in the previous section. Note that the Bayes risk satisfies by definition that  $R(f_m^*) \leq R(\phi_m^*)$, and violations of this in the figure may be due to numerical instability of the MCMC approximation.

The non-regularized and the regularized neural estimator constructed on a small training set of size $N=100$ seem to overfit and exhibit a large generalization error, especially for a growing number of replicates $m$. Following our theoretical results, the reason is that $N$ is not large enough compared to $m$. This is usually not a problem since simulation of the model should be relatively cheap, and we can use larger numbers of training samples. Indeed, for $N=1000$ both neural estimators perform better. The risk of the regularized estimator has a smaller generalization error and quickly approaches the Bayes risk. This confirms that our theoretical results on the empirical risk minimizer are conservative compared to the regularized estimators usually applied in practice, which
typically have even lower risks and potentially better convergence rates.

\begin{figure}[h]
    \centering
    \begin{subfigure}[t]{0.495\textwidth}
        \centering
        \includegraphics[width=\textwidth]{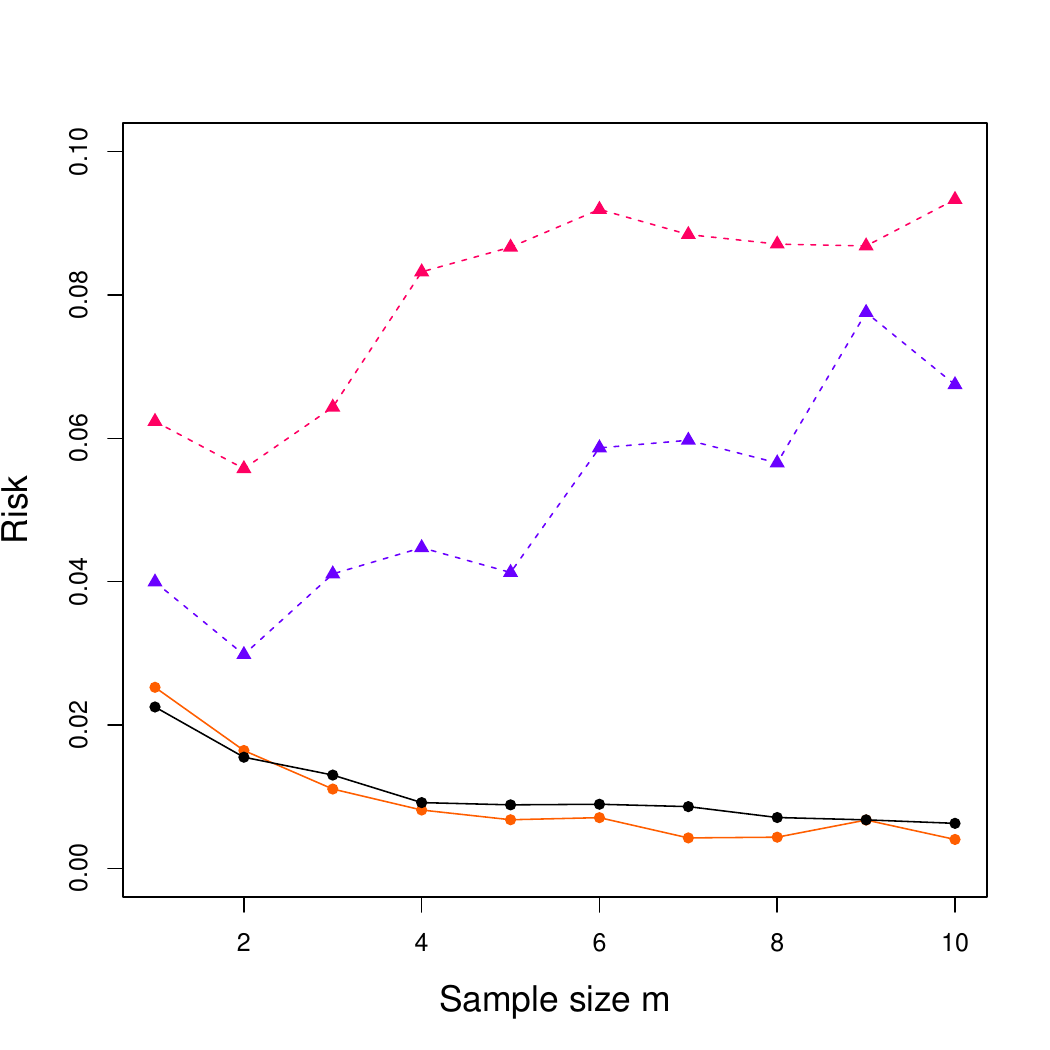}
    \end{subfigure}
    \begin{subfigure}[t]{0.495\textwidth}
        \centering
        \includegraphics[width=\textwidth]{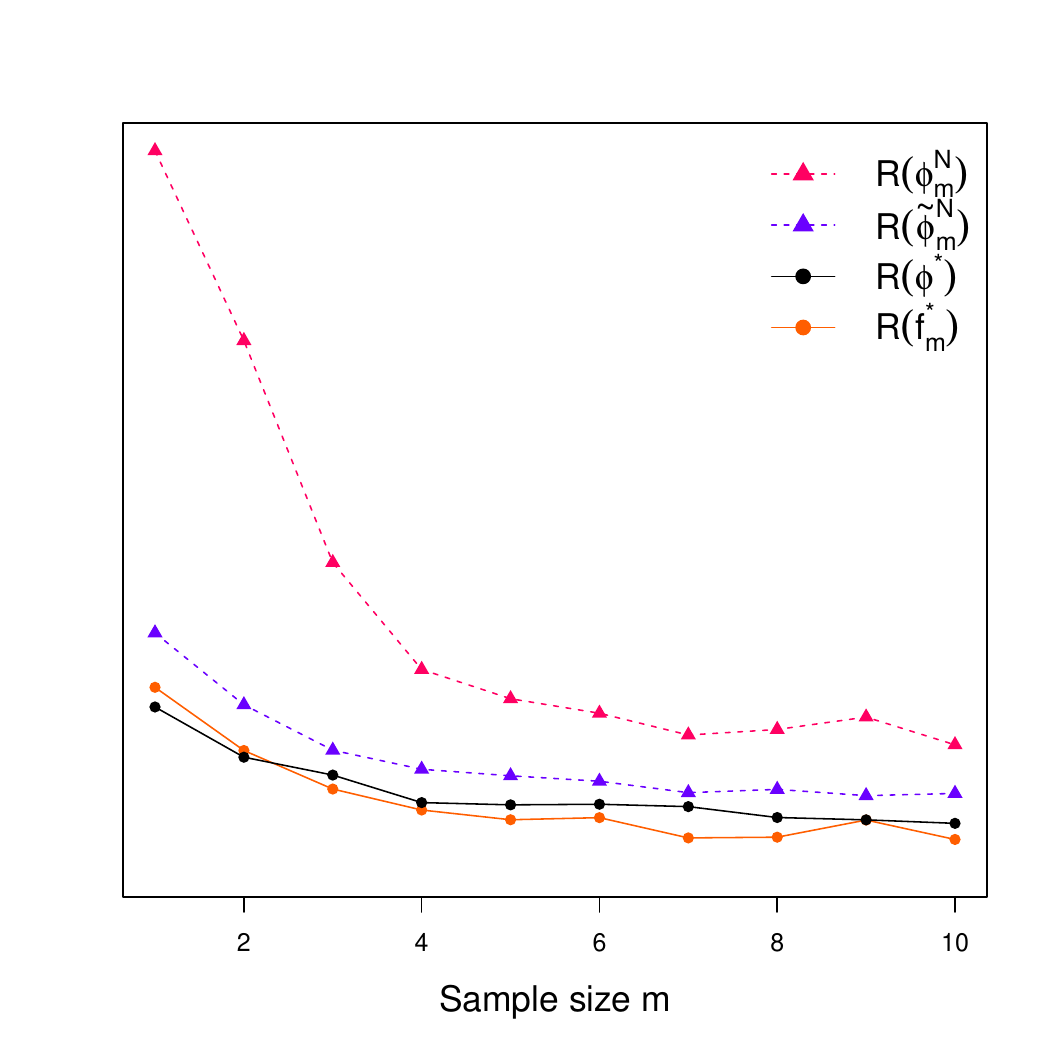}
    \end{subfigure}
    \caption{%
        Comparison of risks appearing in decomposition~\cref{eq:errdec} for the Bayes estimator and neural estimators for different number of replicates $m$.
        Neural networks were trained on $N=100$ (left) and $N=1000$ (right) training samples.
    }
    \label{fig:errdec_combined}
\end{figure}

\subsection{Generalization error decomposition}

Following the theoretical results in~\cref{sec:generalizationError}, we 
 analyze the decomposition~\eqref{eq:genErrdec} of the generalization error when varying the number of replicates \(m\) and the amount of training data $N$. 
Figure~\ref{fig:generrdec_combined} compares the training risks $R_N$
and the population (or test) risks $R$ of the regularized and non-regularized neural estimators.

Consistent with predictions from Hoeffding's inequality, the training and test risks of the optimal neural estimator $\phi_m^*$ are 
almost identical since no overfitting can occur. 
Similarly to~\cref{fig:errdec_combined} we see that for $N=100$ both
the regularized and non-regularized neural estimators overfit, that is,
the training and test risks have a large gap.
This gap diminishes for larger $N=1000$, but is still present in particular for the
non-regularized empirical risk minimizer.

\begin{figure}
    \centering
    \begin{subfigure}[t]{0.495\textwidth}
        \centering
        \includegraphics[width=\textwidth]{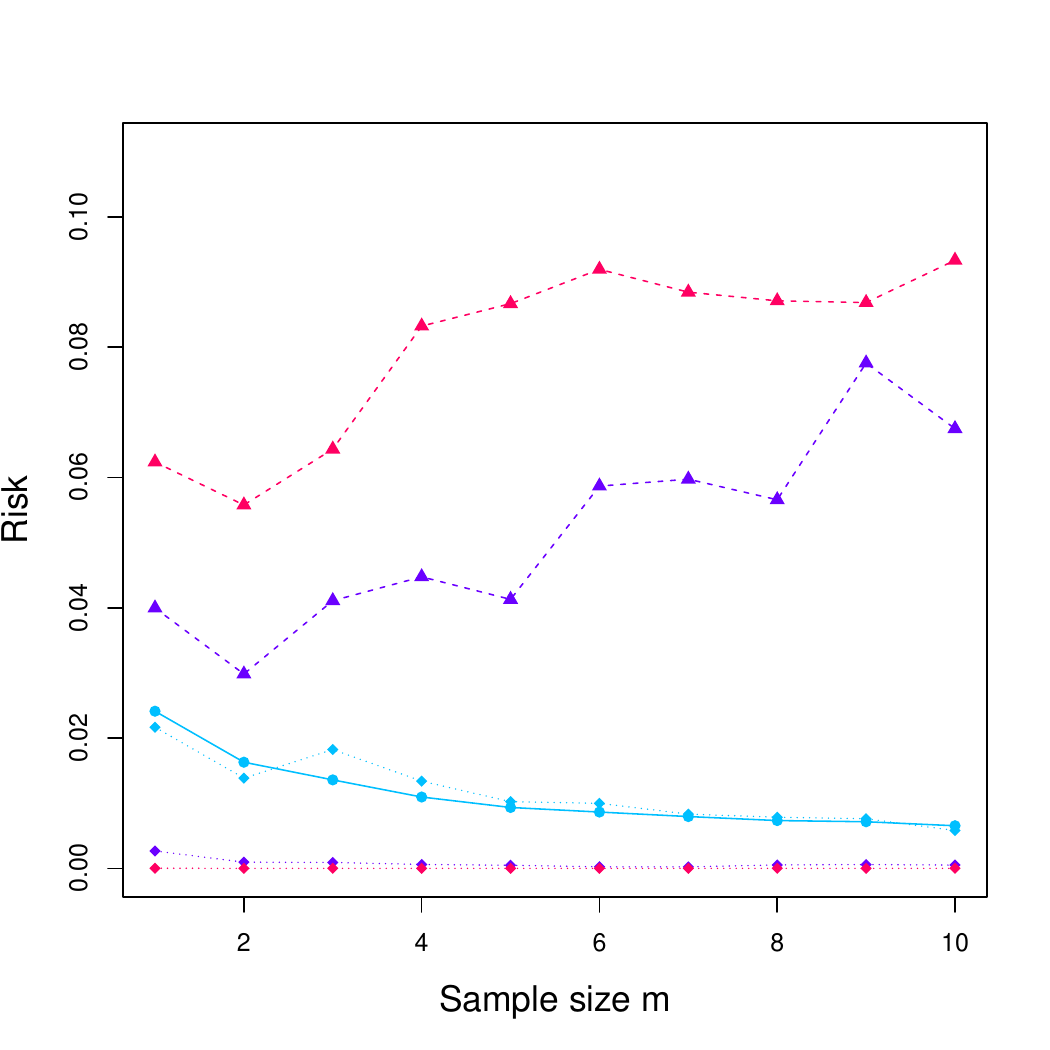}
    \end{subfigure}
    \begin{subfigure}[t]{0.495\textwidth}
        \centering
        \includegraphics[width=\textwidth]{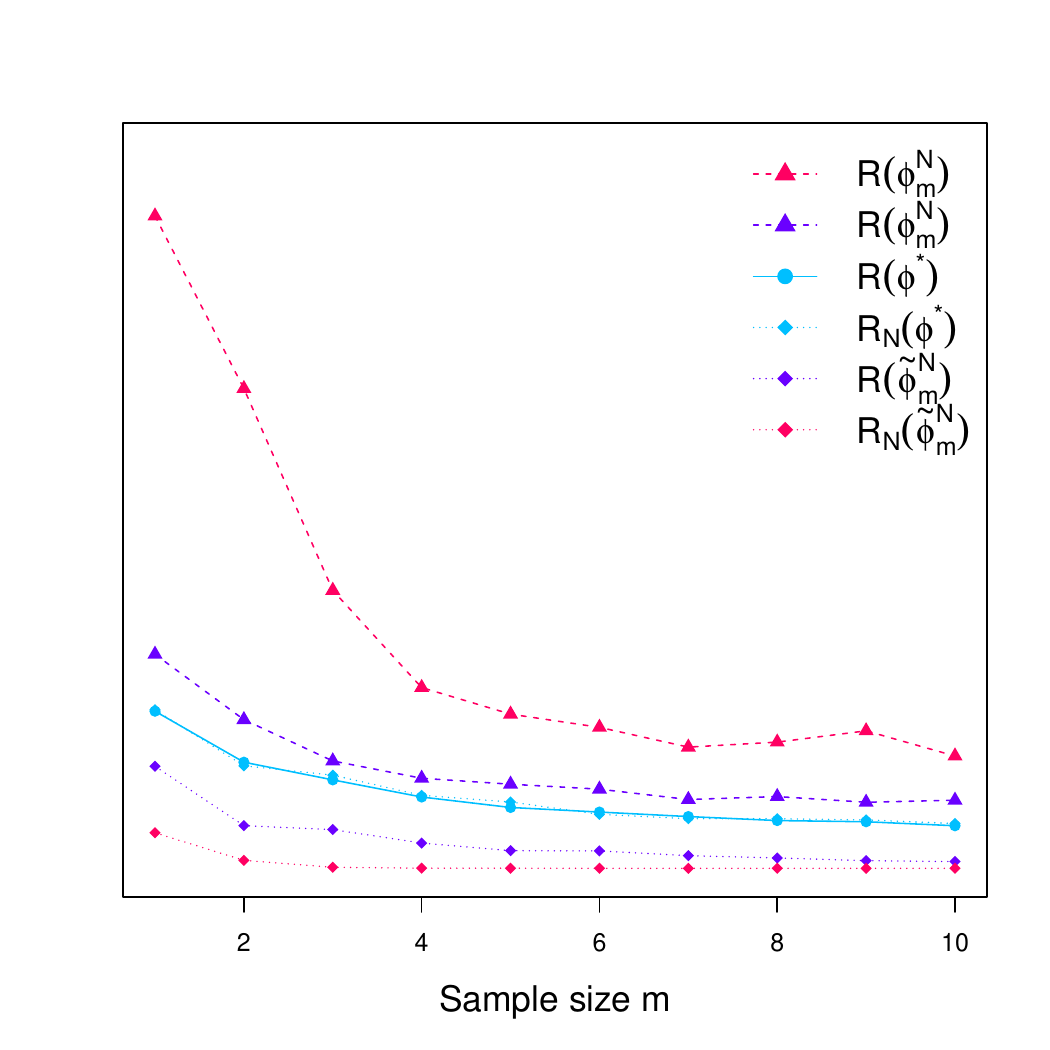}
    \end{subfigure}
    \caption{Comparison of risks appearing in decomposition~\cref{eq:genErrdec} of the generalization error for neural estimators for different number of replicates $m$. Neural networks were trained on $N=100$ (left) and $N=1000$ (right) training samples.}    \label{fig:generrdec_combined}
\end{figure}

Our theoretical results in~\cref{sec:combi} highlight the importance of scaling
the number of training samples \(N\) properly with the number of replicates \(m\). Figure~\ref{fig:generrdec_all} studies this interplay for an increasing number of $N$ and $m\in\{1,10\}$ in case of the regularized estimator. We observe that while the training and test risks
are very different for small $N$, for larger $N$ this difference decreases. Eventually, both the training and test risk converge to the risk of the optimal neural network, and therefore also the Bayes risk. 
This highlights that a sufficiently large number of training samples $N$ 
is crucial to ensure good performance of neural estimators. Since neural estimators are 
applied in scenarios where simulation is much cheaper than estimation, this
should not be the main bottleneck. Training neural networks with larger 
sample size is certainly more expensive, but 
efficient implementations of backpropagation alleviate this issue.
Moreover, this training has to be done only once. Neural estimators are then
amortized in the sense that they can be applied to any new data 
set with $m$ replicates from a distribution $\Pf_\theta$ for 
any $\theta \in\Theta$. This estimation step is extremely fast since 
it involves only a forward pass in the trained neural network.

\begin{figure}[h]
    \centering 
        \includegraphics[width=0.495\textwidth]{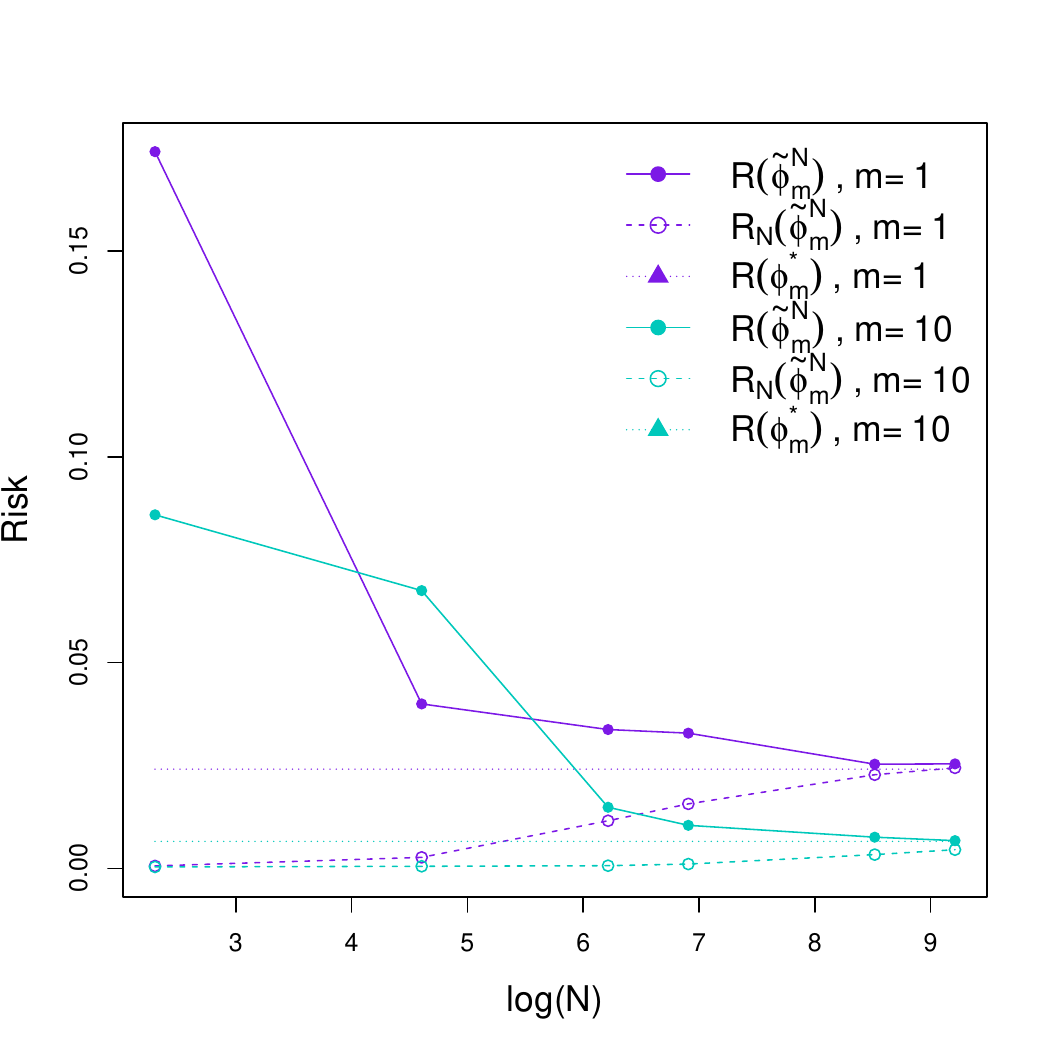}
        \caption{Training and test risks of neural estimators for different number of replicates $m$ and training samples $N$ compared to the optimal neural network risk.}
        \label{fig:generrdec_all}
\end{figure}

%% file: tex/08_conclusion.tex
\section{Conclusion}

The results of this paper are a first step in establishing theoretical 
guarantees for statistical estimators based on neural networks.
Certain bounds that we obtain are conservative since we 
tried to avoid additional assumptions; see the discussion in \cref{sec:rates}.
In the theory of statistical learning for deep neural networks, it is known that one can obtain better convergence
rates and sharper bounds on the required network complexity
under assumptions on the 
smoothness \citep{yarotsky2017error} or a compositional structure of the regression function \citep{poggio2017and, schmidt2020nonparametric}, 
or if the data generating distribution approximately
concentrates on a low-dimensional manifold \citep{jia2023}.

The assumption of a compact parameter space $\Theta$ is crucial in several of our proofs. This emphasizes the fact that neural networks excel at interpolating data but struggle with extrapolation.
In the setting of neural estimators, this means that the whole range of parameters of 
interest 
has to be covered with training samples $\theta_i$. This 
poses an obvious problem in applications where natural bounds 
on this range are not available.
While extrapolation is known to be a difficult problem, there
are first statistical approaches to address this issue 
\citep{shen2024engressionextrapolationlensdistributional,buritica024progressionextrapolationprincipleregression}.

There remain many open questions related to whether 
our results can be extended to neural estimators 
that are based on other architectures 
such as convolutional neural network \citep{gerber2021fast,lenzi2021neural},
or that use different loss functions such the quantile loss. 
Moreover, at the moment it is unclear if neural estimators satisfy some form of asymptotic normality,
and whether the popular bootstrap method for uncertainty quantification
is consistent.

%% file: tex/A02_proofs.tex
\section{Proofs of general results}

\subsection{Bayes Risk}

\begin{lemma}
    \provenlabel{lemma:bayesConvergence}
    Under the assumptions of \cref{thm:bvm} for $\theta_0 \in \Theta$,
    we have
    \begin{align*}
        \P_{\theta_0}\lr{
            m^{1/4} \norm[2]{f_m^*(Z) - \theta_0} > \varepsilon
        }
        \to
        0
    \end{align*}
    for all $\varepsilon > 0$.
\end{lemma}
\begin{proof}[\proofref{lemma:bayesConvergence}]
    For simplicity, consider the one-dimensional case.
    The $p$-dimensional case then follows by the bound
    $\norm[2]{x} \leq \sqrt{p} \norm[\infty]{x}$.

    Under the assumptions of \cref{thm:bvm},
    \cite[Theorem~10.8]{vaart_1998} implies that
    $\sqrt{m}(f_m^*(Z) - \theta_0)$ converges in distribution to a centered normal distribution
    with distribution function, say, $F$.
    Fix any $\delta > 0$ and let $x = F\inv(1-\delta/4)$.
    Then, for large $m$, we have
    \begin{align*}
        \P_{\theta_0}\lr{
            m^{1/4} \lr{f_m^*(Z) - \theta_0} > \varepsilon
        }
        &=
        \P_{\theta_0}\lr{
            \sqrt{m} \lr{f_m^*(Z) - \theta_0} > m^{1/4} \varepsilon
        }
        \\ &\leq
        \P_{\theta_0}\lr{
            \sqrt{m} \lr{f_m^*(Z) - \theta_0} > x
        }
        \\ &\leq
        1 - \P_{\theta_0}\lr{
            \sqrt{m} \lr{f_m^*(Z) - \theta_0} < x
        }
        \\ &\leq
        1 - F(x) + \delta/4
        \\ &=
        \delta/2
        .
    \end{align*}
    The last inequality follows for large enough $m$ from the convergence in distribution
    of $\sqrt{m}(f_m^*(Z) - \theta_0)$ to $F$.
    By symmetry of the normal distribution, we also have
    \begin{align*}
        \P_{\theta_0}\lr{
            m^{1/4} \abs{f_m^*(Z) - \theta_0} > \varepsilon
        }
        \leq
        \delta
        ,
    \end{align*}
    and hence the claimed convergence in probability.
\end{proof}

\begin{proof}[\proofref{cor:bayesRiskToZero}]
    First, consider the pointwise result in \cref{eq:bayesRiskPointwise}.
    Under the assumptions of \cref{thm:bvm} for $\theta_0 \in \Theta$,
    we have by \cref{lemma:bayesConvergence} that
    \begin{align*}
        \P_{\theta_0}\lr{
            m^{1/4} \norm[2]{f_m^*(Z) - \theta_0} > \varepsilon
        }
        \to
        0
        ,
    \end{align*}
    for all $\varepsilon > 0$.
    For a bounded parameter space $\tnorm[\infty]{\Theta} < B$,
    this implies
    \begin{align*}
        R_{\theta_0}(f_m^*)
        &=
        \E_{\theta_0}\brackets{\norm[2]{f_m^*(Z) - \theta_0}^2}
        \\ &=
        \E_{\theta_0}\brackets{
            \norm[2]{f_m^*(Z) - \theta_0}^2 \indic{\norm[2]{f_m^*(Z) - \theta_0} \leq \varepsilon / m^{1/4}}
        }
        +
        \E_{\theta_0}\brackets{
            \norm[2]{f_m^*(Z) - \theta_0}^2 \indic{\norm[2]{f_m^*(Z) - \theta_0} > \varepsilon / m^{1/4}}
        }
        \\ &\leq
        \varepsilon^2 / m^{1/2}
        +
        4pB^2 \P_{\theta_0}\lr{
            m^{1/4} \norm[2]{f_m^*(Z) - \theta_0} > \varepsilon
        }
        \\ &\to
        0
        .
    \end{align*}
    
    For the integrated result in \cref{eq:bayesRiskIntegrated}, consider
    \begin{align*}
        \limit{m} R(f_m^*)
        &=
        \limit{m} \E_{\theta \sim \Pi} R_{\theta}(f_m^*)
        .
    \end{align*}
    Since $\Theta$ is bounded by $B$,
    the risk $R_{\theta}(f_m^*)$ is bounded by $4pB^2$.
    Furthermore, we show above that $R_{\theta}(f_m^*) \to 0$ for all $\theta \in \Theta$.
    Therefore, by the dominated convergence theorem,
    we have
    \begin{align*}
        \limit{m} R(f_m^*)
        &=
        \E_{\theta \sim \Pi} \limit{m} R_{\theta}(f_m^*)
        =
        0
        .
        \qedhere
    \end{align*}
\end{proof}

\subsection{Approximation Error}
\label{sec:proofApproxError}

\begin{proof}[\proofref{eq:approxError}]
    For the squared loss function, and any $f$ (such as $f_{\phi_m^*}$) we have
    \begin{align*}
        R(f)
        &=
        \E\brackets{\norm[2]{f(Z) - \theta}^2}
        \\ &=
        \E\brackets{\norm[2]{f(Z) - \E\brackets{\theta|Z} + \E\brackets{\theta|Z} - \theta}^2}
        \\ &=
        \E\brackets{\norm[2]{f(Z) - \E\brackets{\theta|Z}}^2
        + \E\brackets{\E\brackets{\theta|Z} - \theta}^2}
        + 2 \underbrace{%
            \E\brackets{\E\brackets{(f(Z)-\E\brackets{\theta|Z})(\E\brackets{\theta|Z} - \theta)|Z}}
        }_{=0}
        \\ &=
        \E\brackets{f(Z) - f^*(Z)}^2 + R(f_m^*)
        .
        \qedhere
    \end{align*}
\end{proof}

\begin{proof}[\proofref{lemma:bayesContinuousThree}]
    Note that the continuity of $p_\theta$ implies that of $p_{\theta,m}$.
    Let $B < \infty$ be such that $\Theta \subseteq [-B,B]$.
    Consider a point $z_0 \in \Zcal^m$
    and a sequence $z_n \to z_0$, eventually contained in the domain of $f$.
    Let $K \subset \Zcal^m$ be a compact set containing $z_0$ and almost all $z_n$,
    and define
    \begin{align*}
        P
        &=
        \sup_{\theta \in \Theta, z \in K} p_{\theta,m}(z)
        .
    \end{align*}
    Since $p$ is continuous and $\Theta \times K$ is compact,
    $P$ is finite.
    Next, let
    \begin{align*}
        q_n(\theta)
        &=
        \theta \pi(\theta) p_{\theta,m}(z_n)
        , \qquad
        r_n(\theta)
        =
        \pi(\theta) p_{\theta,m}(z_n)
        .
    \end{align*}
    These functions are dominated by
    $BP\pi(\theta)$ and $P\pi(\theta)$, respectively,
    and converge pointwise to $q_0$ and $r_0$.
    Hence by dominated convergence,
    \begin{align*}
        f(z_n)
        &=
        \E\lr{\theta | Z = z_n}
        =
        \frac{
            \int_\Theta q_n(\theta) d\theta
        }{
            \int_\Theta r_n(\theta) d\theta
        }
        \,\to\,
        \frac{
            \int_\Theta q_0(\theta) d\theta
        }{
            \int_\Theta r_0(\theta) d\theta
        }
        =
        f(z_0)
        .
    \end{align*}
    Note that for large $n$ the denominator
    $\int_\Theta r_n(\theta) d\theta$
    is positive,
    since $z_n$ is eventually in the domain of $f$.
    Since the choice of $z_0$ and $z_n$ was arbitrary, this implies that $f$ is continuous.
\end{proof}

\begin{proof}[\proofref{prop:approximation}]
    Fix $\varepsilon > 0$ and consider $f^*$ to be continuous by \cref{asm:bayesContinuous}.
    Let $B$ be the bound on $\Theta$ from \cref{asm:parametersCompact}.
    Let $\varepsilon_1 = \frac{\varepsilon}{8pB^2}$ and,
    using \cref{asm:modelTight},
    choose $M$ such that
    \begin{align*}
        \P_{\theta_0}\lr{
            \norm[\infty]{Z} > M
        } < \varepsilon_1
        \quad \forall
        \theta_0 \in \Theta.
    \end{align*}
    Next, let $\varepsilon_2 = \sqrt{\frac{\varepsilon}{2}}$ and,
    using \cref{thm:cyb},
    choose a neural network $f_{\phi_m^*}$ with
    \begin{align*}
        \sup_{z \in \brackets{-M,M}^D}
        \norm[2]{f_{\phi_m^*}(z) - f^*(z)}
        <
        \varepsilon_2
        .
    \end{align*}
    By adding another layer to this network that maps each output
    to the interval $[-B,B]$,
    we can ensure that the output of the network is bounded by $B$.
    This can be achieved using the ReLU activation function $\sigma$
    and the mapping
    \begin{align*}
        x
        \mapsto
        \sigma(X+B) - \sigma(X-B) - B
        ,
    \end{align*}
    which is the identity on $[-B,B]$ and equal to $\pm B$ elsewhere.
    Then,
    for any $\theta_0 \in \Theta$,
    \begin{align*}
        R_{\theta_0}(f_{\phi_m^*})
        -
        R_{\theta_0}(f^*)
        &=
        \E_{\theta_0}\brackets{\norm[2]{f_{\phi_m^*}(Z) - f_m^*(Z)}^2}
        \\ &\leq
        \P_{\theta_0}\lr{
            \norm[\infty]{Z} > M
        }
        4pB^2
        +
        \P_{\theta_0}\lr{
            \norm[\infty]{Z} \leq M
        }
        \varepsilon_2^2
        \\ &\leq
        \varepsilon_1 4pB^2 + \varepsilon_2^2
        \\ &=
        \frac{\varepsilon}{8pB^2}4pB^2 + \frac{\varepsilon}{2}
        \\ &=
        \varepsilon
        .
    \end{align*}
    Since this bound holds pointwise for all $\theta_0 \in \Theta$,
    we can integrate over $\theta_0$ to obtain
    \begin{align*}
        R(f_{\phi_m^*})
        -
        R(f_m^*)
        &=
        \E_{\theta_0 \sim \Pi}\brackets{
            R_{\theta_0}(f_{\phi_m^*})
            -
            R_{\theta_0}(f_m^*)
        }
        \\ &\leq
        \E_{\theta_0 \sim \Pi}\brackets{
            \varepsilon
        }
        \\ &=
        \varepsilon
        . \qedhere
    \end{align*}
\end{proof}

\subsection{Generalization Error}

\begin{proof}[\proofref{lemma:pseudoRobustnessBound}]
    Starting from \cref{eq:robustnessBound},
    with ${\delta}/{2}$ in place of $\delta$,
    we obtain with probability at least $1-{\delta}/{2}$ that
    \begin{align*}
        \abs{
            R(A_S)
            -
            R_N(A_S)
        }
        &\leq
        \frac{\hat N(S)}{N}
        \varepsilon
        +
        E \cdot \lr{
            \frac{N - \hat N(S)}{N}
            +
            \sqrt{\frac{2K\log(2)+2\log(\frac{2}{\delta})}{N}}
        }
        \\ &\leq
        E \cdot \lr{
            \frac{N - \hat N(S)}{N}
        }
        +
        \varepsilon
        +
        E \cdot 
            \sqrt{\frac{2K\log(2)+2\log(\frac{2}{\delta})}{N}}        
        .
    \end{align*}
    Note that the ratio $\frac{N-\hat N(S)}{N}$
    is the proportion of (i.i.d.) samples not in $\Xcal' \times \Ycal'$.
    Hence, we can use Hoeffding's inequality
    to bound this term as
    \begin{align*}
        \P\lr{
            \frac{N-\hat N(S)}{N}
            -
            \P\lr{(X,Y) \notin \Xcal' \times \Ycal'}
            >
            t
        }
        &\leq
        \exp\lr{
            -2Nt^2
        }
        \\
        \Rightarrow \quad
        \P\lr{
            \frac{N-\hat N(S)}{N}
            -
            \P\lr{(X,Y) \notin \Xcal' \times \Ycal'}
            >
            \sqrt{\frac{\log\lr{\frac{2}{\delta}}}{2N}}
        }
        &\leq
        \frac{\delta}{2}
        \\
        \Rightarrow \quad
        \P\lr{
            \frac{N-\hat N(S)}{N}
            \leq
            \P\lr{(X,Y) \notin \Xcal' \times \Ycal'}
            +
            \sqrt{\frac{\log\lr{\frac{2}{\delta}}}{2N}}
        }
        &\geq
        1-\frac{\delta}{2}
        .
    \end{align*}
    Plugging this into the above bound,
    and adapting the probability bound of the statement,
    we get with probability at least $1-\delta$ that
    \begin{align*}
        \abs{
            R(A_S)
            -
            R_N(A_S)
        }
        &\leq
        E \cdot \lr{
            \P\lr{(X,Y) \notin \Xcal' \times \Ycal'}
            +
            \sqrt{\frac{\log\lr{\frac{2}{\delta}}}{2N}}
        }
        +
        \varepsilon
        +
        E \cdot
            \sqrt{\frac{2K\log(2)+2\log(\frac{2}{\delta})}{N}}
        \\ &=
        \varepsilon
        +
        E \cdot \lr{
            \P\lr{(X,Y) \notin \Xcal' \times \Ycal'}
            +
            \sqrt{\frac{\log\lr{\frac{2}{\delta}}}{2N}}
            +
            \sqrt{\frac{2K\log(2)+2\log(\frac{2}{\delta})}{N}}
        }
        .
    \end{align*}
\end{proof}

\begin{proof}[\proofref{lemma:NNRobust}]
    Consider a neural network with $L$ layers,
    a single output,
    and activation function satisfying $\abs{\sigma(a) - \sigma(b)} \leq \beta \abs{a-b}$.
    For the absolute value loss function $\ell'(y_1,y_2) = \abs{y_1-y_2}$,
    and any $x, x' \in \Xcal$, $y, y' \in \Ycal$,
    Lemma~4 in \cite{xu2012robustness} then yields
    \begin{align}
        \bigabs{
            \tabs{A_s(x) - y}
            -
            \tabs{A_s(x') - y'}
        }
        &=
        \abs{\ell'(A_s(x), y) - \ell'(A_s(x'), y')}
        \nonumber
        \\ &\leq
        (1 + \alpha^L \beta^L)\norm[\infty]{(x,y)-(x',y')}
        .
        \label{eq:NNRobustLOne}
    \end{align}
    For the quadratic loss function $\ell$,
    $p$-dimensional output,
    and the ReLU activation function used in the lemma, which implies $\beta = 1$,
    we have
    \begin{align}
        \abs{\ell(A_s(x), y) - \ell(A_s(x'), y')}
        &=
        \abs{
            \norm[2]{A_s(x) - y}^2
            -
            \norm[2]{A_s(x') - y'}^2
        }
        \nonumber
        \\ &\leq
        \sum_{i=1}^p
        \abs{
            \abs{A_s(x)_i - y_i}^2
            -
            \abs{A_s(x')_i - y'_i}^2
        }
        \nonumber
        \\ &\leq
        \sum_{i=1}^p
        \abs{
            \biglr{
                \abs{A_s(x)_i - y_i}
                -
                \abs{A_s(x')_i - y'_i}
            }
            \biglr{
                \abs{A_s(x)_i - y_i}
                +
                \abs{A_s(x')_i - y'_i}
            }
        }
        \nonumber
        \\ &=
        4B
        \sum_{i=1}^p
        \bigabs{
            \abs{A_s(x)_i - y_i}
            -
            \abs{A_s(x')_i - y'_i}
        }
        \label{eq:NNRobustLTwoA}
        \\ &\leq
        4pB
        \max_i
        \bigabs{
            \abs{A_s(x)_i - y_i}
            -
            \abs{A_s(x')_i - y'_i}
        }
        \nonumber
        \\ &\leq
        4pB
        (1 + \alpha^L)
        \norm[\infty]{(x,y)-(x',y')}
        \label{eq:NNRobustLTwoB}
        .
    \end{align}
    In \cref{eq:NNRobustLTwoA} we use that $\Ycal$ is bounded by $B$,
    and in \cref{eq:NNRobustLTwoB} we use \cref{eq:NNRobustLOne}.
    
    Setting $c(s) \equiv c = 4pB(1+\alpha^L)$,
    and considering the infinity norm,
    Example~4 in \cite{xu2012robustness} implies that the neural networks are
    $(K(\gamma), c\gamma)$-robust on compact sets $\Xcal' \times \Ycal$ for
    \begin{align*}
        K(\gamma)
        &=
        \covering(\gamma/2, \Xcal' \times \Ycal, \norm[\infty]{\cdot})
        \\
        c
        &=
        4pB(\alpha^L + 1)
        .
    \end{align*}

    Furthermore, since the bound $c\gamma$ does not depend on $s$,
    the definition of pseudo-robustness is immediately satisfied, as well, for
    $\hat N(S) = \#\setm{i}{(x_i, y_i) \in \Xcal' \times \Ycal}$.
\end{proof}

Before proving \cref{cor:approximationRisk},
we establish a lemma about the term labelled ``Generalization Error 2'' in~\cref{eq:genErrdec}.
Recall the notation $D = md$.
\begin{lemma}
    \provenlabel{lemma:approximationRiskTwo}
    Suppose \cref{asm:parametersCompact,asm:modelTight} hold.
    Consider a fixed network architecture with $L$ layers,
    let $\phi_m^*$ be its optimal weights with respect to risk $R$,
    and let $\phi_m^{N,\alpha(N)}$ be the restricted empirical risk minimizer
    defined in~\eqref{eq:restrictedMinimizer}.
    Assume that the output of the network $\phi_m^*$ is bounded by the same $B$ as in
    \cref{eq:restrictedParameters} and \cref{asm:parametersCompact}.
    Then there exists $\zeta_1(\delta, N)$ that converges to zero as $N$ grows,
    and $\alpha(N)$,
    such that for every $\delta \in (0,1)$,
    with probability at least $1-\delta$,
    \begin{align*}
        R(\phi_m^{N,\alpha(N)}) - R_N(\phi_m^{N,\alpha(N)})
        <
        \zeta_1(\delta, N)
        .
    \end{align*}
    An expression for $\zeta_1(\delta, N)$ is given in the proof.
\end{lemma}
\begin{proof}[\proofref{lemma:approximationRiskTwo}]
    For $M<\infty$ let $\Xcal' = \brackets{-M, M}^D \cap \Xcal \subseteq \Xcal$. %
    By \cref{asm:parametersCompact},
    the output space $\Ycal$ is bounded by $B$.
    Hence, a uniformly spaced covering with hyperrectangles yields %
    \begin{align*}
        \covering(\gamma/2, \Xcal' \times \Ycal, \norm[\infty]{\cdot})
        &\leq
        \ceil{\frac{2B}{\gamma}}^p \ceil{\frac{2M}{\gamma}}^D
        .
    \end{align*}
    
    Recall that for $p$-dimensional outputs bounded by $B$,
    the quadratic loss is bounded by $E=4pB^2$.
    Using the statements and notation of
    \cref{lemma:pseudoRobustnessBound,lemma:NNRobust} %
    we have with probability $1-\delta$ that
    \begin{align}
        R(\phi_m^N) - R_N(\phi_m^N)
        &\leq
        \varepsilon
        +
        E \cdot \lr{
            \P(\norm[\infty]{Z} > M)
            +
            \sqrt{\frac{\log(\tfrac{2}{\delta})}{2N}}
            +
            \sqrt{\frac{2K\log(2)+2\log(\frac{2}{\delta})}{N}}
        }
        \nonumber
        \\ &=
        4pB(\alpha^L + 1) \gamma
        \label{eq:approxRiskTwoA}
        \\ &\phantom{=}
        +
        4pB^2
        \P(\norm[\infty]{Z} > M)
        \label{eq:approxRiskTwoB}
        \\ &\phantom{=}
        +
        4pB^2
        \sqrt{\frac{\log(\tfrac{2}{\delta})}{2N}}
        \label{eq:approxRiskTwoC}
        \\ &\phantom{=}
        +
        4pB^2
        \sqrt{\frac{
            2
            \ceil{\frac{2B}{\gamma}}^p
            \ceil{\frac{2M}{\gamma}}^D
            \log(2)+2\log(\frac{2}{\delta})
        }{N}}
        \label{eq:approxRiskTwoD}.
    \end{align}
    In \cref{eq:approxRiskTwoB}, we use the fact that the probability
    $\P\lr{(X,Y) \notin \Xcal' \times \Ycal'}$
    can be rewritten as
    $\P\lr{\norm[\infty]{Z} > M}$,
    where the probability is taken with respect to both the prior distribution
    $\theta \sim \Pi$ and the data distribution $Z \sim p_\theta$.
    For this sum to converge to zero,
    we let $\alpha$, $\gamma$, and $M$ be functions of $N$.
    To this end, %
    fix $0 < \xi < \frac{1}{D + p}$
    and $0 < \kappa < \frac{1-2\xi}{D+p} - \xi$,
    and set
    \begin{gather}
        \begin{aligned}
        M(N)
        &=
        N^{\frac{1-2\xi}{D+p} - \xi - \kappa}
        , \\
        \gamma(N)
        &=
        N^{-\xi - \kappa}
        , \\
        \alpha(N)
        &=
        N^{\frac{\kappa}{L}}
        .
        \end{aligned}
        \label{eq:approximationRiskTwoParams}
    \end{gather}
    
    Note that 
    $M(N) \to \infty$ 
    for $N \to \infty$,
    and hence by
    the assumption of a uniformly tight model in~\cref{asm:modelTight},
    term~\cref{eq:approxRiskTwoB} goes to zero.
    Apart from $N$ itself, all expressions in~\cref{eq:approxRiskTwoC} are constant w.r.t. $N$,
    and hence this term goes to zero as well.

    For \cref{eq:approxRiskTwoA}, we have
    \begin{align*}
        4pB(\alpha(N)^L + 1) \gamma(N)
        &=
        4pB(N^{\kappa} + 1) N^{-\xi - \kappa}
        \\ &=
        4pB(1 + N^{-\kappa}) N^{-\xi}
        \to
        0
        .
    \end{align*}

    For \cref{eq:approxRiskTwoD}, we observe that, for large $N$, $M(N)$ will be larger than $B$,
    so we have
    \begin{align}
        2
        \lr{\frac{2B}{\gamma(N)}}^p
        \lr{\frac{2M(N)}{\gamma(N)}}^D
        N\inv
        &\leq
        2\lr{\frac{2M(N)}{\gamma(N)}}^{D+p} N\inv
        \nonumber
        \\ &=
        2^{D+p+1}
        \lr{\frac{\;
            N^{\frac{1-2\xi}{D+p} - \xi - \kappa}
        \;}{
            N^{-\xi - \kappa}
        }}^{D+p}
        N\inv
        \nonumber
        \\ &=
        2^{D+p+1}
        N^{-2\xi}
        \label{eq:approxRiskTwoDTrafo}
        \\ &\to
        0
        .
        \nonumber
    \end{align}
    Up to some vanishing noise introduced by the ceiling operation $\ceil{\cdot}$,
    expression \cref{eq:approxRiskTwoD} is a monotone transformation of this term,
    and also converges to zero.
    
    The expression for $\zeta_1(\delta, N)$ is then given by the sum of the terms in \cref{eq:approxRiskTwoA} to \cref{eq:approxRiskTwoD},
    replacing $M$, $\gamma$, and $\alpha$ by their respective functions of $N$.
\end{proof}

\begin{proof}[\proofref{cor:approximationRisk}]
    We have the decomposition as in \cref{eq:genErrdec}
    \begin{align}
        R(\phi_m^{N,\alpha(N)}) - R(\phi_m^*)
        &=
        R(\phi_m^{N,\alpha(N)}) - R_N(\phi_m^{N,\alpha(N)})
        \label{eq:approxRiskA}
        \\ &\phantom{=}
        + R_N(\phi_m^{N,\alpha(N)}) - R_N(\phi_m^*)
        \label{eq:approxRiskB}
        \\ &\phantom{=}
        + R_N(\phi_m^*) - R(\phi_m^*)
        \label{eq:approxRiskC}
        .
    \end{align}
    Using the rates defined in \cref{eq:approximationRiskTwoParams},
    the term on the right-hand side of \cref{eq:approxRiskA}
    satisfies with probability at least $1-\delta/2$ that
    \begin{align*}
        R(\phi_m^{N,\alpha(N)}) - R_N(\phi_m^{N,\alpha(N)})
        <
        \zeta_1(\delta/2, N)
        .
    \end{align*}
    To bound \cref{eq:approxRiskB},
    let $N_1$ be such that $\phi_m^* \in \Phi^{\alpha(N), L}$ for all $N \geq N_1$; such an $N_1$ exists since $\alpha(N) \to \infty$ as $N\to\infty$.
    For these $N\geq N_1$ the parametrization $\phi_m^*$ is a candidate for the empirical risk minimizer
    $\phi_m^{N,\alpha(N)}$ and hence
    \begin{align*}
        R_N(\phi_m^{N,\alpha(N)}) - R_N(\phi_m^*)
        &\leq
        0
        .
    \end{align*}
    Lastly, for \cref{eq:approxRiskC},
    observe that $R_N(\phi_m^*)$ is a finite sample approximation of $R(\phi_m^*)$,
    and hence we can use Hoeffding's inequality to bound the difference.
    \begin{align}
        \P\lr{
            R_N(\phi_m^*)
            -
            R(\phi_m^*)
            >
            \lambda
        }
        &=
        \P\lr{
            N\inv\sum_{i=1}^N \ell{(\phi_m^*(Z_i), Y_i)}
            -
            \E\brackets{\ell{(\phi_m^*(Z), Y)}}
            >
            \lambda
        }
        \nonumber
        \\ &\leq
        \exp\lr{
            -\frac{N\lambda^2}{8p^2B^4}
        }
        .
        \label{hoeffding}
    \end{align}
    Setting $\lambda = pB^2\sqrt{\frac{8\log(\frac{2}{\delta})}{N}}$,
    we obtain with probability at least $1-\delta/2$ that
    \begin{align*}
        R_N(\phi_m^*) - R(\phi_m^*)
        <
        pB^2\sqrt{\frac{8\log(\frac{2}{\delta})}{N}}
        =:
        \zeta_2(\delta/2, N)
        .
    \end{align*}
    Note that $\zeta_2(\delta, N) \to 0$ for $N \to \infty$.
    Putting everything together,
    we have for all $N \geq N_1$ with probability at least $1-\delta$ that
    \begin{align*}
        R(\phi_m^{N,\alpha(N)}) - R(\phi_m^*)
        <
        \zeta_1(\delta/2, N)
        +
        \zeta_2(\delta/2, N)
        =:
        \zeta(\delta, N)
        .
    \end{align*}
    Since $\zeta(\delta, N)$ is the sum of two converging sequences,
    it also converges to zero for $N \to \infty$.
\end{proof}

\subsection{Joint Result}

\begin{proof}[\proofref{thm:approxconsistency}]
    Recall the error decomposition~\cref{eq:errdec}:
    \begin{align*}
        R(\phi_m^N)
        &=
        \underbrace{R(f_m^*)}_{\text{Bayes Risk}}
        + \underbrace{R(\phi_m^*)-R(f_m^*)}_{\text{Approximation Error}}
        + \underbrace{R(\phi_m^N) - R(\phi_m^*)}_{\text{Generalization Error}}
        .
    \end{align*}
    
    For $m \in \N$, let
    $\varepsilon_m = R(f_m^*)$
    and, using \cref{prop:approximation},
    choose $\phi_m^*$ such that
    \begin{align}
        \label{eq:approxConsistencyA}
        R(\phi_m^*) - R(f_m^*)
        &\leq
        \varepsilon_m
        .
    \end{align}
    This $\phi_m^*$ satisfies the conditions of \cref{cor:approximationRisk},
    so we can,
    with $N_1$ and $\zeta$ from the \namecref{cor:approximationRisk},
    choose $N(m)$ satisfying $N(m) \geq N_1$ and
    \begin{align}
        \label{eq:approxConsistencyB}
        \zeta(3\varepsilon_m, N(m))
        &\leq
        \varepsilon_m
        .
    \end{align}

    Putting these bounds together,
    we have with probability at least $1 - 3\varepsilon_m$ that
    \begin{align*}
        R(\phi_m^{N(m), \alpha(m)})
        &=
        R(f_m^*)
        + \lr{R(\phi_m^*)-R(f_m^*)}
        + \lr{R(\phi_m^{N(m), \alpha(m)}) - R(\phi_m^*)}
        \\ &\leq
        \varepsilon_m
        +
        \varepsilon_m
        +
        \varepsilon_m
        =
        3\varepsilon_m
        .
    \end{align*}
    By \cref{cor:bayesRiskToZero},
    the term $\varepsilon_m$ converges to zero as $m \to \infty$.

    To show that $R(\phi_M^{N(m), \alpha(m)})$ converges to zero in probability,
    consider a fixed $\varepsilon > 0$,
    and choose $M$ such that $3\varepsilon_m < \varepsilon$ for all $m \geq M$.
    This yields
    \begin{align*}
        \P\lr{
            R(\phi_m^{N(m), \alpha(m)}) > \varepsilon
        }
        &\leq
        \P\lr{
            R(\phi_m^{N(m), \alpha(m)}) > 3\varepsilon_m
        }
        \leq
        3\varepsilon_m
        \leq
        \varepsilon
        , \quad
        \forall m \geq M
        ,
    \end{align*}
    which implies
    \begin{align*}
        R(\phi_m^{N(m), \alpha(m)}) \pto 0
        \quad \text{as } m \to \infty
        .
    \end{align*}
\end{proof}

\begin{proof}[\proofref{cor:approxconsistency}]
    In this proof, we repeatedly make use of the fact that
    for uniformly bounded random variables $0 \leq \abs{X_m} < B < \infty$,
    convergence in probability and convergence in $r$-th mean are equivalent,
    i.e.,
    \begin{align*}
        X_m \pto 0
        \;\Longleftrightarrow\;
        X_m \Lrto 0
        .
    \end{align*}

    Recall that the parameter space $\Theta$ is bounded by some $B < \infty$,
    and we constructed $\phi_m$ such that its output is also bounded by $B$.
    The quadratic loss function
    $\ell(\theta, \hat\theta) = (\theta - \hat\theta)^2$
    is therefore nonnegative and bounded by $E = 4pB^2$.
    Recall that the risk $R$ is the expectation of the loss function,
    conditioned on the training sample $S_m$.
    Consequently, the risk $R(\phi_m(Z))$ is bounded by $E$ as well.

    Next, note that the convergence stated in \cref{thm:approxconsistency}
    is equivalent to the convergence of $R(\phi_m)$ to zero in probability.
    Hence, we also have convergence in mean, implying
    \begin{align}
        R(\phi_m)
        &\Loneto
        0
        \nonumber
        \\
        \Rightarrow\quad
        \E(\E(\ell(\theta, f_{\phi_m}(Z)) \mid S_m))
        &\longrightarrow
        0
        \nonumber
        \\
        \Rightarrow\quad
        \E(\ell(\theta, f_{\phi_m}(Z)))
        &\longrightarrow
        0
        \label{eq:corApproxConsistencyLossLone}
        \\
        \Rightarrow\quad
        \E((\theta - f_{\phi_m}(Z))^2)
        &\longrightarrow
        0
        \nonumber
        \\
        \Rightarrow\quad
        f_{\phi_m}(Z)
        &\Ltwoto
        \theta
        \nonumber
        ,
    \end{align}
    which in turn implies $f_{\phi_m}(Z) \pto \theta$,
    completing the proof of the first two expressions in the corollary.

    Regarding the pointwise risk,
    recall that $R_\theta(\phi_m) = \E(\ell(\theta, f_{\phi_m}(Z)) \mid S_m, \theta)$,
    which is a random variable with respect to the distribution of $S_m, \theta$.
    Starting with the convergence in \cref{eq:corApproxConsistencyLossLone},
    we have
    \begin{align*}
        \Rightarrow\quad
        \E(\ell(\theta, f_{\phi_m}(Z)))
        &\longrightarrow
        0
        \\
        \Rightarrow\quad
        \E(\E(
            \ell(\theta, f_{\phi_m}(Z))
            \mid
            S_m, \theta
        ))
        &\longrightarrow
        0
        \\
        \Rightarrow\quad
        R_\theta(\phi_m)
        &\Loneto
        0
        ,
    \end{align*}
    which is again equivalent to convergence in probability,
    completing the proof.
\end{proof}

\subsection{Rate Discussion}
\label{sec:rate}
\subsubsection{Setup}

\cref{thm:approxconsistency} guarantees the existence of a sequence $N(m)$ that ensures convergence
of the generalization error to zero
at the desired rate.
In the following, we analyze how the convergence rate depends on $N$ and derive a lower bound on $N(m)$ sufficient to guarantee this convergence.

In the proof of \cref{lemma:approximationRiskTwo},
the generalization error is decomposed into the terms
\cref{eq:approxRiskTwoA,eq:approxRiskTwoB,eq:approxRiskTwoC,eq:approxRiskTwoD},
whose joint convergence to zero
is guaranteed by the choices of $M(N)$, $\gamma(N)$, and $\alpha(N)$
in~\eqref{eq:approximationRiskTwoParams},
which depend on suitably chosen constants $\kappa$ and $\xi$.
In order to investigate the required growth rate of $N(m)$,
we first define these values $\kappa$ and $\xi$
as functions of $m$:
\begin{align}
    \label{eq:ratesChoiceXiKappa}
    \xi(m) &= \frac{1}{2(md + p)},
    \qquad
    \kappa(m) = \frac{1 - \frac{2}{(md + p)}}{4(md + p)}.
\end{align}
These values are chosen as the midpoint of their admissible ranges, respectively,
different choices would yield similar rates to the ones discussed below.
Note that for large $m$,
the value of $\kappa(m)$ is asymptotically equivalent to $\frac{1}{4(md + p)}$.

We recall the decomposition~\cref{eq:genErrdec} and note that its
second term can be dropped since it is almost surely negative.
The third term in this decomposition can be bounded by Hoeffding's 
inequality as in~\cref{hoeffding}, and a sufficient condition for the 
training sample size is that $N(m)$ grows faster than $\varepsilon_m^{-2}$
as $m\to\infty$, which is a much weaker condition than the ones derived below.

\subsubsection{Generalization Error}
Firstly,
for~\cref{eq:approxRiskTwoDTrafo}, and hence \cref{eq:approxRiskTwoD}, to converge to zero, we need to choose $N(m)$ such that
\begin{align*}
    2^{md+p} N(m)^{-2\xi(m)} \to 0,
\end{align*}
implying superexponential growth in $m$.
In order to obtain convergence at a certain rate~$\varepsilon_m$,
for example at the same rate as the Bayes risk,
we choose
\begin{align}
    \label{eq:rateGeneralization}
    N(m) \geq \left(\frac{1}{\varepsilon_m}\right)^{md+p}\exp\left((\log 2)(md+p)^2\right)
    .
\end{align}
The two remaining terms~\eqref{eq:approxRiskTwoA} and~\eqref{eq:approxRiskTwoC} then converge to zero faster than~\eqref{eq:approxRiskTwoD}.
Convergence of~\eqref{eq:approxRiskTwoB} is guaranteed by the choice of $M$, the rate will be discussed below.

\subsubsection{Network Complexity}
In order to apply the error decomposition in the proof of \cref{cor:approximationRisk},
the optimal network $\phi_m^*$ must be in the set of considered networks, i.e.,
$\phi_m^* \in \Phi^{\alpha(m), L}$.
To get a first insight into the growth of $N(m)$
required to ensure this,
we consider the mild assumption that the threshold $\alpha(m)$ needs to diverge to infinity
as $m \to \infty$.
Applying logarithms to the definition of $\alpha$ in \cref{eq:approximationRiskTwoParams},
we then obtain
\begin{align}
    \label{eq:exprate}
    \frac{\kappa(m)}{L} \log(N(m))
    \to
    \infty.
\end{align}
Keeping the number of layers $L$ fixed,
as is the case for the shallow networks considered in this paper,
and plugging in $\kappa(m) \approx \frac{1}{4(md+p)}$,
this implies that $\log(N(m))$ must grow faster than linearly in $m$,
or equivalently, $N(m)$ must grow faster than exponentially in $m$.
Concretely,
\begin{align*}
    N(m) \in \omega\lr{\exp\big(4(md + p)L\big)},
\end{align*}
where the notation $u(m) \in \omega(v(m))$ means that $u(m)/v(m) \to \infty$ as $m \to \infty$,
which is already satisfied by the growth rate implied by~\eqref{eq:rateGeneralization}.

However,
a more realistic assumption on the network complexity
is that the number of nodes %
in the optimal network $\phi_m^*$ grows %
both in the input dimension $m$ and the required approximation accuracy $\varepsilon_m$,
for example as $\exp(g(m,\varepsilon_m))$ for some function $g(m, \varepsilon)$.
Assuming that the average magnitudes of the optimal network weights stay constant, %
$\alpha(m)$ must also satisfy
\begin{alignat*}{6}
    &&
    \alpha(m)
    =
    N(m)^{\frac{\kappa(m)}{L}}
    &\geq
    \exp(g(m,\varepsilon_m))
    \\
    \Rightarrow\quad&&
    N(m)
    &\in
    \Omega\lr{\exp\big(g(m,\varepsilon_m)4(md + p)L\big)}
    ,
\end{alignat*}
where the notation $u(m) \in \Omega(v(m))$ means that $u(m)/v(m)$ is bounded away from zero as $m \to \infty$.
Considering the shallow networks used in this paper
and some mild assumptions on the regularity of the target function $f_m^*$,
the rate $g(m,\varepsilon_m)$ can be linear in $m \log(\varepsilon_m\inv)$ \citep{yarotsky2017error},
exceeding the growth rate implied by~\cref{eq:rateGeneralization}.
Whether it is possible to construct more efficient network architectures for multiple replicates remains an open question.
From a practical standpoint,~\citep{sainsbury2023likelihood} propose an estimator that mitigates the complexity growth with $m$ to some extent.

\subsubsection{Tail Probability}
\label{}
Finally, the overall bound on the generalization error also involves the convergence rate of the tail term
$\P(\norm[\infty]{Z} > M)$
in~\eqref{eq:approxRiskTwoB}, which is model-dependent.
With the choice of $\xi(m)$ and $\kappa(m)$ from \cref{eq:ratesChoiceXiKappa},
the threshold $M$ can be expressed as
\begin{align*}
    M(m)
    &=
    N(m)^{\frac{1-2\xi(m)}{md+p} - \xi(m) - \kappa(m)}
    =
    N(m)^{\kappa(m)}
    ,
\end{align*}
since $\kappa(m)$ was chosen as exactly half of the remaining terms in the exponent.
For distributions with sub-Gaussian tails, the probability of exceeding this threshold decays at least exponentially in $M(m)^2 = N(m)^{2\kappa(m)}$.
To obtain rate $\varepsilon_m$ in this case it is necessary that 
\begin{align*}
    N(m) \in \Omega\lr{ \log\lr{\varepsilon_m\inv}^{2(md+p)}}.
\end{align*}
In contrast, for heavy-tailed distributions such as those with Fréchet margins, the decay of~\eqref{eq:approxRiskTwoB} is merely polynomial in $M(m)$, leading to
\begin{align*}
    N(m) \in \Omega\lr{\lr{\varepsilon_m\inv}^{C(md+p)}},
\end{align*}
for a constant $C>0$ in order to obtain decrease at rate $\varepsilon_m$.

%% file: tex/A03_proofsExamples.tex
\section{Proofs of applicability to examples}

\subsection{Multivariate normal distribution}
Below, let $\norm[\infty]{\Sigma}$ denote the element-wise (absolute) maximum of a matrix $\Sigma$.
\begin{lemma}
    \label{lemma:BvMmultivariateNormal}
    Let $\Sigma_\theta$ be a parametric covariance matrix,
    satisfying for all $\varepsilon > 0$
    \begin{align}
        \inf_{\substack{%
            \theta \in \Theta
            \\
            \norm{\theta - \theta_0} \geq \varepsilon
        }}
        \norm[\infty]{\Sigma_\theta - \Sigma_{\theta_0}}
        =:
        I_\varepsilon
        &>
        0
        \label{eq:BvMmultivariateNormalInf}
        \\
        \sup_{\theta \in \Theta}
        \norm[\infty]{\Sigma_\theta}
        =:
        S
        &<
        \infty
        \label{eq:BvMmultivariateNormalSup}
        .
    \end{align}
    Let $\hat \Sigma$ be the empirical covariance matrix.
    For $\delta = I_\varepsilon / 2$,
    define the tests
    \begin{align*}
        \phi_m^\varepsilon
        =
        \mycases{
            \mycase{
                0
            }{
                \mathrm{if}\quad
                \tnorm[\infty]{\hat \Sigma - \Sigma_{\theta_0}}
                \leq
                \delta
            }\\
            \mycase{
                1
            }{
                \mathrm{else.}
            }
        }
    \end{align*}
    Then the tests satisfy
    \begin{align*}
        \E_{\theta_0}(\phi_m^\varepsilon)
        \to
        0
        \quad \text{for} \quad
        m \to \infty
        , \\
        \sup_{\norm{\theta - \theta_0} \geq \varepsilon}
        \E_{\theta}(1 - \phi_m^\varepsilon)
        \to
        0
        \quad \text{for} \quad
        m \to \infty
        .
    \end{align*}
\end{lemma}
\begin{proof}[\proofref{lemma:BvMmultivariateNormal}]
    Since, under $\theta_0$, $\hat \Sigma$ is a consistent estimator of $\Sigma_{\theta_0}$,
    and $\delta > 0$,
    we have
    \begin{align*}
        \E_{\theta_0}(\phi_m^\varepsilon)
        &=
        \P_{\theta_0}(\tnorm[\infty]{\hat \Sigma - \Sigma_{\theta_0}} > \delta)
        \to
        0
        .
    \end{align*}

    Furthermore,
    for $\norm{\theta - \theta_0} \geq \varepsilon$,
    let $i_\theta, j_\theta$ be the index of the maximum element of
    $\abs{\Sigma_\theta - \Sigma_{\theta_0}}$.
    This index satisfies
    $\abs{\Sigma_{\theta,i_\theta j_\theta} - \Sigma_{\theta_0,i_\theta j_\theta}} \geq I_\varepsilon$.
    Then
    \begin{align*}
        \E_\theta(1 - \phi_m^\varepsilon)
        &=
        \P_\theta(\tnorm[\infty]{\hat \Sigma - \Sigma_{\theta_0}} \leq \delta)
        \\ &\leq
        \P_\theta(
            \tabs{\hat \Sigma_{i_\theta j_\theta} - \Sigma_{\theta_0,i_\theta j_\theta}}
            \leq
            \delta
        )
        \\ &\leq
        \P_\theta(
            \tabs{\Sigma_{\theta, i_\theta j_\theta} - \Sigma_{\theta_0,i_\theta j_\theta}}
            -
            \tabs{\hat \Sigma_{i_\theta j_\theta} - \Sigma_{\theta,i_\theta j_\theta}}
            \leq
            \delta
        )
        \\ &\leq
        \P_\theta(
            I_\varepsilon
            -
            \tabs{\hat \Sigma_{i_\theta j_\theta} - \Sigma_{\theta,i_\theta j_\theta}}
            \leq
            \delta
        )
        \\ &=
        \P_\theta(
            \tabs{\hat \Sigma_{i_\theta j_\theta} - \Sigma_{\theta,i_\theta j_\theta}}
            \geq
            I_\varepsilon
            -
            \delta
        )
    \end{align*}
    Since $\delta = I_\varepsilon/2$, the last threshold simplifies to $I_\varepsilon - \delta = \delta$.
    To bound this probability,
    observe that for any entry $i, j$ of
    the empirical covariance matrix, we have
    (using \cite[Theorem~3.3.2]{anderson2003} and well-known properties of the chi-squared distribution)
    \begin{align*}
        \Var_\theta(\hat \Sigma_{ij})
        \leq
        \frac{2\norm[\infty]{\Sigma_\theta}^2}{m-1}
        \leq
        \frac{2S^2}{m-1}
    \end{align*}
    and hence by Chebyshev's inequality and the fact that $\hat\Sigma$ is unbiased,
    \begin{align*}
        \P_\theta(
            \tabs{\hat \Sigma_{ij} - \Sigma_{\theta,ij}}
            \geq
            \delta
        )
        \leq
        \frac{\Var_\theta(\hat\Sigma_{ij})}{\delta^2}
        \leq
        \frac{2S^2}{(m-1) \delta^2}
        .
    \end{align*}
    In particular, this holds for $i, j = i_\theta, j_\theta$.
    Since the bound is independent of $\theta$, we have
    \begin{align*}
        \sup_{\norm{\theta - \theta_0} \geq \varepsilon}
        \E_\theta(1 - \phi_m^\varepsilon)
        &\leq
        \frac{2S^2}{(m-1) \delta^2}
        \to
        0
        ,
        \qquad \text{for }
        m \to \infty
        .
        \qedhere
    \end{align*}
\end{proof}

\subsection{Gaussian processes}

Below, we consider a fixed set of points $x_1, \ldots, x_d$ at which the Gaussian processes are observed
and denote
\begin{align*}
    H_h
    &=
    \setm{(i,j)}{\norm{x_i - x_j} = h}
    , \\
    H
    &=
    \setm{h}{H_h \neq \emptyset}
    .
\end{align*}
For a parametric covariance function $C_\theta(h)$,
denote
\begin{align*}
    \Sigma_\theta
    &=
    \lr{C_\theta(\norm{x_i - x_j})}_{ij}
    .
\end{align*}

\begin{lemma}
    \label{lemma:bvmTestsGaussExp}
    Recall the (powered) exponential model from~\cref{eq:poweredExponential}
    with parameter vector $\theta = (\tau, \lambda, \alpha)$
    and covariance function
    \begin{align*}
        C_\theta(h)
        &=
        \exp\lr{
            -\lr{
                h/\lambda
            }^\alpha
        }
        +
        \tau^2 \indic{h=0}
        .
    \end{align*}
    Let $\tau$ be bounded by some $\tau_1 < \infty$.
    For $\abs{H} \geq 3$,
    the powered exponential model
    satisfies the conditions of \cref{lemma:BvMmultivariateNormal}.
    If $\alpha > 0$ is fixed and omitted from the parameter vector $\theta$,
    the conditions are already satisfied for $\abs{H} \geq 2$.
\end{lemma}
\begin{proof}[\proofref{lemma:bvmTestsGaussExp}]
    Note that \cref{eq:BvMmultivariateNormalSup} is satisfied in both scenarios
    because by positive definiteness of the covariance function,
    we have
    \begin{align*}
        \sup_{\theta \in \Theta}
        \norm[\infty]{\Sigma_\theta}
        &\leq
        \sup_{\theta \in \Theta}
        C_\theta(0)
        =
        1 + \tau_1^2
        <
        \infty
        .
    \end{align*}

    For the powered exponential model,
    denote $\theta = (\tau, \lambda, \alpha)$,
    choose some distinct, non-zero $h_1, h_2 \in H$,
    and define the function $f:\mathbb R^3 \to \mathbb R^3$ with
    \begin{align*}
        f(\tau, \lambda, \alpha)
        &=
        \lr{
            C_\theta(0),
            C_\theta(h_1),
            C_\theta(h_2)
        }
        .
    \end{align*}
    If we can show that this function has a continuous inverse $f\inv$,
    condition \cref{eq:BvMmultivariateNormalInf} is satisfied,
    because then for every $y_0\in\mathbb R^3$ and $\varepsilon > 0$,
    there is some $\delta > 0$ such that for all $y\in\mathbb R^3$ we have 
    \begin{align}
        \norm{y - y_0} < \delta
        &\Rightarrow
        \norm{f\inv(y) - f\inv(y_0)} < \varepsilon
        \nonumber
        \\
        \therefore \quad
        \norm{f\inv(y) - f\inv(y_0)} > \varepsilon
        &\Rightarrow
        \norm{y - y_0} > \delta
        \nonumber
        \\
        \therefore \quad
        \norm{x - x_0} > \varepsilon
        &\Rightarrow
        \norm{f(x) - f(x_0)} > \delta
        \nonumber
        \\
        \therefore \quad
        \inf_{\norm{x - x_0} > \varepsilon}
        \norm{f(x) - f(x_0)}
        &\geq
        \delta
        .
        \label{eq:continuousInverseBvm}
    \end{align}
    To do so,
    observe that $f\inv$ can be expressed coordinatewise by
    \begin{align*}
        \tau
        &=
        \sqrt{C_\theta(0) - 1}
        \\
        \alpha
        &=
        \frac{
            \log\lr{-\log C_\theta(h_1)}
            -
            \log\lr{-\log C_\theta(h_2)}
        }{
            \log h_1
            -
            \log h_2
        }
        \\
        \lambda
        &=
        \frac{
            h_1
        }{
            \lr{
                -\log C_\theta(h_1)
            }^{1/\alpha}
        }
        .
    \end{align*}
    Each of these are continuous functions,
    and hence $f\inv$ is continuous as well.

    For the fixed $\alpha$ case,
    the same argument applies with
    \begin{align*}
        f(\tau, \lambda)
        &=
        \lr{
            C_\theta(0),
            C_\theta(h_1)
        }
    \end{align*}
    and the inverse defined coordinatewise as above,
    plugging in the fixed $\alpha$ in the expression for $\lambda$.
\end{proof}

\begin{lemma}
    \label{lemma:bvmTestsGaussMatern}
    Recall the \Matern{} model from \cref{eq:matern},
    with fixed smoothness $\nu$ and parameter vector $\theta = (\tau, \lambda)$.
    \begin{align*}
        C_\theta(h)
        &=
        \frac{2^{1-\nu}}{\Gamma(\nu)}
        \lr{\sqrt{2\nu} \frac{h}{\lambda}}^\nu
        K_\nu\lr{\sqrt{2\nu} \frac{h}{\lambda}}
        +
        \tau^2 \indic{h=0}
        ,
    \end{align*}
    Let $\tau$ be bounded by some $\tau_1 < \infty$.
    For $\abs{H} \geq 2$,
    this model satisfies the conditions of \cref{lemma:BvMmultivariateNormal}.
\end{lemma}
\begin{proof}[\proofref{lemma:bvmTestsGaussMatern}]
    The value of the \Matern{} covariance function at $h=0$ is
    \begin{align*}
        C_\theta(0)
        &=
        1
        +
        \tau^2
        .
    \end{align*}
    Hence, \cref{eq:BvMmultivariateNormalSup} is satisfied
    because by positive definiteness of the covariance function,
    we have
    \begin{align*}
        \sup_{\theta \in \Theta}
        \norm[\infty]{\Sigma_\theta}
        &\leq
        \sup_{\theta \in \Theta}
        C_\theta(0)
        =
        1 + \tau_1^2
        <
        \infty
        .
    \end{align*}

    For $x>0$, the function
    $x \mapsto x^\nu K_\nu(x)$ is strictly decreasing and continuous \cite[Proposition~E.2]{Zhang_typeG2023}.
    Hence, for fixed $0 < h \in H$, the function
    $\lambda \mapsto C_\theta(h)$ is strictly increasing and continuous,
    and so is its inverse.
    Therefore, the map
    \begin{align*}
        f(\tau, \lambda)
        &=
        \lr{
            C_\theta(0),
            C_\theta(h)
        }
    \end{align*}
    has a continuous inverse $f\inv$,
    and by the same argument as in \cref{eq:continuousInverseBvm},
    it follows that \cref{eq:BvMmultivariateNormalInf} is satisfied
    for the \Matern{} model.
\end{proof}

\begin{lemma}
    \label{lemma:bvmDiffInQuadraticMeanGauss}
    Both the powered exponential and \Matern{} covariance functions
    are differentiable in quadratic mean
    with non-singular Fisher information matrix.
\end{lemma}
\begin{proof}[\proofref{lemma:bvmDiffInQuadraticMeanGauss}]
    Lemma~7.6 and Example~7.7 in \cite{vaart_1998}
    state that the density of an exponential family
    \begin{align*}
        p_\theta(x)
        &=
        d(\theta)
        h(x)
        \exp\lr{
            Q(\theta)\T t(x)
        }
    \end{align*}
    is differentiable in quadratic mean
    if the maps $\theta \mapsto Q(\theta)$ are continuously differentiable
    and map the parameter space $\Theta$ into the interior of the natural parameter space.

    In the case of Gaussian distributions with fixed mean,
    the entries of $Q(\theta)$ correspond to the entries of the precision matrix $\Sigma_\theta\inv$.
    Since both the powered exponential and \Matern{} covariance functions are strictly
    positive definite \citep{stein2012interpolation},
    the precision matrix is non-singular, and hence in the interior of the natural parameter space.
    Since matrix inversion is a continuous operation,
    it remains to show that the covariance functions, evaluated at fixed distances $h \in H$, %
    are continuously differentiable in $\theta$.

    For the powered exponential model from \cref{eq:poweredExponential},
    this is immediately clear.
    For the \Matern{} model from \cref{eq:matern} with fixed smoothness $\nu$,
    this follows from the fact that $x \mapsto x^\nu K_\nu(x)$ is continuously differentiable
    \cite[e.g., proof of Proposition~E.2]{Zhang_typeG2023}.
\end{proof}

\begin{proof}[\proofref{lemma:bvmGauss}]
    By \cref{lemma:bvmTestsGaussExp,lemma:bvmTestsGaussMatern},
    we can construct the required tests,
    and by \cref{lemma:bvmDiffInQuadraticMeanGauss},
    both models are differentiable in quadratic mean
    with non-singular Fisher information matrix.
    The condition on the prior can be satisfied
    for a compact parameter space $\Theta$,
    for example by choosing the uniform distribution.
\end{proof}

\begin{proof}[\proofref{lemma:approxGauss}]
    To check \cref{asm:modelTight},
    observe that the variance of the
    powered exponential and \Matern{} covariance functions
    with bounded $\tau$
    is bounded by $S = 1 + \tau_1^2$.
    Furthermore, recall that the mean function is assumed to be zero.
    Hence, for $\varepsilon > 0$
    we can choose $M = \sqrt{S} \Phi\inv(1-\varepsilon/(2d))$,
    where $\Phi$ denotes the standard normal distribution function.
    Then
    \begin{align*}
        \P(\abs{Z_i} > M)
        &\leq
        \Phi\lr{-M/\sqrt{S}}
        +
        1
        -
        \Phi\lr{M/\sqrt{S}}
        \\ &=
        2\lr{1- \Phi\lr{M/\sqrt{S}}}
        \\ &=
        2(1-(1-\varepsilon/(2d)))
        \\ &=
        \varepsilon/d
        ,
    \end{align*}
    and hence
    \begin{align*}
        \P(\norm[\infty]{Z} > M)
        &\leq
        d \P(\abs{Z_i} > M)
        \leq
        \varepsilon
        .
    \end{align*}

    To show continuity of the Bayes estimator (\cref{asm:bayesContinuous}),
    note that the density of a centered Gaussian distribution
    \begin{align*}
        p_\theta(z)
        &=
        \frac{1}{(2\pi)^{d/2} \sqrt{\det{\Sigma_\theta}}}
        \exp\lr{
            -\half z\T \Sigma_\theta\inv z
        }
    \end{align*}
    is always continuous in $z$ and
    continuous in $\theta$ if the covariance matrix $\Sigma_\theta$ is continuous in $\theta$.
    For the powered exponential and \Matern{} models,
    the covariance functions are continuous in $\theta$,
    as discussed in the proof of \cref{lemma:bvmDiffInQuadraticMeanGauss}.
    Hence, by \cref{lemma:bayesContinuousThree}, the Bayes estimator is continuous
    for the considered models and compact parameter spaces.

\end{proof}

\subsection{Max-stable processes}

\begin{proof}[\proofref{lemma:bvmEV}]
    The applicability of the \BvM{} Theorem to the considered extreme value models
    is shown in Proposition~3 and Corollary~2 of \cite{dombry2017bayesian}.
\end{proof}

\begin{proof}[\proofref{lemma:approxEV}]
    To check \cref{asm:modelTight},
    recall that we consider models with
    unit \Frechet{} margins,
    independent of the parameter.
    These margins satisfy
    $Z_i > 0$ almost surely, so we only need to consider the probability of the event $Z_i > M$.
    Let $F$ denote the \Frechet{} distribution function,
    $F\inv$ its quantile function,
    and for $\varepsilon > 0$
    choose $M = F\inv(1 - \varepsilon/d)$.
    Then
    \begin{align*}
        \P(\norm[\infty]{Z} > M)
        &\leq
        \sum_{i=1}^d \P(Z_i > M)
        =
        d (1 - F(M))
        =
        \varepsilon
        .
    \end{align*}

    \cref{asm:bayesContinuous} is checked using \cite{dombry2017asymptotic}.
    Their equation~(1) gives an explicit expression for the probability density function
    of a max-stable distribution.
    Given Condition~A2 from Proposition~3.2 in their paper,
    this expression is continuous in $\theta$ and $z$.
    In turn, Condition~A2 is satisfied for the considered models,
    as shown
    in Lemma~B.1 (logistic),
    the proof of Proposition~4.5 (\BrownResnick{}),
    and Lemma~B.3 together with the proof of Proposition~3.3 (Schlather).
    By \cref{lemma:bayesContinuousThree},
    the continuity of the density in $\theta$ and $z$
    combined with the compact parameter space $\Theta$
    implies that the Bayes estimator is continuous in $z$.

\end{proof}

%% file: tex/A04_Linearmodel.tex
\section{Simulation study for the linear model}\label{app:linear}

We consider the example of the linear model as an illustrative example, since in this case we can compute all error terms explicitly.
Let $(\Omega,\mathcal{A},\mathbb{P})$ denote a probability space.
Define $Z := (Z^1, \dots, Z^m)$ such that $Z^1, \dots, Z^m$ are independently and identically distributed random variables following a normal distribution with mean $\theta\in\Theta =\mathbb R$ and fixed variance~$\sigma^2$.
The distribution $\Pi$ of the parameter $\theta$ is normal with mean $\mu$ and variance $\gamma^2$.
The neural estimator for estimating $\theta$ from $Z$ is constructed using a 
linear model with $k \leq m$ non-zero coefficients and one intercept; this corresponds to a simple neural network with no hidden layers.
The training dataset consist of $N$ i.i.d.~observations $\left((Z_i^1,\dots,Z_i^m), \theta_i\right)$, for $i=1,\dots,N$, drawn from the joint distribution 
\begin{align*}
    \mathbb{P}(Z \in A, \theta \in B) = \int_B \int_A \frac{1}{\sqrt{\left(2\pi \sigma^2\right)^m}} \exp{-\frac{\sum_{j=1}^{m}\left(z_j-\theta\right)^2}{2\sigma^2}}\frac{1}{\sqrt{2\pi \gamma^2}} \exp{-\frac{\left(\theta-\mu\right)^2}{2\gamma^2}}\,d (z_1,\dots,z_m) \,d \theta ,
\end{align*}
where $A \in \mathcal{B}(\R^m)$ and $B \in \mathcal{B}(\Theta)$.
The empirical risk minimization~\eqref{eq:intro_emp_risk_minimizer} simplifies in this case to 
\begin{align}\label{prob:empriskmin}
    \min_{A_1,\dots,A_m, b} \frac{1}{N} \sum_{i=1}^N \left(\sum_{j=1}^m A_j Z_i^j + b - \theta_i\right)^2 \quad \text{s.t. } \|A\|_0 \leq k .
\end{align}
As $N$ tends to infinity, this problem converges almost surely to minimization problem of the population risk, that is, 
\begin{align}\label{prob:riskmin}
    \min_{A_1,\dots,A_m, b} \E \left[\left(\sum_{j=1}^m A_j Z^j + b - \theta\right)^2\right] \quad \text{s.t. } \|A\|_0 \leq k .
\end{align}
When $k=m$, the solution of \eqref{prob:riskmin} corresponds to the Bayes estimator $f^*_m(z) = \mathbb{E}\left[\theta \mid {Z} = z\right]$, which is a linear model with optimal parameter vector $\phi^*_m$ with components
\begin{align*}
    A_1^* = \cdots = A_m^* = \frac{\gamma^2}{m \gamma^2+ \sigma^2}, \quad
    b^* = \mu \frac{\sigma^2}{m \gamma^2+ \sigma^2},
\end{align*}
whose dimension and value depend on $m$. The solution depends on the prior distribution of $\theta$, the variance of ${Z}$, and the number of replicates $m$. Notably, the Bayes estimator differs from the maximum likelihood estimator as it is biased towards the prior assumptions. As $m$ and $\gamma^2$ increase or $\sigma^2$ decreases, the prior assumptions become less influential, causing the Bayes estimator to approach the maximum likelihood estimator.

While the optimization problem~\eqref{prob:riskmin} aims to minimize the expected risk $\E\left[ R_\theta(\phi_m^*)\right]$,
we can explicitly determine the pointwise risk of the estimator in this case
\begin{align*}
    R_\theta(\phi_m^*) = (\mu - \theta)^2  \frac{\sigma^4}{(m \gamma^2+ \sigma^2)^2} + \frac{m\sigma^2 \gamma^4}{(m \gamma^2+ \sigma^2)^2} .
\end{align*}
When $k<m$, the solution of problem \eqref{prob:riskmin} does not correspond to the Bayes estimator. Instead, we have
\begin{align}\label{eq:approx_error}
    R(\phi^*_m) - R(f^*_m) = \frac{\gamma^2\sigma^4+k\sigma^2\gamma^4}{(k\gamma^2+\sigma^2)^2} - \frac{\gamma^2\sigma^4+m\sigma^2\gamma^4}{(m\gamma^2+\sigma^2)^2}> 0.
\end{align}
The optimal risk $R(\phi^*)$ as a function of the number of active parameters $k$ is shown in Figure~\ref{fig:risk_overview}. We clearly see how the expressiveness of the approximating function class, here measured as the number of non-zero coefficients $k$ in a linear model, governs the approximation error. In fact, for a constant $k=10$ this error does not vanish for growing~$m$.

\begin{figure}[bt!]
    \centering 
        \includegraphics[width=0.495\textwidth]{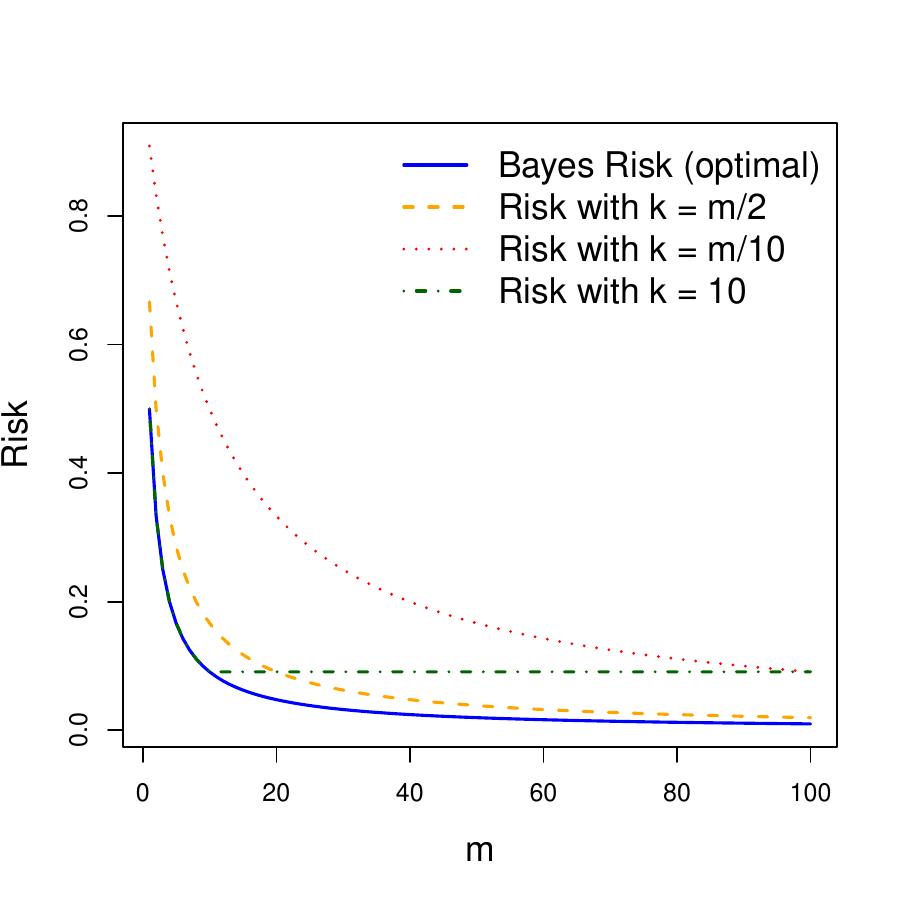}
        \caption{Bayes risk compared to the risks of different approximating functions corresponding to different choices of model complexity $k$, as a function of sample size $m$.}
        \label{fig:risk_overview}
\end{figure}

In the following, we analyze the error decomposition~\cref{eq:genErrdec} and assume that $k=m$, that is, the optimal neural estimator is the Bayes estimator and there is no approximation error. The theoretical analysis suggests that the trained estimator is close to the true estimator when $N$ increases exponentially in the number of replica $m$. Figure~\ref{fig:error_dec} illustrates this for $N=100$ and $N=1000$. While $N=100$ is a sufficiently large training sample size for smaller $m$, there is a gap between the optimal neural estimator and the trained neural estimator for larger $m$. This gap becomes smaller when $N$ increases.
\begin{figure}[h]
    \centering
    \begin{subfigure}[b]{0.495\textwidth}
        \centering
        \includegraphics[width=\textwidth]{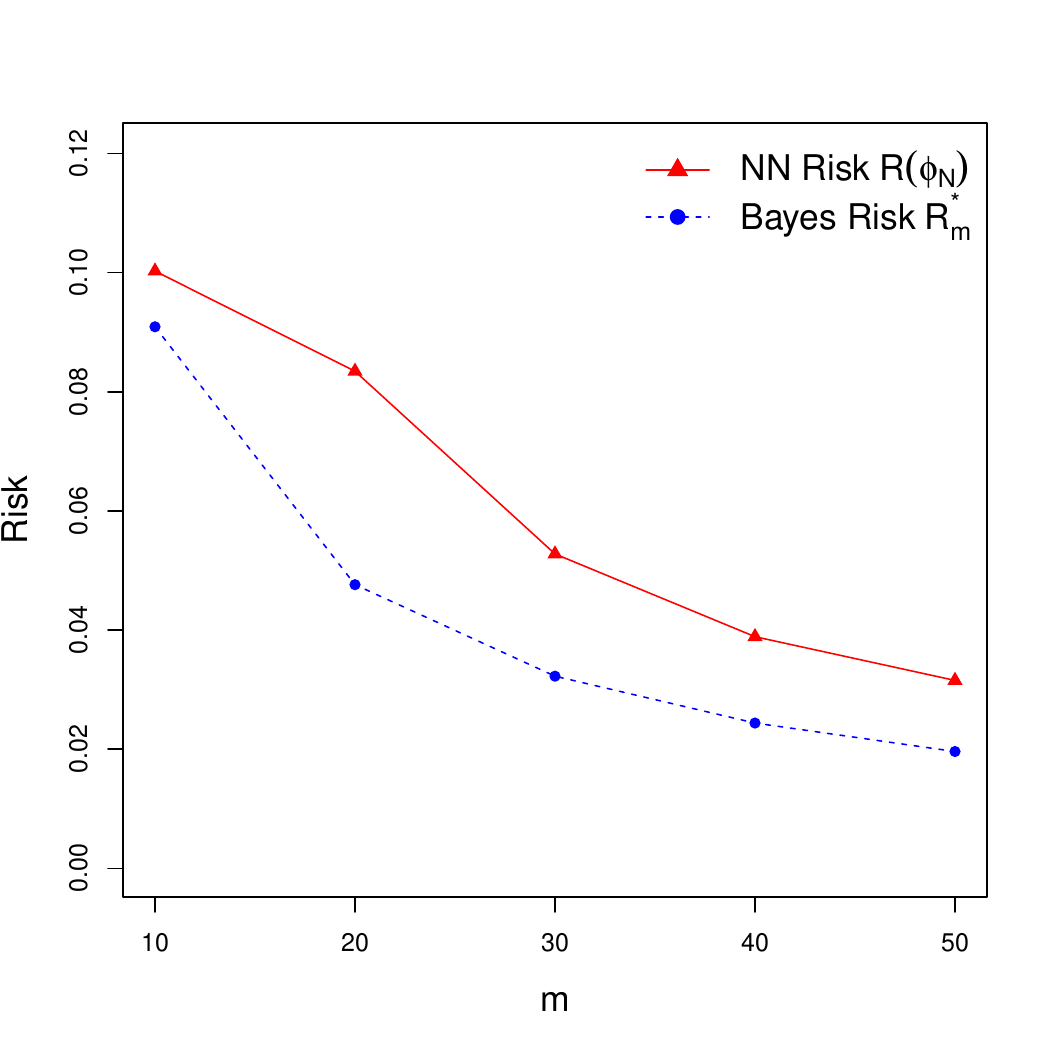}
    \end{subfigure}
    \hfill
    \begin{subfigure}[b]{0.495\textwidth}
        \centering
        \includegraphics[width=\textwidth]{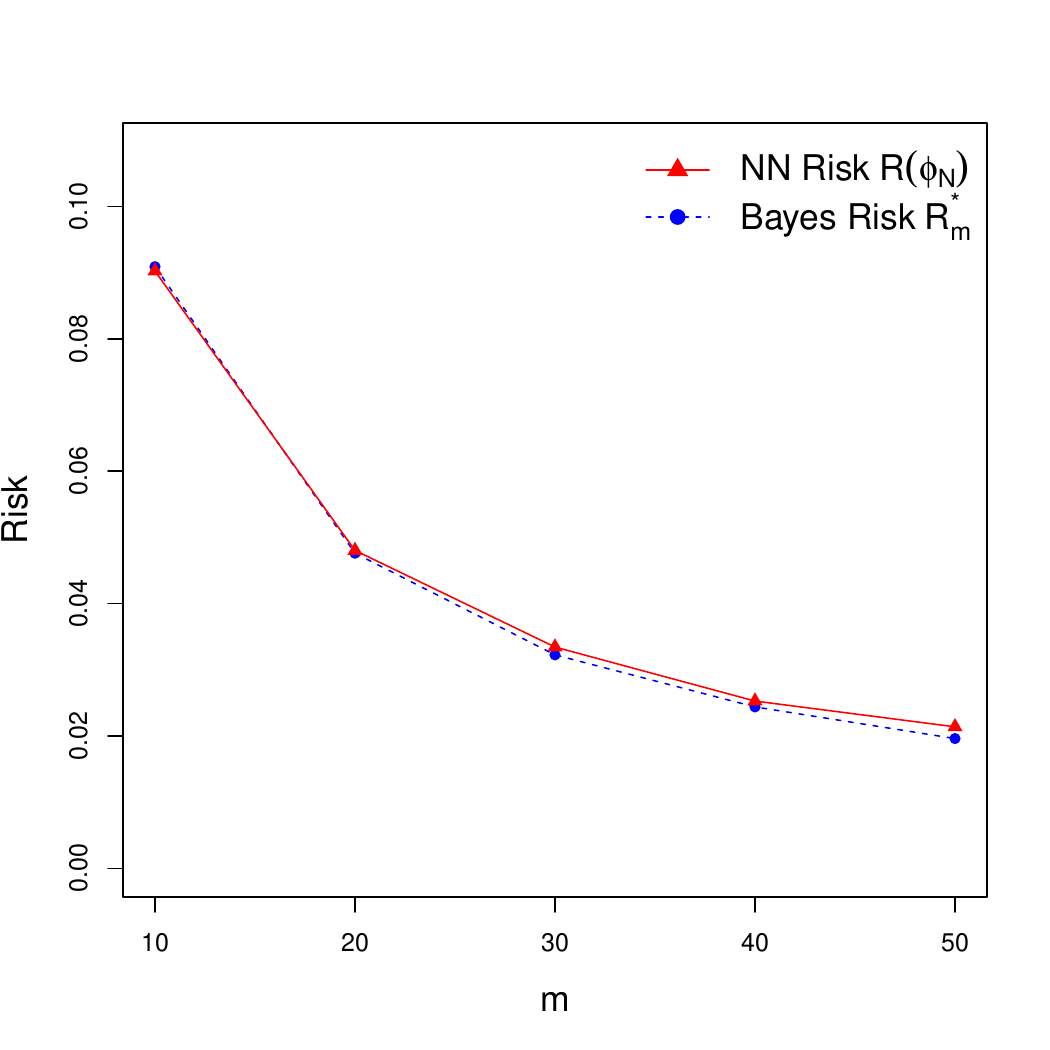}
    \end{subfigure}
    \caption{Bayes risk compared to the risk of the empirical risk minimizer as a function of number of replicates $m$, left: $N=100$, right: $N=1000$.}
    \label{fig:error_dec}
\end{figure}

Figure~\ref{fig:gen_error} shows the improvement of the neural estimator as a function of the training sample size. Here, the number of replica $m=30$ is fixed. While for low $N$, the neural estimator has a large generalization error, the empirical and true risk converge to the Bayes risk in exponential dependency of $N$.
\begin{figure}[h]
    \centering 
        \includegraphics[width=0.495\textwidth]{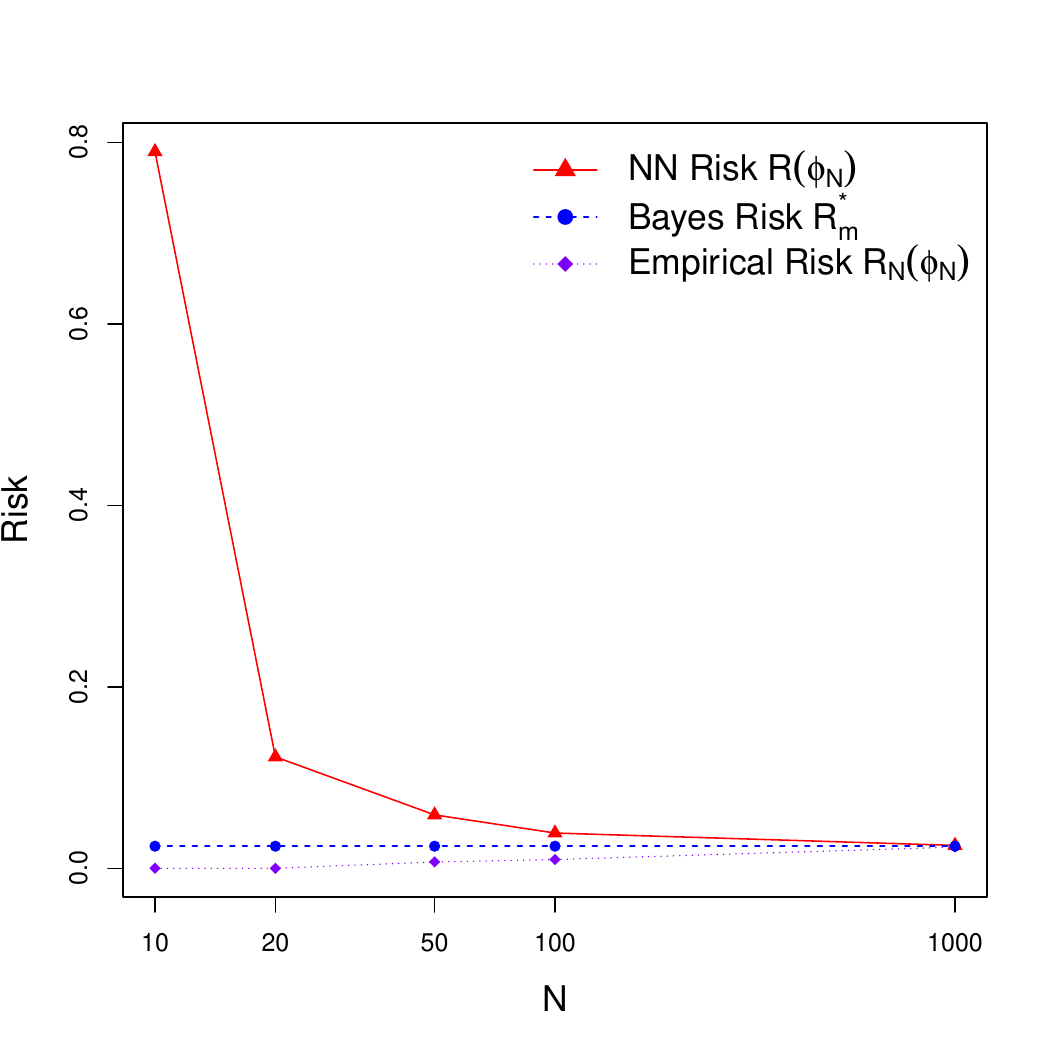}
        \caption{Generalization error for $m=40$, logarithmically in the training data size $N$.}
        \label{fig:gen_error}
\end{figure}